\documentclass{article}

\usepackage{booktabs}       
\usepackage{amsfonts,amsthm,amsmath,bbm,amssymb}       
\usepackage{dsfont}
\usepackage[linesnumbered,ruled,lined]{algorithm2e}
\usepackage{nicefrac}       
\usepackage{microtype}      
\usepackage{lipsum}
\usepackage{fancyhdr}       
\usepackage{graphicx,adjustbox}       
\usepackage{caption,subcaption}
\usepackage{xparse}         
\usepackage{xcolor}
\usepackage{paracol}
\ExplSyntaxOn
\DeclareExpandableDocumentCommand{\IfNoValueOrEmptyTF}{mmm}
 {
  \IfNoValueTF{#1}{#2}
   {
    \tl_if_empty:nTF {#1} {#2} {#3}
   }
 }
\ExplSyntaxOff

\newtheorem{theorem}{Theorem}[section]
\newtheorem{proposition}[theorem]{Proposition}
\newtheorem{lemma}[theorem]{Lemma}
\newtheorem{corollary}[theorem]{Corollary}
\newtheorem{fact}[theorem]{Fact}
\newtheorem{definition}[theorem]{Definition}
\newtheorem{remark}{Remark}

\newtheorem{assumption}[theorem]{Assumption}
\newtheorem{problem}[theorem]{Problem}
\newtheorem{claim}[theorem]{Claim}

\newcommand{\cleq}[1][]{\overset{\mathrm{#1}}{\leq}}
\newcommand{\cgeq}[1][]{\overset{\mathrm{#1}}{\geq}}

\newcommand{\ceq}[1][]{\overset{\mathrm{#1}}{=}}
\NewDocumentCommand{\sample}{}{\overset{\text{i.i.d.}}{\sim}}

\NewDocumentCommand{\ScriptedMathSymbol}{md<>d()o}{
    \IfNoValueOrEmptyTF{#2}{
        \IfNoValueOrEmptyTF{#3}{
            \IfNoValueOrEmptyTF{#4}{
                {#1}
            }{
                {#1^{#4}}
            }
        }{
            \IfNoValueOrEmptyTF{#4}{
                {#1^{(#3)}}
            }{
                {#1^{(#3)#4}}
            }
        }
    }{
        \IfNoValueOrEmptyTF{#3}{
            \IfNoValueOrEmptyTF{#4}{
                {#1_{#2}}
            }{
                {#1_{#2}^{#4}}
            }
        }{
            \IfNoValueOrEmptyTF{#4}{
                {#1_{#2}^{(#3)}}
            }{
                {#1_{#2}^{(#3)#4}}
            }
        }
    }
}

\DeclareMathOperator*{\expectation}{\mathbb{E}}
\NewDocumentCommand{\E}{sD<>{}om}{
    \IfBooleanTF{#1}{
        \ScriptedMathSymbol{\expectation}<#2>[#3]\mbr*{#4}
    }{
        \ScriptedMathSymbol{\expectation}<#2>[#3]\mbr{#4}
    }
}

\DeclareMathOperator*{\argmin}{argmin}
\DeclareMathOperator*{\argmax}{argmax}

\NewDocumentCommand{\union}{sD<>{}O{}}{
    \IfBooleanTF{#1}{
        \ScriptedMathSymbol{\cup}<#2>[#3]
    }{
        \ScriptedMathSymbol{\ \cup\ }<#2>[#3]
    }
}
\DeclareMathOperator*{\bunion}{\bigcup}

\NewDocumentCommand{\bigunion}{sD<>{}O{}}{
    \IfBooleanTF{#1}{
        \ScriptedMathSymbol{\bunion}<#2>[#3]
    }{
        \ScriptedMathSymbol{\ \bunion\ }<#2>[#3]
    }
}
\NewDocumentCommand{\isect}{sD<>{}O{}}{
    \IfBooleanTF{#1}{
        \ScriptedMathSymbol{\cap}<#2>[#3]
    }{
        \ScriptedMathSymbol{\ \cap\ }<#2>[#3]
    }
}
\NewDocumentCommand{\bigisect}{sD<>{}O{}}{
    \IfBooleanTF{#1}{
        \ScriptedMathSymbol{\bigcap}<#2>[#3]
    }{
        \ScriptedMathSymbol{\ \bigcap\ }<#2>[#3]
    }
}
\NewDocumentCommand{\por}{sD<>{}O{}}{
    \IfBooleanTF{#1}{
        \ScriptedMathSymbol{\bigvee}<#2>[#3]
    }{
        \ScriptedMathSymbol{\vee}<#2>[#3]
    }
}
\NewDocumentCommand{\pand}{sD<>{}O{}}{
    \IfBooleanTF{#1}{
        \ScriptedMathSymbol{\bigwedge}<#2>[#3]
    }{
        \ScriptedMathSymbol{\wedge}<#2>[#3]
    }
}
\newcommand{\cond}{\ | \ }
\NewDocumentCommand{\ceil}{sm}{
    \IfBooleanTF{#1}{
        \ensuremath \lceil {#2} \rceil
    }{
        \ensuremath \left\lceil {#2} \right\rceil
    }
}
\NewDocumentCommand{\floor}{sm}{
    \IfBooleanTF{#1}{
        \ensuremath \lfloor {#2} \rfloor
    }{
        \ensuremath \left\lfloor {#2} \right\rfloor
    }
}
\NewDocumentCommand{\sbr}{sm}{
    \IfBooleanTF{#1}{
        \ensuremath ( #2 )
    }{
        \ensuremath \left( #2 \right)
    }
}
\NewDocumentCommand{\mbr}{sm}{
    \IfBooleanTF{#1}{
        \ensuremath [ #2 ]
    }{
        \ensuremath \left[ #2 \right]
    }
}
\NewDocumentCommand{\lbr}{sm}{
    \IfBooleanTF{#1}{
        \ensuremath \{ #2 \}
    }{
        \ensuremath \left\{ #2 \right\}
    }
}

\NewDocumentCommand{\prob}{sD<>{}m}{
    \IfBooleanTF{#1}{
        \ScriptedMathSymbol{\Pr}<#2>\lbr*{#3}
    }{
        \ScriptedMathSymbol{\Pr}<#2>\lbr{#3}
    }
}

\NewDocumentCommand{\distr}{sD<>{}D(){}O{}}{
    \IfBooleanTF{#1}{
        \ScriptedMathSymbol{\hat{\mathcal{D}}}<#2>(#3)[#4]
    }{
        \ScriptedMathSymbol{\mathcal{D}}<#2>(#3)[#4]
    }
}
\NewDocumentCommand{\gaussian}{O{}O{0}O{1}}{
    \ScriptedMathSymbol{\mathcal{N}}[#1]\sbr*{#2,#3}
}
\NewDocumentCommand{\conceptclass}{D<>{}D(){}O{}}{
    \ScriptedMathSymbol{\mathcal{C}}<#1>(#2)[#3]
}
\NewDocumentCommand{\concept}{sD<>{}D(){}O{}o}{
    \IfBooleanTF{#1}{
        \IfNoValueOrEmptyTF{#5}{
            \ScriptedMathSymbol{\neg c}<#2>(#3)[#4]
        }{
            \ScriptedMathSymbol{\neg c}<#2>(#3)[#4]\sbr{#5}
        }
    }{
        \IfNoValueOrEmptyTF{#5}{
            \ScriptedMathSymbol{c}<#2>(#3)[#4]
        }{
            \ScriptedMathSymbol{c}<#2>(#3)[#4]\sbr{#5}
        }
    }
}
\NewDocumentCommand{\hypothesisclass}{D<>{}D(){}O{}}{
    \ScriptedMathSymbol{\mathcal{H}}<#1>(#2)[#3]
}
\NewDocumentCommand{\hypothesis}{sD<>{}D(){}o}{
    \IfBooleanTF{#1}{
        \IfNoValueOrEmptyTF{#4}{
            \ScriptedMathSymbol{h}<#2>(#3)[c]
        }{
            \ScriptedMathSymbol{h}<#2>(#3)[c]\sbr*{#4}
        }
    }{
        \IfNoValueOrEmptyTF{#4}{
            \ScriptedMathSymbol{h}<#2>(#3)
        }{
            \ScriptedMathSymbol{h}<#2>(#3)\sbr*{#4}
        }
    }
}
\NewDocumentCommand{\parameterset}{mD<>{}D(){}o}{
    \ScriptedMathSymbol{\mathcal{#1}}<#2>(#3)[#4]
}
\NewDocumentCommand{\literal}{sD<>{}o}{
    \IfBooleanTF{#1}{
        \IfNoValueOrEmptyTF{#3}{
            \ScriptedMathSymbol{\neg l}<#2>
        }{
            \ScriptedMathSymbol{\neg l}<#2>\sbr{#3}
        }
    }{
        \IfNoValueOrEmptyTF{#3}{
            \ScriptedMathSymbol{l}<#2>
        }{
            \ScriptedMathSymbol{l}<#2>\sbr{#3}
        }
    }
}
\NewDocumentCommand{\subgroupclass}{D<>{}D(){}O{}}{
    \ScriptedMathSymbol{\mathcal{G}}<#1>(#2)[#3]
}
\NewDocumentCommand{\subgroup}{sD<>{}D(){}O{}o}{
    \IfBooleanTF{#1}{
        \IfNoValueOrEmptyTF{#5}{
            \ScriptedMathSymbol{\neg g}<#2>(#3)[#4]
        }{
            \ScriptedMathSymbol{\neg g}<#2>(#3)[#4]\sbr{#5}
        }
    }{
        \IfNoValueOrEmptyTF{#5}{
            \ScriptedMathSymbol{g}<#2>(#3)[#4]
        }{
            \ScriptedMathSymbol{g}<#2>(#3)[#4]\sbr{#5}
        }
    }
}

\NewDocumentCommand{\unfairness}{sD<>{}m}{
    \IfBooleanTF{#1}{
        \ScriptedMathSymbol{\bar{u}}<#2>\sbr{#3}
    }{
        \ScriptedMathSymbol{u}<#2>\sbr{#3}
    }
}
\NewDocumentCommand{\weight}{sD<>{}D(){}O{}o}{
    \IfBooleanTF{#1}{
        \IfNoValueOrEmptyTF{#5}{
            \ScriptedMathSymbol{W}<#2>(#3)[#4]
        }{
            \ScriptedMathSymbol{W}<#2>(#3)[#4]\sbr{#5}
        }
    }{
        \IfNoValueOrEmptyTF{#5}{
            \ScriptedMathSymbol{w}<#2>(#3)[#4]
        }{
            \ScriptedMathSymbol{w}<#2>(#3)[#4]\sbr{#5}
        }
    }
}
\NewDocumentCommand{\subsets}{sD<>{}D(){}O{}o}{
    \IfBooleanTF{#1}{
        \IfNoValueOrEmptyTF{#5}{
            \ScriptedMathSymbol{S}<#2>(#3)[{#4}c]
        }{
            \ScriptedMathSymbol{S}<#2>(#3)[{#4}c]\sbr{#5}
        }
    }{
        \IfNoValueOrEmptyTF{#5}{
            \ScriptedMathSymbol{S}<#2>(#3)[#4]
        }{
            \ScriptedMathSymbol{S}<#2>(#3)[#4]\sbr{#5}
        }
    }
}
\NewDocumentCommand{\indexset}{sD<>{}D(){}O{}}{
    \IfBooleanTF{#1}{
        \ScriptedMathSymbol{\mathcal{I}}<#2>(#3)[#4]
    }{
        \ScriptedMathSymbol{I}<#2>(#3)[#4]
    }
}
\NewDocumentCommand{\reals}{D<>{}O{}}{
    \ScriptedMathSymbol{\mathbb{R}}<#1>[#2]
}
\NewDocumentCommand{\naturals}{D<>{}O{}}{
    \ScriptedMathSymbol{\mathbb{N}}<#1>[#2]
}
\NewDocumentCommand{\integer}{D<>{}O{}}{
    \ScriptedMathSymbol{\mathbb{Z}}<#1>[#2]
}
\NewDocumentCommand{\sphere}{D<>{}O{}}{
    \ScriptedMathSymbol{\mathbb{S}}<#1>[#2]
}
\NewDocumentCommand\booldomain{O{}}{
    \ScriptedMathSymbol{\lbr{0 ,1}}[#1]
}
\NewDocumentCommand{\binarydomain}{O{}}{
    \ScriptedMathSymbol{\lbr{-1, +1}}[#1]
}
\NewDocumentCommand{\interval}{O{-1}O{+1}}{
    \ScriptedMathSymbol{\mbr{#1,#2}}
}
\NewDocumentCommand{\algo}{D<>{}D(){}o}{
    \IfNoValueOrEmptyTF{#3}{
        \ScriptedMathSymbol{\mathcal{A}}<#1>(#2)
    }{
        \ScriptedMathSymbol{\mathcal{A}}<#1>(#2)\sbr*{#3}
    }
}

\NewDocumentCommand{\norm}{smD<>{}O{}}{
    \IfBooleanTF{#1}{
        \ScriptedMathSymbol{\left\lVert {#2} \right\rVert}<#3>[#4]
    }{
        \ScriptedMathSymbol{\lVert {#2} \rVert}<#3>[#4]
    }
}
\NewDocumentCommand{\abs}{smO{}}{
    \IfBooleanTF{#1}{
        \ScriptedMathSymbol{| {#2} |}[#3]
    }{
        \ScriptedMathSymbol{\left| {#2} \right|}[#3]
    }
}
\NewDocumentCommand{\innerprod}{smmO{}}{
    \IfBooleanTF{#1}{
        \ScriptedMathSymbol{\langle #2, #3 \rangle}[#4]
    }{
        \ScriptedMathSymbol{\left\langle #2, #3 \right\rangle}[#4]
    }
}

\NewDocumentCommand{\indicator}{D<>{}o}{
    \IfNoValueOrEmptyTF{#2}{
        \ScriptedMathSymbol{\mathds{1}}<#1>
    }{
        \ScriptedMathSymbol{\mathds{1}}<#1>\lbr*{#2}
    }
}

\NewDocumentCommand{\errorregion}{D<>{}D(){}oo}{
    \IfNoValueOrEmptyTF{#4}{
        \ScriptedMathSymbol{\mathcal{E}}<#1>(#2)[#3]
    }{
        \ScriptedMathSymbol{\mathcal{E}}<#1>(#2)[#3]\sbr{#4}
    }
}

\NewDocumentCommand{\identity}{o}{
    \ScriptedMathSymbol{\mathrm{I}}[#1]
}

\newcommand{\opt}{\mathrm{opt}}

\NewDocumentCommand{\loss}{sD<>{}D(){}om}{
    \IfBooleanTF{#1}{
        \ScriptedMathSymbol{\mathcal{L}}<#2>(#3)[#4]\sbr{#5}
    }{
        \ScriptedMathSymbol{\mathcal{L}}<#2>(#3)[#4]\sbr*{#5}
    }
}

\NewDocumentCommand{\err}{sD<>{}D(){}o}{
    \IfBooleanTF{#1}{
        \IfNoValueOrEmptyTF{#4}{
            \ScriptedMathSymbol{\mathrm{err}}<#2>(#3)
        }{
            \ScriptedMathSymbol{\mathrm{err}}<#2>(#3)\sbr{#4}
        }
    }{
        \IfNoValueOrEmptyTF{#4}{
            \ScriptedMathSymbol{\mathrm{err}}<#2>(#3)
        }{
            \ScriptedMathSymbol{\mathrm{err}}<#2>(#3)\sbr*{#4}
        }
    }
}

\NewDocumentCommand{\derivative}{D<>{}o}{
    \ScriptedMathSymbol{\nabla}<#1>[#2]
}
\NewDocumentCommand{\func}{smD<>{}D(){}o}{
    \IfBooleanTF{#1}{
        \IfNoValueOrEmptyTF{#5}{
            \ScriptedMathSymbol{#2}<#3>(#4)
        }{
            \ScriptedMathSymbol{#2}<#3>(#4)\sbr{#5}
        }
    }{
        \IfNoValueOrEmptyTF{#5}{
            \ScriptedMathSymbol{#2}<#3>(#4)
        }{
            \ScriptedMathSymbol{#2}<#3>(#4)\sbr*{#5}
        }
    }
}

\NewDocumentCommand{\bvar}{smD<>{}D(){}o}{
    \IfBooleanTF{#1}{
        \ScriptedMathSymbol{\bar{\mathbf{#2}}}<#3>(#4)[#5]
    }{
        \ScriptedMathSymbol{\mathbf{#2}}<#3>(#4)[#5]
    }
}
\NewDocumentCommand{\bvarseq}{smO{1}D<>{}O{1}D(){}o}{
    \IfBooleanTF{#1}{
        \IfNoValueOrEmptyTF{#4}{
            \ScriptedMathSymbol{\bvar*{#2}}(#5)[#7], \ldots, \ScriptedMathSymbol{\bvar*{#2}}(#6)[#7]
        }{
            \ScriptedMathSymbol{\bvar*{#2}}<#3>(#6)[#7], \ldots, \ScriptedMathSymbol{\bvar*{#2}}<#4>(#6)[#7]
        }  
    }{
        \IfNoValueOrEmptyTF{#4}{
            \ScriptedMathSymbol{\bvar{#2}}(#5)[#7], \ldots, \ScriptedMathSymbol{\bvar{#2}}(#6)[#7]
        }{
            \ScriptedMathSymbol{\bvar{#2}}<#3>(#6)[#7], \ldots, \ScriptedMathSymbol{\bvar{#2}}<#4>(#6)[#7]
        }  
    }   
}

\NewDocumentCommand{\lvar}{smD<>{}D(){}o}{
    \IfBooleanTF{#1}{
        \ScriptedMathSymbol{\bar{#2}}<#3>(#4)[#5]
    }{
        \ScriptedMathSymbol{#2}<#3>(#4)[#5]
    }
}
\NewDocumentCommand{\lvarseq}{smO{1}D<>{}O{1}D(){}o}{
    \IfBooleanTF{#1}{
        \IfNoValueOrEmptyTF{#4}{
            \ScriptedMathSymbol{\lvar*{#2}}(#5)[#7], \ldots, \ScriptedMathSymbol{\lvar*{#2}}(#6)[#7]
        }{
            \ScriptedMathSymbol{\lvar*{#2}}<#3>(#6)[#7], \ldots, \ScriptedMathSymbol{\lvar*{#2}}<#4>(#6)[#7]
        }  
    }{
        \IfNoValueOrEmptyTF{#4}{
            \ScriptedMathSymbol{\lvar{#2}}(#5)[#7], \ldots, \ScriptedMathSymbol{\lvar{#2}}(#6)[#7]
        }{
            \ScriptedMathSymbol{\lvar{#2}}<#3>(#6)[#7], \ldots, \ScriptedMathSymbol{\lvar{#2}}<#4>(#6)[#7]
        }  
    }   
}

\NewDocumentCommand\conjunction{sD<>{}oD(){}}{
    \IfBooleanTF#1
        {\IfNoValueTF{#3}
            {c_{#2}^{*#4}}
            {c_{#2}^{*(#3)#4}}
        }
        {\IfNoValueTF{#3}
            {c_{#2}^{#4}}
            {c_{#2}^{(#3)#4}}
        }
}

\NewDocumentCommand{\bigO}{sm}{
    \IfBooleanTF{#1}{
        \ScriptedMathSymbol{\tilde{O}}\sbr*{#2}
    }{
        \ScriptedMathSymbol{O}\sbr*{#2}
    }
}

\NewDocumentCommand{\poly}{sm}{
    \IfBooleanTF{#1}{
        \mathrm{poly}\sbr*{#2}
    }{
        \mathrm{poly}\sbr{#2}
    }
}

\usepackage{PRIMEarxiv}

\usepackage[utf8]{inputenc} 
\usepackage[T1]{fontenc}    
\usepackage{hyperref}       
\usepackage{url}            
\usepackage[square,numbers]{natbib}
\usepackage{paralist}
\graphicspath{{media/}}     

\title{Distribution-Specific Agnostic Conditional Classification With Halfspaces
}

\author{
  Jizhou Huang\\
  Washington Universtiy in St. Louis \\
  St. Louis, MO, USA\\
  \texttt{jizhou.huang@wustl.edu} \\
   \And
  Brendan Juba \\
  Washington Universtiy in St. Louis \\
  St. Louis, MO, USA\\
  \texttt{bjuba@wustl.edu} \\
}
\usepackage{tikz}
\usepackage{tikz-3dplot-circleofsphere}

\definecolor{mypink}{HTML}{D691A4}
\definecolor{darkblue}{HTML}{343F65}
\definecolor{myblue}{HTML}{78B9D2}
\definecolor{myorange}{HTML}{FCAF7C}

\NewDocumentCommand{\drawlengthincrease}{O{\bvar{w}(i)}O{\func{g}<\bvar{w}(i)>}O{\bvar{u}(i+1)}O{\bvar{w}(i+1)}D(){1}}{
    \begin{tikzpicture}[scale=#5]

        \coordinate (O) at (0,0);
        \node at (O) [below left] {O};
    
        \draw[->, thick] (O) -- (0,4);
        \node at (0.15,4) [above left] {$#1$};
    
        \draw[->,myorange, thick] (0,4) -- (2.3094,4);
        \node at (1.25,4) [above] {$\color{myorange}\beta\E{-{#2}}$};

        \draw (0,3.7) -- (0.3,3.7) -- (0.3,4);
    
        \draw[->, gray, dashed, thick] (O) -- (2.3094,4);
        \node at (2.3094,4) [above right] {$\color{gray}#3$};
    
        \draw[->, thick] (O) -- (2,3.4641); 
        \node at (2,3.4641) [right] {$#4$};
    
        \draw[dashed] (0,4) arc[start angle=90, end angle=60, radius=4]; 
    
    \end{tikzpicture}
}

\NewDocumentCommand{\drawerrorregion}{sD(){1}}{
    \begin{tikzpicture}[scale=#2]

        \coordinate (O) at (0,0);
        \node at (O) [below right] {O};
    
        \draw[->, thick] (-3,0) -- (3,0);
        \node at (3,0) [right] {$\bvar{e}<1>$};
    
        \draw[dashed, gray, thick, opacity = 0.75] (-0.75,-0.75) -- (2.12,2.12); 
        \node at (2.12,2.12) [above right] {};
    
        \fill[myorange, opacity=0.5] (0,0) -- (3,0) arc[start angle=0, end angle=45, radius=3] -- cycle;
        \fill[myblue, opacity=0.5] (0,0) -- (2.12,2.12) arc[start angle=45, end angle=180, radius=3] -- cycle;
    
        \draw[->, thick] (O) -- (0,3);
        \node at (0.3,3) [above] {$\bvar{e}<2> \sbr*{\bvar{w}}$};
    
        \draw[->, thick] (O) -- (-2.12,2.12); 
        \node at (-2.12,2.12) [above left] {$\bvar{v}$};
    
        \draw (0.3,0) -- (0.3,0.3) -- (0,0.3);
    
        \draw (-0.2121,0.2121) -- (-0.424,0) -- (-0.2121,-0.2121);
    
        \draw[-, thick] (0.6,0) arc (0:45:0.6);
        \node at (0.6,0.35) [right] {$\theta(\bvar{v},\bvar{w})$};

        \IfBooleanTF{#1}{
            \def\const{1.5}
            \draw[myblue, thick, opacity = 0.75] (-{sqrt(8)},1) -- (-{\const*sqrt(8)},\const);
            \node at (-{\const*sqrt(8)},\const) [above left] {$I_2$};
            \draw[myorange, thick, opacity = 0.75] ({sqrt(8)},1) -- ({\const*sqrt(8)},\const);
            \node at ({\const*sqrt(8)},\const) [above right] {$I_1$};
        }{}
    
    \end{tikzpicture}
}
\NewDocumentCommand{\drawintegralspace}{s}{
    \tdplotsetmaincoords{60}{110} 

    \begin{tikzpicture}[tdplot_main_coords]

        \coordinate (O) at (0,0,0);
        \node at (O) [above right] {O};

        \def\tanthirty{0.57735}
        \def\sinthirty{0.5}
        \def\costhirty{0.86603}
        \def\tanfourtyfive{1}
        \def\sinfourtyfive{0.70711}
        \def\cosfourtyfive{0.70711}
        \def\tansixty{0.57735}
        \def\sinsixty{0.86603}
        \def\cossixty{0.5}
        \def\len{1}
        \def\llen{1}

        \draw[tdplot_screen_coords,thin,black!40] (0,0,0) circle (\len);
        \tdplotCsDrawLatCircle[thin,black!40]{\len}{0}
        \tdplotCsDrawLonCircle[thin,black!40]{\len}{90}
        
        \tdplotsetrotatedcoords{-45}{90}{0};
        \fill[tdplot_rotated_coords, myorange, opacity=0.5] (-1,0,0) arc (180:360:1);
        
        \draw[thin, dashed] (0,0,0) -- (-{\len*\sinfourtyfive}, -{\len*\cosfourtyfive}, 0);

        \tdplotsetrotatedcoords{0}{90}{0};
        \fill[tdplot_rotated_coords, myblue, opacity=0.75] (0.995,0,0) arc (0:180:0.995);

        \draw[thin,->] (0,0,0) -- (0,1.25,0) node[anchor=north]{$\bvar{e}<1>$};

        \foreach \a in {1,2,...,44}{
            \tdplotsetrotatedcoords{-{\a+180}}{90}{0};
            \draw[tdplot_rotated_coords, thick, mypink!75,opacity=0.5] (1,0,0) arc (0:180:1);
        }
        \def\num{100}
        \foreach \a in {-\num, ..., \num}{
            \draw[thick, mypink!75,opacity=0.5] (0,-{sqrt(1 -(\a/\num)*(\a/\num))},{\a/\num}) arc (270:225:{sqrt(1 -(\a/\num)*(\a/\num))});
        }

        \IfBooleanTF{#1}{
            \pgfmathsetmacro{\phivec}{30}
            \pgfmathsetmacro{\thetavec}{60}
            
            \tdplotsetcoord{P}{0.6}{\phivec}{\thetavec}
            \draw[->,thick, mypink!150] (O) -- (P) node[above] {$\bvar{x}$};
            \draw[dashed, thick,  mypink!150] (O) -- (Pxy);
            \draw[dashed, thick,  mypink!150] (P) -- (Pxy);
            \tdplotdrawarc{(O)}{0.3}{\thetavec}{90}{anchor=north west}{$\theta$}
            \tdplotsetthetaplanecoords{\thetavec}
            \tdplotdrawarc[tdplot_rotated_coords]{(0,0,0)}{0.2}{0}%
                {\phivec}{anchor=south}{$\ \ \phi$}
        }{}

        \tdplotCsDrawLonCircle[thick,black!60]{\len}{0}
        
        \foreach \a in {1,2,...,44}{
            \tdplotsetrotatedcoords{-\a}{90}{0};
            \draw[tdplot_rotated_coords, thick, mypink,opacity=0.5] (1,0,0) arc (0:180:1);
        }

        \def\num{100}
        \foreach \a in {-\num, ..., \num}{
            \draw[thick, mypink,opacity=0.5] (0,{sqrt(1 -(\a/\num)*(\a/\num))},{\a/\num}) arc (90:45:{sqrt(1 -(\a/\num)*(\a/\num))});
        }
        
        \tdplotsetrotatedcoords{-45}{90}{0};
        \fill[tdplot_rotated_coords, thick, myorange, opacity=0.5] (1,0,0) arc (0:180:1);

        \draw[thin, dashed] (0,0,0) -- ({\len*\sinfourtyfive}, {\len*\cosfourtyfive}, 0);

        \tdplotCsDrawLonCircle[thick,black!60]{\len}{-45}

        \tdplotsetrotatedcoords{0}{90}{0};
        \fill[tdplot_rotated_coords, myblue, opacity=0.75] (0.995,0,0) arc (0:-180:0.995);

        \draw[thin, dashed] (0,0,0) -- (0,-\len,0) node[anchor=north east]{};

        \draw[thick,->] (0,0,0) -- (\len,0,0) node[anchor= north west]{$\bvar{w}$};

        \draw[thick,->] (0,0,0) -- ({\len*\sinfourtyfive},{-\len*\sinfourtyfive},0) node[anchor=north]{$\bvar{w} '$};

        \draw[thin,->] (0,0,0) -- (1.25,0,0) node[anchor=north west]{$\bvar{e}<2>$};

        \draw[thin,->] (0,0,-0) -- (0,0,1.25) node[anchor=west]{$\bvar{e}<3>$};
        \draw[dashed] (0,0,0) -- (0,0,-\len) node[anchor=north east]{};
        
        \tdplotdrawarc[tdplot_main_coords]{(O)}{0.3}{-45}{0}{anchor=north east}{$\Delta\theta$}

        \tdplotsetrotatedcoords{180}{-90}{0};
        \draw[tdplot_rotated_coords, thick, black!60] (1,0,0) arc (0:-112:1);
        \tdplotsetrotatedcoords{-1}{90}{0};
        \draw[tdplot_rotated_coords, thick, mypink,opacity=0.5] (1,0,0) arc (0:180:1);

        \tdplotsetrotatedcoords{-90}{0}{180};
        \draw[tdplot_rotated_coords, thin, black!40] (1,0,0) arc (0:-150:1);
        \tdplotsetrotatedcoords{90}{-90}{0};
        \draw[tdplot_rotated_coords, thin, black!40] (1,0,0) arc (0:-150:1);
    \end{tikzpicture}
}

\NewDocumentCommand{\drawintegralspaceALT}{s}{
    \tdplotsetmaincoords{60}{110} 

    \begin{tikzpicture}[scale=3,tdplot_main_coords]

        \coordinate (O) at (0,0,0);
        \node at (O) [above right] {O};

        \def\tanthirty{0.57735}
        \def\sinthirty{0.5}
        \def\costhirty{0.86603}
        \def\tanfourtyfive{1}
        \def\sinfourtyfive{0.70711}
        \def\cosfourtyfive{0.70711}
        \def\tansixty{0.57735}
        \def\sinsixty{0.86603}
        \def\cossixty{0.5}
        \def\len{1}
        \def\llen{1}

        \tdplotCsDrawLatCircle[thin,black!40]{\len}{0}
        \tdplotCsDrawLonCircle[thin,black!40]{\len}{90}
        
        \tdplotsetrotatedcoords{-45}{90}{0};
        \fill[tdplot_rotated_coords, myorange, opacity=0.5] (-1,0,0) arc (180:360:1);
        
        \draw[thin, dashed] (0,0,0) -- (-{\len*\sinfourtyfive}, -{\len*\cosfourtyfive}, 0);

        \tdplotsetrotatedcoords{0}{90}{0};
        \fill[tdplot_rotated_coords, myblue, opacity=0.75] (0.995,0,0) arc (0:180:0.995);

        \draw[thin,->] (0,0,0) -- (0,1.25,0) node[anchor=north]{$\bvar{e}<1>$};

        \tdplotCsDrawLonCircle[thick,black!60]{\len}{-45}
        \foreach \a in {1,2,...,44}{
            \tdplotsetrotatedcoords{-{\a+180}}{90}{0};
            \draw[tdplot_rotated_coords, thick, mypink!75,opacity=0.5] (1,0,0) arc (0:180:1);
        }
        \def\num{100}
        \foreach \a in {-\num, ..., \num}{
            \draw[thick, mypink!75,opacity=0.5] (0,-{sqrt(1 -(\a/\num)*(\a/\num))},{\a/\num}) arc (270:225:{sqrt(1 -(\a/\num)*(\a/\num))});
        }

        \IfBooleanTF{#1}{
            \pgfmathsetmacro{\phivec}{30}
            \pgfmathsetmacro{\thetavec}{60}
            
            \tdplotsetcoord{P}{0.5}{\phivec}{\thetavec}
            \draw[->,thick, mypink!150] (O) -- (P) node[above] {$\bvar{x}<V>$};
            \draw[dashed, thick,  mypink!150] (O) -- (Pxy);
            \draw[dashed, thick,  mypink!150] (P) -- (Pxy);
            \tdplotdrawarc{(O)}{0.25}{\thetavec}{90}{anchor=north west}{$\theta$}
            \tdplotsetthetaplanecoords{\thetavec}
            \tdplotdrawarc[tdplot_rotated_coords]{(0,0,0)}{0.2}{0}%
                {\phivec}{anchor=south}{$\ \ \phi$}
        }{}

        \tdplotCsDrawLonCircle[thick,black!60]{\len}{0}
        
        \foreach \a in {1,2,...,44}{
            \tdplotsetrotatedcoords{-\a}{90}{0};
            \draw[tdplot_rotated_coords, thick, mypink,opacity=0.5] (1,0,0) arc (0:180:1);
        }

        \def\num{100}
        \foreach \a in {-\num, ..., \num}{
            \draw[thick, mypink,opacity=0.5] (0,{sqrt(1 -(\a/\num)*(\a/\num))},{\a/\num}) arc (90:45:{sqrt(1 -(\a/\num)*(\a/\num))});
        }
        
        \tdplotsetrotatedcoords{-45}{90}{0};
        \fill[tdplot_rotated_coords, thick, myorange, opacity=0.5] (1,0,0) arc (0:180:1);

        \draw[thin, dashed] (0,0,0) -- ({\len*\sinfourtyfive}, {\len*\cosfourtyfive}, 0);

        \tdplotsetrotatedcoords{-45}{90}{0};
        \draw[tdplot_rotated_coords, thick, black!60] (1,0,0) arc (0:213:1);
        \tdplotsetrotatedcoords{-224}{90}{0};
        \draw[tdplot_rotated_coords, thick, mypink!75,opacity=0.5] (1,0,0) arc (0:180:1);

        \tdplotsetrotatedcoords{0}{90}{0};
        \fill[tdplot_rotated_coords, myblue, opacity=0.75] (0.995,0,0) arc (0:-180:0.995);
        \draw[tdplot_rotated_coords, thick, black!60] (1,0,0) arc (0:-180:1);

        \draw[thin, dashed] (0,0,0) -- (0,-\len,0) node[anchor=north east]{};

        \draw[thick,->] (0,0,0) -- (\len,0,0) node[anchor= north west]{$\bvar{w}$};

        \draw[thick,->] (0,0,0) -- ({\len*\sinfourtyfive},{-\len*\sinfourtyfive},0) node[anchor=north]{$\bvar{w} '$};

        \draw[thin,->] (0,0,0) -- (1.25,0,0) node[anchor=north west]{$\bvar{e}<2>$};

        \draw[thin,->] (0,0,-0) -- (0,0,1.25) node[anchor=west]{$\bvar{e}<3>$};
        \draw[dashed] (0,0,0) -- (0,0,-\len) node[anchor=north east]{};
        
        \IfBooleanTF{#1}{
            \tdplotdrawarc[tdplot_main_coords]{(O)}{0.3}{-45}{0}{anchor=north east}{$\theta_1$}
        }{
            \tdplotdrawarc[tdplot_main_coords]{(O)}{0.3}{-45}{0}{anchor=north east}{$\Delta\theta$}
        }

        \tdplotsetrotatedcoords{180}{-90}{0};
        \draw[tdplot_rotated_coords, thick, black!60] (1,0,0) arc (0:-112:1);
        \tdplotsetrotatedcoords{-1}{90}{0};
        \draw[tdplot_rotated_coords, thick, mypink,opacity=0.5] (1,0,0) arc (0:180:1);

        \tdplotsetrotatedcoords{-90}{0}{0};
        \draw[tdplot_rotated_coords, thin, black!40] (1,0,0) arc (0:235:1);
        \tdplotsetrotatedcoords{90}{90}{0};
        \draw[tdplot_rotated_coords, thin, black!40] (1,0,0) arc (0:-247:1);
    \end{tikzpicture}
}
\RestyleAlgo{ruled}
\IncMargin{1em}

\begin{document}
\maketitle
    \begin{abstract}
        We study ``selective'' or ``conditional'' classification problems under an agnostic setting. Classification tasks commonly focus on modeling the relationship between features and categories that captures the vast majority of data. In contrast to common machine learning frameworks, conditional classification intends to model such relationships only on a subset of the data defined by some selection rule. Most work on conditional classification either solves the problem in a realizable setting or does not guarantee the error is bounded compared to an optimal solution. In this work, we consider selective/conditional classification by sparse linear classifiers for subsets defined by halfspaces, and give both positive as well as negative results for Gaussian feature distributions. On the positive side, we present the first PAC-learning algorithm for homogeneous halfspace selectors with error guarantee $\bigO*{\sqrt{\opt}}$, where $\opt$ is the smallest conditional classification error over the given class of classifiers and homogeneous halfspaces. On the negative side, we find that, under cryptographic assumptions, approximating the conditional classification loss within a small additive error is computationally hard even under Gaussian distribution. We prove that approximating conditional classification is at least as hard as approximating agnostic classification in both additive and multiplicative form.
    \end{abstract}
    \section{Introduction}
        Classification is the task of modeling the relationship between some features and membership in some category. Classification tasks are common across various fields, such as spam detection (classifying emails as "spam" or "not spam"), image recognition (identifying objects like "cat" or "dog"), and medical diagnosis (predicting whether a patient has a certain condition or not). Standard classification approaches seek to model the whole data distribution. By contrast, we consider the problems where a \textbf{better classifier} exists on a subset of the data. In particular, we will consider cases in which classifiers are sparse linear functions (or more generally, any small set of functions), and subsets are described by \textbf{selector} functions, given here by homogeneous halfspaces. 
        
        We study the distribution-specific PAC-learnability \citep{kearns1994toward} of the class of classifier-selector pairs in the presence of adversarial label noise. In the literature, this problem is known as ``conditional'' classification, but it is also part of a family of problems that are generally known as ``selective'' classification.
        
        \subsection{Background and Motivation}
            The first ``selective classification'' problem was introduced decades ago \citep{5222035,1054406}. The focus was on finding Bayes classifiers for the case where the data distribution is fully known. The appeal of effective selective classification is clear in applications where partial domain coverage is acceptable, or in scenarios where achieving extremely low risk is essential but unattainable with standard classification methods. Classification tasks in medical diagnosis and bioinformatics are often falling into this category \citep{khan2001classification,hanczar2008classification}.

            \citet{el2010foundations} gave a thorough theoretical analysis for selective classification based on a ``risk-coverage'' model. They proved that, for the optimal classifier and selector, there exists a natural trade-off between the performance of the classifier on the selected subset and the size of the subset.

            Prior work has either considered the ``realizable'' case \citep{el2012active,pmlr-v130-gangrade21a}, where there exists a ``perfect'' classifier-selector pair that does not make any error, or endowed the learner with a rejection mechanism that is based on heuristic rules or confidence scores \citep{geifman2017selective,pmlr-v206-pugnana23a}. For the ``agnostic'' case, where no perfect classifier-selector pair exists, few works had been done on model-based selective learning \citep{wiener2011agnostic,wiener2015agnostic,JMLR:v20:17-147}. More importantly, these works do not simultaneously guarantee computational efficiency together with good performance with respect to the optimal classifier and selector.
            
            We consider a more general formulation of \textbf{agnostic selective classification} under the PAC-learning semantics in Definition \ref{def:agnostic-conditional-classification}. In particular, we do not make any assumptions on the labels while the performance of the learned classifier and selector are guaranteed to be close to the optimal solution. 
            
            \begin{definition}[Agnostic Conditional Classification]\label{def:agnostic-conditional-classification}
                Let $\distr$ be any distribution on $\reals[d]\times\booldomain$, $\conceptclass$ be a finite class of classifiers on $\reals[d]\times\booldomain$, and $\hypothesisclass = \lbr*{S\subseteq\reals[d]\cond \prob<\distr>{S}\in[a,b]}$ for $0\leq a\leq b\leq 1$. Suppose $\min_{S\in\hypothesisclass, c\in\conceptclass}\prob<(\bvar{x}, y)\sim\distr>{y\neq c(\bvar{x})\cond \bvar{x}\in S} = \opt$, for some $C > 1$. A $C$-approximate learning algorithm (or an algorithm with approximation factor $C$), given sample access to $\distr$, outputs an $S'\in\hypothesisclass$ such that $\min_{c\in\conceptclass}\prob<(\bvar{x}, y)\sim\distr>{y\neq c(\bvar{x})\cond \bvar{x}\in S'}\leq C\cdot\opt$ with high probability.
            \end{definition}

            The imposed ``population'' bounds on the subsets $S\in\hypothesisclass$ are critical. On the one hand, the lower bound, $\prob{S}\geq a$ can both prevent trivial optimal solutions such as $S' = \varnothing$ and make the selected subsets statistically meaningful. On the other hand, if the selector chooses a majority of the data, the performance advantage of the optimal solution of selective classification could vanish compared with that of the regular classification model \citep{el2010foundations,pmlr-v89-hainline19a}.

            Consider a halfspace $\hypothesis$, i.e., a subset of $\reals[d]$ such that the membership in $h$ is defined by some linear threshold function. In this work, we wish to solve the problem of agnostic conditional classification with halfspace selectors under standard normal distributions described as follows. 
            \begin{problem}[Distribution-Specific Agnostic Conditional Classification With Halfspaces]\label{prob:distribution-specific-agnostic-conditional-classification-with-halfspaces}
                Let $\distr$ be any distribution on $\reals[d]\times\booldomain$ with standard normal $\bvar{x}$-marginal on $\reals[d]$, $\conceptclass$ be a finite class of classifiers on $\reals[d]\times\booldomain$, and $\hypothesisclass$ be the class of halfspaces on $\reals[d]$ with population size in the range of $[a,b]$ for $0\leq a\leq b\leq 1$. Suppose $\min_{h\in\hypothesisclass, c\in\conceptclass}\prob<(\bvar{x}, y)\sim\distr>{y\neq c(\bvar{x})\cond \bvar{x}\in h} = \opt$, how close to $\opt$ can a polynomial-time learning algorithm achieve on $\hypothesisclass$ with high probability?
            \end{problem}

            An algorithm for Problem~\ref{prob:distribution-specific-agnostic-conditional-classification-with-halfspaces} may be leveraged to perform conditional classification for large or infinite classes $\conceptclass$ by using an algorithm for list learning of classifiers for some richer class \citep{charikar2017learning}, taking $\conceptclass$ in Problem~\ref{prob:distribution-specific-agnostic-conditional-classification-with-halfspaces} to be the list of classifiers produced by the list learning algorithm:

            \begin{definition}[Robust list learning]\label{def:robust-list-learning}
            Let $\distr=\alpha \distr^*+(1-\alpha)\tilde{\distr}$ for an \emph{inlier} distribution $\distr^*$ and \emph{outlier} distribution $\tilde{\distr}$ each supported on $\reals[d]\times\booldomain$, with $\alpha\in (0,1)$. A \emph{robust list learning} algorithm for a class of Boolean classifiers $\conceptclass$, given $\alpha$ and parameters $\epsilon,\delta\in (0,1)$, and sample access to $\distr$ such that for $(\bvar{x},b)$ in the support of $\distr^*$, $b=c^*(\bvar{x})$ for some $c^*\in\conceptclass$, runs in time $\poly{d,\frac{1}{\alpha},\frac{1}{\epsilon},\log\frac{1}{\delta}}$, and with probability $1-\delta$ returns a list of $\ell=\poly{d,\frac{1}{\alpha},\frac{1}{\epsilon},\log\frac{1}{\delta}}$ classifiers $\{h_1,\ldots,h_\ell\}$ such that for some $h_i$ in the list, $\Pr_{\distr^*}[h_i(\bvar{x})=c^*(\bvar{x})]\geq 1-\epsilon$.
            \end{definition}

            As we review (Appendix \ref{sec:robust-learning}), for sparse linear classifiers (with $s=O(1)$ nonzero coefficients), list learning from a sample of size $m=O(\frac{1}{\alpha\epsilon}(s\log d+\log\frac{1}{\delta}))$ is possible in time and list size $O((md)^s)$ \citep{juba2016conditional,mossel-sudan2016}.

        \subsection{Challenges Of Distribution-Specific Conditional Classification}

            Problem \ref{prob:distribution-specific-agnostic-conditional-classification-with-halfspaces} is similar to \textbf{agnostic linear classification}, where we seek to minimize the classification error over the vast majority of data. Agnostic linear classification has been extensively studied over decades, and it is known to be computationally hard in both distribution-free \citep{kearns1994toward} and distribution-specific settings \citep{diakonikolas2023near}. 
            
            Despite the intractability of agnostic learning, numerous distribution-specific approximation schemes have been developed with approximation factor of $\bigO{1/\sqrt{\opt}}$ or even constants \citep{frei2021agnostic,diakonikolas2020non,diakonikolas2022learning,diakonikolas2024efficient,pmlr-v139-shen21a}. Given the similarity between agnostic linear classification and Problem \ref{prob:distribution-specific-agnostic-conditional-classification-with-halfspaces}, it is natural to ask if we can leverage the existing techniques for standard agnostic classification in conditional classification. However, it is not clear how these could lead to a meaningful error guarantee for conditional classification.

            Directly, Definition \ref{def:agnostic-conditional-classification} (correspondingly, Problem \ref{prob:distribution-specific-agnostic-conditional-classification-with-halfspaces}) can be reduced to a ``one-sided'' classification problem, where we seek to minimize the error rate of the classifier on only one class. As the error rate could be extremely unbalanced across the classes, a constant factor approximation scheme for the agnostic linear classification problem may not yield approximation guarantees for the one-sided agnostic classification problem. 

            An analogous difficulty arose in ``fairness auditing'' \citep{kearns2018preventing}. In the problem of fairness auditing, instead of minimizing the classification error, we wish to verify some specific fairness criteria for a subset of the data. \citet{kearns2018preventing} showed that the auditing problem is equivalent to agnostic classification for any simple representation classes (including halfspaces) under distribution-free settings. Despite the similarity between these two problems, as well as the existence of constant factor approximation algorithms for agnostic linear classification under distributional assumptions, recent work by \citet{hsu2024distribution} showed there does not exist any nontrivial multiplicative factor approximation algorithm for auditing halfspace subgroup fairness even under Gaussian distributions. The connection in the distribution-free setting simply does not carry over to Gaussian data.

        \subsection{Our Contribution}
            Let $\opt$ be as defined in Problem \ref{prob:distribution-specific-agnostic-conditional-classification-with-halfspaces} for $\hypothesisclass$ being the class of \textbf{homogeneous} halfspace. Our first contribution is a polynomial-time $\bigO*{1/\sqrt{\opt}}$-approximation algorithm to learn a pair of classifier and selector for Problem \ref{prob:distribution-specific-agnostic-conditional-classification-with-halfspaces} with homogeneous halfspace selectors. This is the first polynomial-time algorithm for agnostic conditional/selective classification with provable approximation guarantee w.r.t.\ the optimal solution.

            \begin{remark}
                Even for homogeneous halfspace selectors, the imbalance of error rates between classes could still exist, as we will show in our hardness result that the difference between the error rates of different classes of the homogeneous halfspace always equals to the amount that the probability of either label deviates from $1/2$; see Lemma \ref{lma:classification-error-decomposition} for details.
            \end{remark}

            Our second contribution is a negative result for Problem \ref{prob:distribution-specific-agnostic-conditional-classification-with-halfspaces}. We show that agnostic conditional classification in Definition \ref{def:agnostic-conditional-classification} is at least as hard as agnostic linear classification under any distribution. With the distribution-specific hardness result of agnostic linear classification \citep{diakonikolas2023near}, we prove that no polynomial-time algorithm can achieve an error guarantee of $\opt + \bigO{1/\log^{1/2 + \alpha} d}$ for any constant $\alpha > 0$ for Problem \ref{prob:distribution-specific-agnostic-conditional-classification-with-halfspaces}. We show more generally that approximating the conditional classification objective is at least as hard as approximating the regular classification objective.

            \textbf{Organization.} In Section \ref{sec:preliminary}, we give some necessary background. We will present our algorithmic results in Section \ref{sec:conditional-classification-with-homogeneous-halfspaces}. The distribution-specific hardness result for conditional classification with general halfspaces is in Section \ref{sec:conditional-classification-with-general-halfspaces-is-hard}. In the last section, we will discuss the limitations of our results and a few possible directions for extensions.

        \subsection{Related Works}
            \textbf{Selective Learning.} Besides the results we have mentioned above, there are many works on selective classification. For the realizable cases, \citet{el2012active} reduced active learning to selective learning, and used this reduction to prove a exponential lower bound on label complexity for learning linear classifiers when using the \emph{CAL algorithm}, which is one of the main strategies for active learning in the realizable setting. \citet{pmlr-v130-gangrade21a} proposed a optimization-based selective learning framework that guarantees to maximize the classifiers' coverage with a specified one-side prediction error rate. They proved that any representation class with finite VC-dimension can be used successfully in their models. For the agnostic cases, \citet{wiener2011agnostic,wiener2015agnostic,JMLR:v20:17-147} presented a selective learning approach to learn a classifier-selector pair that is at least as competitive as the ERM of the non-selective learning task. However, the computation of both the classifier and selector in these methods relies on an agnostic learning oracle, and the selector function is not guaranteed to minimize the conditional classification error down to any approximation factor. \citet{geifman2017selective} proposed a method to design selector functions for any given deep neural network. Their selector is built upon a given heuristic scoring function for data examples and can provably guaranteed to achieve strong performance. Aside from the theoretical results, empirically, \citet{pmlr-v206-pugnana23a} developed an model-agnostic learning algorithm to learn a confidence-based selective classifier that seeks to minimize the AUC-based loss within the selected region and \citet{pmlr-v97-geifman19a} proposed the SelectiveNet architecture that simultaneously learns a pair of classifier and selector in a single neural networks with required coverage.

            \textbf{Conditional Learning.} The problem of conditional learning (including conditional classification) incorporates two sub-problems, obtaining a finite list of classifiers as well as learning a classifier-selector pair out this finite list and some class of selector functions. For the former task, a series of positive results \citep{charikar2017learning,kothari2018robust,calderon2020conditional,bakshi2021list} have been obtained under the ``list-decodable'' setting of Definition~\ref{def:robust-list-learning}. For the latter task, \citet{Juba_2016} introduced the problem of learning abductions, where they propose to learn a subgroup of the data distribution where the entropy of the labels can be minimized. In their work, they showed that subgroups characterized by $k$-DNFs can be efficiently learned in realizable cases without any distributional assumptions. Subsequent improvements were obtained for the agnostic setting \citep{zhang2017improved,10.5555/3504035.3505075}.
            
            \textbf{Learning To Abstain.} \citet{cortes2016learning} considered a different formulation of selective classification. Instead of optimizing the classification error conditioned the selected subgroup, they proposed to minimize the classification error jointly with the selector function while enforcing a cost for ``abstaining''. They designed a few convex surrogate losses to upper bound the joint classification loss in the setting that abstaining has a cost. Later works \citep{pmlr-v237-mao24a,mao2024two,mao2024theoretically} proposed new families of surrogate losses to approximate the classification loss with abstaining and proved various upper bounds classification error of any classifier-selector pair in terms of different surrogate loss measures for two different selective learning strategies.


    \section{Preliminaries}
    \label{sec:preliminary}
        We use lowercase bold font characters to represent real vectors. In addition, subscripts will be used to index the coordinates of each vector $\bvar{x}\in\reals[d]$, e.g., $\bvar{x}<i>$ represents the $i$th coordinate of vector $\bvar{x}$. 
        
        For $\bvar{x}\in\reals[d]$, let $\norm{\bvar{x}}<p> = \sbr*{\sum_{i=1}^d\bvar{x}<i>[p]}^{1/p}$ denote the $l_p$-norm of $\bvar{x}$, and $\bvar*{x} = \bvar{x}/\norm{\bvar{x}}<2>$ denote the normalized vector of $\bvar{x}$. For any matrix $A\in\reals[m\times n]$, let $\norm{A}<\mathrm{op}> = \max_{\norm{\bvar{u}}_2 = 1}\norm{A\bvar{u}}<2>$ denote the operator norm of a matrix. We will use $\innerprod{\bvar{x}}{\bvar{y}}$ to represent the inner product of $\bvar{x}, \bvar{y}\in\reals[d]$ and $\bvar{x}[\otimes k]$ to represent the outer product of $\bvar{x}\in\reals[d]$ to the $k$th degree. Further, we will write $\bvar{w}[\bot] = \lbr*{\bvar{u}\in\reals[d]\cond \innerprod{\bvar{u}}{\bvar{w}} = 0}$ as the orthogonal subspace of $\bvar{w}\in\reals[d]$, and $\bvar{x}<\bvar{w}[\bot]> = (\identity - \bvar*{w}[\otimes 2])\bvar{x}$ as the projection of $\bvar{x}\in\reals[d]$ onto $\bvar{w}[\bot]$. Additionally, we will use $\theta(\bvar{u},\bvar{w})$ to denote the angle between two vectors $\bvar{u},\bvar{w}\in\reals[d]$.

        For probabilistic notations, we use $\distr<\bvar{x}>$ to denote the marginal distribution of $\distr$ on $\bvar{x}\in\reals[d]$, $\prob<\distr>{E}$ to denote the probability of an event $E$, and $\E<\distr>{\lvar{X}}$ to denote the expectation of some statistic $\lvar{X}$ under distribution $\distr$. In particular, for an empirical sample $\distr*\sample\distr$, we use $\E<\distr*>{\lvar{X}}$ to denote the empirical average of $\lvar{X}$, i.e., $\E<\lvar{X}\sim\distr*>{\lvar{X}} = 1/\abs*{\distr*}\sum_{\lvar{X}\in\distr*}\lvar{X}$. In addition, let $\gaussian[d]$ denote the $d$-dimensional standard normal distribution. For simplicity, we may drop $\distr$ from the subscript when context is clear, i.e., we may simply write $\prob{E}, \E{f}$ for $\prob<\distr>{E}, \E<\distr>{f}$. 

        \textbf{In this paper, we denote halfspaces as a subset of $\reals[d]$ in the following way.} For any $\subsets<1>,\subsets<2>\subseteq\reals[d]$, we denote $\subsets<1>\backslash\subsets<2> = \lbr*{\bvar{x}\in\reals[d]\cond \bvar{x}\in\subsets<1>, \bvar{x}\notin\subsets<2>}$ and $\subsets* = \lbr*{\bvar{x}\in\reals[d]\cond \bvar{x}\notin\subsets}$. For any $t\in\reals, \bvar{w}\in\reals[d]$, let $\func{l}<t>:\reals[d]\rightarrow\reals$ be an affine function such that $\func{l}<t>[\bvar{x},\bvar{w}] = \innerprod{\bvar{x}}{\bvar{w}} - t$. Then, a halfspace in $\reals[d]$ with threshold $t\in\reals$ and normal vector $\bvar{w}$ is defined as $\hypothesis<t>[\bvar{w}] = \lbr*{\bvar{x}\in\reals[d]\cond \func{l}<t>[\bvar{x},\bvar{w}]\geq 0}$ (resp. $\hypothesis*<t>[\bvar{w}] = \lbr*{\bvar{x}\in\reals[d]\cond \func{l}<t>[\bvar{x},\bvar{w}]\leq 0}$). When a halfspace is homogeneous, we will drop the threshold from the subscript, i.e., when $t = 0$, we will write $\hypothesis[\bvar{w}]$ instead of $\hypothesis<0>[\bvar{w}]$.

        We will make use of an algorithm for robust list learning of sparse linear classifiers. \citet{mossel-sudan2016} observed that the approach to robust regression for the sup norm used by \citet{juba2016conditional} gives such an algorithm:

        \begin{theorem}\label{thm:robust-learn-alg}
        There is an algorithm for robust list-learning of linear classifiers with $s=O(1)$ nonzero coefficients from $m=O(\frac{1}{\alpha\epsilon}(s\log d+\log\frac{1}{\delta}))$ examples in polynomial time with list size $O((md)^s)$.
        \end{theorem}

        For completeness, we review this algorithm in Appendix \ref{sec:robust-learning}.

    \section{Conditional Classification With Homogeneous Halfspaces}
    \label{sec:conditional-classification-with-homogeneous-halfspaces}
        In this section, we present our algorithmic results for conditional classification with \textbf{homogeneous halfspaces (selectors)} on $\reals[d]$ for sparse linear classifiers or, more generally \textbf{any small set of binary classifiers} $\conceptclass$ under any distribution $\distr$ with \textbf{standard normal} $\bvar{x}$-marginals.
        
        In the case that $\conceptclass$ is finite, we find a homogeneous halfspace as the selector that minimizes its conditional classification loss, $\prob{c(\bvar{x})\neq y\cond \bvar{x}\in\hypothesis[\bvar{w}]}$, for each classifier $c\in\conceptclass$. Eventually, we choose the best classifier-selector pair as the output. 

        To extend to the case of any sparse linear classes, our strategy is to use a robust list-learning algorithm to generate a finite list $\conceptclass$, then run our conditional learning algorithm for finite classes on the obtained $\conceptclass$ to find a classifier-selector pair.

        Notice that, for homogeneous halfspaces under standard normal distributions, minimizing $\prob{c(\bvar{x})\neq y\cond \bvar{x}\in\hypothesis[\bvar{w}]}$ is equivalent to minimizing $\prob{c(\bvar{x})\neq y\isect \bvar{x}\in\hypothesis[\bvar{w}]}$ since every homogeneous halfspace $\hypothesis[\bvar{w}]$ satisfies $\prob<\bvar{x}\sim\gaussian[d]>{\bvar{x}\in \hypothesis[\bvar{w}]} = 1/2$. Hence, we will only consider minimizing $\prob{c(\bvar{x})\neq y\isect \bvar{x}\in\hypothesis[\bvar{w}]}$ in this section.
        The core challenge for our strategy is finding such a halfspace for each $c\in\conceptclass$. We give the details in the following sections.

        \subsection{Algorithm Overview}
        
        \begin{algorithm}[ht]
            \caption{Conditional Classification For Finite $\conceptclass$}\label{algo:conditional-classification-with-homogeneous-halfspaces-for-finite-classes}
            \DontPrintSemicolon
            \SetKwProg{myproc}{procedure}{}{}
            \myproc{\scshape{Ccfc}$(\distr,\conceptclass, \epsilon, \delta)$}{
                $T\leftarrow \sbr*{4d+\ln(8\abs{\conceptclass}/\delta)}/\epsilon^4$\;
                $N\leftarrow 1600\ln^2\sbr*{16T\abs{\conceptclass}/\delta}/\epsilon^2$\;
                $\distr*\leftarrow$ $\ln(4\abs{\conceptclass}T/\delta)/2\epsilon$ i.i.d. examples from $\distr$\;
                $\bvar{w}(0)\leftarrow$ any basis\;
                \For{$c\in \conceptclass$}{
                    $\distr(c)\leftarrow \distr<\bvar{x}>\times \indicator[c(\bvar{x})\neq y]$\;\label{line:data-mapping}
                    $\parameterset{W}(c)\leftarrow\text{\scshape{Psgd}}\sbr{\distr(c),T,N,\bvar{w}(0)}\union \text{\scshape{Psgd}}\sbr{\distr(c),T,N,-\bvar{w}(0)}$\;\label{line:clc-call-psgd}
                    $\bvar{w}(c)\leftarrow\argmin_{\bvar{w}\in\parameterset{W}(c)}\prob<\distr*>{\bvar{x}\in\hypothesis[\bvar{w}]\isect c(\bvar{x})\neq y}$\;\label{line:find-the-loss-minimizer}
                }
                \Return $\argmin_{\bvar{w}(c)}\prob*<\distr*>{\bvar{x}\in\hypothesis[\bvar{w}(c)]\isect c(\bvar{x})\neq y}$
            }
        \end{algorithm}
        
        In Algorithm \ref{algo:conditional-classification-with-homogeneous-halfspaces-for-finite-classes}, for each binary classifier $c\in\conceptclass$, we map the label $\lvar{y}$ from $\distr$ to $\indicator[c(\bvar{x})\neq y]$ to form a new distribution $\distr(c)$, then pass $\distr(c)$ to Algorithm \ref{algo:projected-stochastic-gradient-descent-for-minimizing-convex-surrogate-loss} to obtain a sequence of halfspaces, and only keep the halfspace $\hypothesis[\bvar{w}(c)]$ with the smallest empirical conditional classification error for this classifier $c$. The last step picks out the classifier-selector pair that performs the best among all $c\in\conceptclass$ in terms of conditional classification error estimated on an large enough empirical distribution $\distr*$. 
        
        Notably, the mapping step (line \ref{line:data-mapping}) for each $c\in\conceptclass$ essentially just creates another adversarial distribution $\distr(c)$, which is a key step to reduce the conditional classification problem to a \textbf{``one-sided'' agnostic linear classification} problem. While directly optimizing over the conditional classification loss $\prob{\bvar{x}\in\hypothesis[\bvar{w}]\isect c(\bvar{x})\neq y}$ is intractable in general, it turns out that a simple convex surrogate approximation to the classification loss  captures the ``one-sided'' nature for a standard normal distribution.
        
        \begin{algorithm}[ht]
            \caption{Projected SGD for $\loss<\distr>{\bvar{w}}$}\label{algo:projected-stochastic-gradient-descent-for-minimizing-convex-surrogate-loss}
            \DontPrintSemicolon
            \SetKwProg{myproc}{procedure}{}{}
            \myproc{\scshape{Psgd}$(\distr, T, N, \bvar{w}(0))$}{
                $\beta\leftarrow \sqrt{1/Td}$\;
                \For{$i = 1, \ldots, T$}{
                    $\distr*(i)\leftarrow$ $N$ i.i.d. samples from $\distr$\;\label{line:rejecting-sampling}
                    $\bvar{u}(i)\leftarrow \bvar{w}(i - 1) - \beta \E<(\bvar{x},\lvar{y})\sim\distr*(i)>{\func{g}<\bvar{w}(i-1)>[\bvar{x},\lvar{y}]}$\;\label{line:psgd-gradient-update}
                    $\bvar{w}(i)\leftarrow \bvar{u}(i)/\norm{\bvar{u}(i)}_2$\;\label{line:psgd-projection-step}
                }
                \Return $\sbr*{\bvarseq{w}(T)}$
            }
            
        \end{algorithm}

        Algorithm \ref{algo:projected-stochastic-gradient-descent-for-minimizing-convex-surrogate-loss} is  a variant of Stochastic Gradient Descent, and the loss function $\loss<\distr>{\bvar{w}}$ we are minimizing is a \textbf{convex surrogate approximation} of the conditional classification error, known as ReLU. We formally define our loss function with respect to the distribution $\distr$ to be $\loss<\distr>{\bvar{w}} = \E<\sbr{\bvar{x}, y}\sim\distr>{y\cdot\max\sbr*{0, \innerprod{\bvar{x}}{\bvar{w}}}}$.
        
        Inspired by \citet{diakonikolas2020learning}, the updating policy in Algorithm \ref{algo:projected-stochastic-gradient-descent-for-minimizing-convex-surrogate-loss} uses the projected gradient $\func{g}<\bvar{w}>[\bvar{x},\lvar{y}]$, defined as $\func{g}<\bvar{w}>[\bvar{x},\lvar{y}] = y\cdot\bvar{x}<\bvar{w}[\bot]>\cdot\indicator[\bvar{x}\in\hypothesis[\bvar{w}]]$. We will show in the next section that the goal of Algorithm \ref{algo:projected-stochastic-gradient-descent-for-minimizing-convex-surrogate-loss} is not minimizing $\loss<\distr>{\bvar{w}}$, but the norm of the projected gradient $\norm{\E{\func{g}<\bvar{w}>}}<2>$.
        
        Note that the objective function considered in \citet{diakonikolas2020learning} is completely different from ours so that their convergence analysis does not obviously hold for our surrogate loss $\loss<\distr>{\bvar{w}}$. Also,  our choice of $\func{g}<\bvar{w}>[\bvar{x},\lvar{y}]$ is similar to that of \citet{pmlr-v139-shen21a}. Nonetheless, the problem they were solving is agnostic linear classification and they used a quite different gradient descent policy.

        \begin{algorithm}[ht]
            \caption{Conditional Classification For Sparse Linear $\conceptclass$}\label{algo:conditional-classification-for-sparse-linear-classes}
            \DontPrintSemicolon
            \SetKwProg{myproc}{procedure}{}{}
            \myproc{\scshape{Ccslc}$(\distr, \epsilon, \delta, m)$}{
                $\conceptclass\leftarrow$\scshape{SparseList}$(\distr, m)$\; 
                \Return \scshape{Ccfc}$(\distr, \conceptclass, \epsilon,\delta)$
            }
        \end{algorithm}
        Algorithm \ref{algo:conditional-classification-for-sparse-linear-classes} solves conditional learning of sparse linear classifiers. Specifically, {\scshape{SparseList}} (cf. Algorithm \ref{alg:robust-list-learn}) generates a list of sparse linear classifiers that contains a sparse linear classifier approximating the minimizer of the conditional classification error for sparse linear classifiers with homogeneous halfspace selectors (cf.~Theorem \ref{thm:robust-list-learn}). Then, we run Algorithm \ref{algo:conditional-classification-with-homogeneous-halfspaces-for-finite-classes} on the above $\conceptclass$ to obtain the optimal classifier-selector pair.

        \subsection{Conditional Classification For Finite Classes}
        We introduce our main guarantee at first, but postpone the proof to Appendix \ref{sec:optimality-analysis-of-approximate-stationary-point} due to the page limit. As a sketch of the proof, we will see Proposition \ref{prop:gradient-projection-lower-bound} and Proposition \ref{prop:upper-bound-on-the-norm-of-statistic-relu-gradient} together indicate the optimality of Projected SGD, as captured by Lemma \ref{lma:psgd-returns-at-least-one-optimal-solution}. Combined with a standard concentration analysis, this implies our main theorem.
        \begin{theorem}[Main Theorem]\label{thm:main-theorem}
            Let $\distr$ be a distribution on $\reals[d]\times\booldomain$ with standard normal $\bvar{x}$-marginal, and $\conceptclass$ be a class of binary classifiers on $\reals[d]\times\booldomain$. If there exists a unit vector $\bvar{v}\in\reals[d]$ such that, for some sufficiently small $\epsilon\in[0, 1/e]$, $\min_{c\in\conceptclass}\prob<\sbr{\bvar{x}, y}\sim\distr>{\bvar{x}\in\hypothesis[\bvar{v}]\isect c(\bvar{x})\neq y}\leq\epsilon$, then, with at most $\bigO*{d/\epsilon^6}$ examples, Algorithm \ref{algo:conditional-classification-with-homogeneous-halfspaces-for-finite-classes} will return a $\bvar{w}(c)$, with probability at least $1 - \delta$, such that $\prob<\sbr{\bvar{x}, y}\sim\distr>{\bvar{x}\in\hypothesis[\bvar{w}(c)]\isect c(\bvar{x})\neq y} = \bigO*{\sqrt{\epsilon}}$ and run in time $\bigO{d\abs{\conceptclass}/\epsilon^6}$. 
        \end{theorem}

        The most important component that enables our approach is the following proposition, which states that, for any sub-optimal halfspace $\hypothesis[\bvar{w}]$, the projected negative gradient $\E{-\func{g}<\bvar{w}>}$ of the surrogate loss $\loss<\distr>{\bvar{w}}$ must have non-negligible projection on the normal vector of the optimal halfspace $\hypothesis[\bvar{v}]$. 
        \begin{proposition}\label{prop:gradient-projection-lower-bound}
            Let $\distr$ be a distribution on $\reals[d]\times\booldomain$ with standard normal $\bvar{x}$-marginal, and $\func{g}<\bvar{w}>[\bvar{x}, y] = y\cdot\bvar{x}<\bvar{w}[\bot]>\indicator[\bvar{x}\in\hypothesis[\bvar{w}]]$. Suppose $\bvar{v},\bvar{w}\in\reals[d]$ are unit vectors such that $\theta(\bvar{v}, \bvar{w})\in[0,\pi/2)$ and $\prob<\sbr{\bvar{x}, y}\sim\distr>{\bvar{x}\in\hypothesis[\bvar{v}]\isect y = 1}\leq\epsilon$, then, if $\prob<\sbr{\bvar{x}, y}\sim\distr>{\bvar{x}\in\hypothesis[\bvar{w}]\isect y = 1}\geq \frac{5}{2}\sbr*{\epsilon\sqrt{\ln\epsilon^{-1}}}^{1/2}$, we have $\innerprod*{\E<(\bvar{x}, y)\sim\distr>{-\func{g}<\bvar{w}>[\bvar{x},\lvar{y}]}}{\bvar*{v}<\bvar{w}[\bot]>} \geq \frac{2}{5}\epsilon\sqrt{\ln\epsilon^{-1}}$ for sufficiently small $\epsilon$.
        \end{proposition}

        \begin{figure}[ht]
            \begin{center}
                \drawerrorregion(0.6)
            \end{center}
            \caption{\textcolor{myblue}{Blue} area represents $\hypothesis[\bvar{v}]\isect\hypothesis[\bvar{w}]$, \textcolor{myorange}{orange} area represents $\hypothesis[\bvar{w}]\backslash\hypothesis[\bvar{v}]$.}\label{fig:error-region}
        \end{figure}

        We leave the formal proof to Appendix \ref{sec:optimality-analysis-of-approximate-stationary-point} due to the page limit. The proof is based on the following observation (also see Figure \ref{fig:error-region}): When a homogeneous halfspace $\hypothesis[\bvar{w}]$ is substantially sub-optimal, the probability of labels being true within the domain that the optimal halfspace $\hypothesis[\bvar{v}]$ disagrees with it, i.e. $\hypothesis[\bvar{w}]\backslash\hypothesis[\bvar{v}]$, must be large. However, the same probability cannot be too large in the optimal halfspace $\hypothesis[\bvar{v}]$ and, hence, $\hypothesis[\bvar{v}]\isect\hypothesis[\bvar{w}]$. Then, if the underlying distribution has a well-behaved $\bvar{x}$-marginal, the $l_2$ norm of the expectation of $\bvar{x}$ within that domain should also be large. 

        In fact, the observation also gives an insight into why we choose ReLU as the surrogate loss. As we are concerned about the one-sided loss, $\prob{\bvar{x}\in\hypothesis[\bvar{w}]\isect y = 1}$, we cannot make any assumption on the domain of $\hypothesis*[\bvar{w}]$, which is also the key difference between the analysis of agnostic classification and that of conditional classification. Notice that $\loss<\distr>{\bvar{w}}$ completely ``blocks'' the information from $\hypothesis*[\bvar{w}]$ so that we only need to argue about $\E{\func{g}<\bvar{w}>[\bvar{x},\lvar{y}]}$ on the domain where we have control.

        Besides, an important implication of Proposition \ref{prop:gradient-projection-lower-bound} is that, once $\theta(\bvar{v}, \bvar{w})\in[0,\pi/2)$ and $\hypothesis[\bvar{w}]$ is sub-optimal, $\E{-\func{g}<\bvar{w}>[\bvar{x},\lvar{y}]}$ always ``points'' to $\bvar{v}$. Then, the update step (line \ref{line:psgd-gradient-update}) in Algorithm \ref{algo:projected-stochastic-gradient-descent-for-minimizing-convex-surrogate-loss} will make $\theta(\bvar{v},\bvar{w})$ contractive, which will, in turn, guarantee that the assumption $\theta(\bvar{v}, \bvar{w})\in[0,\pi/2)$ is satisfied in the next iteration. This property plays a key role in proving Lemma \ref{lma:psgd-returns-at-least-one-optimal-solution}.
        
        

        To effectively utilize Proposition \ref{prop:gradient-projection-lower-bound}, we also have to show that its assumption is satisfied. That is, at least one of the weight vectors, $\bvarseq{w}(T)$, produced by Algorithm \ref{algo:projected-stochastic-gradient-descent-for-minimizing-convex-surrogate-loss} has small $\norm{\E{\func{g}<\bvar{w}>[\bvar{x},\lvar{y}]}}<2>$. We show this can be achieved within a bounded number of iterations as the proposition below.
        \begin{proposition}\label{prop:upper-bound-on-the-norm-of-statistic-relu-gradient}
            Let $\distr$ be a distribution on $\reals[d]\times\booldomain$ with standard normal $\bvar{x}$-marginal, $\func{g}<\bvar{w}>[\bvar{x}, y] = y\cdot\bvar{x}<\bvar{w}[\bot]>\cdot\indicator[\bvar{x}\in\hypothesis[\bvar{w}]]$, and $\loss<\distr>{\bvar{w}} = \E<(\bvar{x},\lvar{y})\sim\distr>{y\cdot\max(0, \innerprod{\bvar{x}}{\bvar{w}})}$. With $\beta = \sqrt{1/Td}$, after $T$ iterations, the output $\sbr*{\bvarseq{w}(T)}$ in Algorithm \ref{algo:projected-stochastic-gradient-descent-for-minimizing-convex-surrogate-loss} will satisfy $\E*<\distr*(1),\ldots, \distr*(T)\sim\distr>{1/T\sum_{i=1}^T\norm{\E<(\bvar{x},\lvar{y})\sim\distr>{\func{g}<\bvar{w}(i)>[\bvar{x},\lvar{y}]}}<2>[2]}\leq\sqrt{d/T}$. In addition, if $T \geq \sbr*{4d + \ln\sbr*{1/\delta}}/\epsilon^4$, we have $\min_{i = 1,\ldots,T}\norm{\E<(\bvar{x},\lvar{y})\sim\distr>{\func{g}<\bvar{w}(i)>[\bvar{x},\lvar{y}]}}<2>\leq \epsilon$ with probability at least $1 - \delta$.
        \end{proposition}

        We defer the formal proof to Appendix \ref{sec:convergence-analysis-of-projected-sgd}. Our technique resembles the work of \citet{diakonikolas2020learning}, which showed that, if the objective function is \textbf{bounded} and has \textbf{Lipschitz continuous} gradient, then the norm of its gradient converges in boundedly many iterations of (Projected) SGD. 

        \begin{figure}[ht]
            \begin{center}
                \begin{subfigure}[t]{0.32\textwidth}
                    \centering
                    \adjustbox{max width=0.87\textwidth}{
                        \centering
                        \drawlengthincrease
                    }
                    \caption{Weight update step (line \ref{line:psgd-gradient-update}) and projection step (line \ref{line:psgd-projection-step}) in algorithm \ref{algo:projected-stochastic-gradient-descent-for-minimizing-convex-surrogate-loss}.}\label{fig:length-increase}
                \end{subfigure}          
                \hfill
                \begin{subfigure}[t]{0.6\textwidth}
                    \centering
                    \adjustbox{max width=0.6\textwidth}{
                        \centering
                        \drawintegralspaceALT
                    }
                    \caption{\textcolor{myorange}{Orange} plane is the decision boundary of $\hypothesis[\bvar{w} ']$, while \textcolor{myblue}{blue} plane is that of $\hypothesis[\bvar{w}]$. $\derivative<\bvar{w}>\loss<\distr>{\bvar{w}}$ and $\derivative<\bvar{w}>\loss<\distr>{\bvar{w} '}$ only differs in the two \textcolor{mypink}{pink} spherical sectors, which is dominated by $\Delta\theta$.}\label{fig:relative-lipschitz-continuity}
                \end{subfigure}    
            \end{center}
            \caption{Boundedness of $\loss<\distr>{\bvar{w}(i)}$ and almost Lipschitz continuity of $\derivative<\bvar{w}>\loss<\distr>{\bvar{w}}.$}
        \end{figure}
        
        However, the magnitude of $\loss<\distr>{\bvar{w}}$ is dominated by $\norm{\bvar{w}}<2>$, which could grow unbounded after many iterations, and its gradient $\derivative<\bvar{w}>\loss<\distr>{\bvar{w}}$ has a ``jumping'' point at zero, which is not Lipschitz continuous in general. So, the key to proving Proposition \ref{prop:upper-bound-on-the-norm-of-statistic-relu-gradient} is to overcome these issues.
        
        Observe that the gradient update (line \ref{line:psgd-gradient-update}) of Algorithm \ref{algo:projected-stochastic-gradient-descent-for-minimizing-convex-surrogate-loss} will always produce $\norm{\bvar{w}(i)}<2> \geq \norm{\bvar{w}(i - 1)}<2>$, while the projection step (line \ref{line:psgd-projection-step}) of Algorithm \ref{algo:projected-stochastic-gradient-descent-for-minimizing-convex-surrogate-loss} will always make $\loss<\distr>{\bvar{w}}$ bounded, cf.\ Figure \ref{fig:length-increase}.

        On the other hand, it turns out that $\derivative<\bvar{w}>\loss<\distr>{\bvar{w}}$ is almost Lipschitz continuous under nice distributions such as a standard normal. Intuitively, if we perturb $\bvar{w}$ a little bit to change it to $\bvar{w} '$, it will only rotate the halfspace $\hypothesis[\bvar{w}]$ by a very small angle, i.e.\ $\Delta\theta = \theta(\bvar{w},\bvar{w} ')$ is small. And, it suffices to consider the difference between $\derivative<\bvar{w}>\loss<\distr>{\bvar{w}}$ and $\derivative<\bvar{w}>\loss<\distr>{\bvar{w} '}$ on a $3$-dimensional subspace as shown in figure \ref{fig:relative-lipschitz-continuity}. Now, if the density of distribution $\distr$ is not concentrated too much in any small spherical sectors in the subspace, it implies that the change of $\derivative<\bvar{w}>\loss<\distr>{\bvar{w}}$ is dominated by $\Delta\theta$ (see Figure \ref{fig:relative-lipschitz-continuity}), which is insignificant. This observation indicates that $\derivative<\bvar{w}>\loss<\distr>{\bvar{w}}$ is Lipschitz continuous under anti-concentrated distributions unless $\norm{\bvar{w}}<2>$ is extremely small.

        Given Proposition \ref{prop:gradient-projection-lower-bound} and Proposition \ref{prop:upper-bound-on-the-norm-of-statistic-relu-gradient}, we show that in the list of parameters returned by Algorithm \ref{algo:projected-stochastic-gradient-descent-for-minimizing-convex-surrogate-loss}, at least one of them is approximately optimal:
        \begin{lemma}\label{lma:psgd-returns-at-least-one-optimal-solution}
            Let $\distr$ be a distribution on $\reals[d]\times\booldomain$ with standard normal $\bvar{x}$-marginal, and $\func{g}<\bvar{w}>[\bvar{x}, y] = y\cdot\bvar{x}<\bvar{w}[\bot]>\cdot\indicator[\bvar{x}\in\hypothesis[\bvar{w}]]$. Suppose $\bvar{v}\in\reals[d]$ is a unit vectors such that $\prob<\sbr{\bvar{x}, y}\sim\distr>{\bvar{x}\in\hypothesis[\bvar{v}]\isect y = 1}\leq\epsilon$, if $T \geq \sbr*{4d + \ln\sbr*{2/\delta}}/\epsilon^4$, $N \geq 1600\ln^2\sbr*{4T/\delta}/\epsilon^2$, and $\theta(\bvar{v},\bvar{w}(0))\in[0, \pi/2)$, it holds that at least one of $\bvar{w}\in\parameterset{W} = \lbr*{\bvarseq{w}(T)}$ returned by Algorithm \ref{algo:projected-stochastic-gradient-descent-for-minimizing-convex-surrogate-loss} satisfies $\prob<\sbr{\bvar{x}, y}\sim\distr>{\bvar{x}\in\hypothesis[\bvar{w}]\isect y = 1}\leq \frac{5}{2}\sbr*{\epsilon\sqrt{\ln\epsilon^{-1}}}^{1/2}$ with probability at least $1 - \delta$ for some sufficiently small $\epsilon\in[0,1/e]$.
        \end{lemma}
        We defer the formal proof to Appendix \ref{sec:optimality-analysis-of-approximate-stationary-point}, but sketch the idea here. Observe that combining the negation of Proposition \ref{prop:gradient-projection-lower-bound} and Proposition \ref{prop:upper-bound-on-the-norm-of-statistic-relu-gradient} already yields Lemma \ref{lma:psgd-returns-at-least-one-optimal-solution}. So, all we need to do is make sure that the assumption $\theta(\bvar{v},\bvar{w})\in[0, \pi/2)$ in Proposition \ref{prop:gradient-projection-lower-bound} is satisfied. 
        
        Notice that, in the sequence of parameters $\bvarseq{w}(T)$ returned by Algorithm \ref{algo:projected-stochastic-gradient-descent-for-minimizing-convex-surrogate-loss}, every $\bvar{w}(i)$ must be significantly sub-optimal until we see a $\bvar{w}$ such that $\prob{\bvar{x}\in\hypothesis[\bvar{w}]\isect y = 1}\leq \frac{5}{2}\sbr*{\epsilon\sqrt{\ln\epsilon^{-1}}}^{1/2}$. If such a sub-optimal halfspace $\hypothesis[\bvar{w}(i)]$ also satisfies $\theta(\bvar{v}, \bvar{w}(i))\in[0,\pi/2)$, its projected negative gradient $\E{-\func{g}<\bvar{w}(i)>}$ must has positive projection on $\bvar*{v}<\bvar{w}[\bot]>$ by Proposition \ref{prop:gradient-projection-lower-bound}. Using such a $\E{-\func{g}<\bvar{w}(i)>}$ to update $\bvar{w}(i)$ in Algorithm \ref{algo:projected-stochastic-gradient-descent-for-minimizing-convex-surrogate-loss} will always produce $\theta(\bvar{v}, \bvar{w}(i + 1))\leq \theta(\bvar{v},\bvar{w}(i))$. Thus, by an inductive argument, we can show that the first $\bvar{w}(t)$ such that $\norm{\E{\func{g}<\bvar{w}(t)>}}<2>< \frac{2}{5}\epsilon\sqrt{\ln\epsilon^{-1}}$ must satisfy $\theta(\bvar{v},\bvar{w}(t))\in[0, \pi/2)$, which enables the application of Proposition \ref{prop:gradient-projection-lower-bound}.

        \subsection{Generalization To Sparse Linear Classeifiers}
        Although Algorithm \ref{algo:conditional-classification-with-homogeneous-halfspaces-for-finite-classes} only applies to finite $\conceptclass$, we can generalize our approach to work with infinite classes of classifiers whenever they are list-learnable (Definition \ref{def:robust-list-learning}); for example, sparse linear classifiers are list-learnable in polynomial time. We present the corresponding performance guarantee for Algorithm \ref{algo:conditional-classification-for-sparse-linear-classes} as follows while deferring the formal proof to Appendix \ref{sec:analysis-of-algorithm-3} due to the page limit.
        \begin{theorem}\label{thm:generalization-to-sparse-linear-classes}
            Let $\distr$ be a distribution on $\reals[d]\times\booldomain$ with standard normal $\bvar{x}$-marginal, and $\conceptclass$ be a class of sparse linear classifiers on $\reals[d]\times\booldomain$ with sparsity $s = \bigO{1}$. If there exists a unit vector $\bvar{v}\in\reals[d]$ such that, for some sufficiently small $\epsilon\in[0, 1/e]$, $\min_{c\in\conceptclass}\prob<\sbr{\bvar{x}, y}\sim\distr>{\bvar{x}\in\hypothesis[\bvar{v}]\isect c(\bvar{x})\neq y}\leq\epsilon$, then, with at most $\poly{d, 1/\epsilon, 1/\delta}$ examples, Algorithm \ref{algo:conditional-classification-for-sparse-linear-classes} will return a $\bvar{w}(c)$, with probability at least $1 - \delta$, such that $\prob<\sbr{\bvar{x}, y}\sim\distr>{\bvar{x}\in\hypothesis[\bvar{w}(c)]\isect c(\bvar{x})\neq y} = \bigO*{\sqrt{\epsilon}}$ and run in time $\poly{d, 1/\epsilon, 1/\delta}$. 
        \end{theorem}

    \section{Conditional Classification With General Halfspaces Is Hard}
    \label{sec:conditional-classification-with-general-halfspaces-is-hard}
        In this section, we show that it is computationally hard to obtain a small additive error for conditional classification with general halfspaces for any finite class of classifiers $\conceptclass$, even under distributions with standard normal $\bvar{x}$-marginals. Specifically, we show that, for each classifier $c\in\conceptclass$, approximating the optimal conditional classification loss over the class of general halfspaces on $\reals[d]$ with an \textbf{additive error} is at least as hard as achieving the same additive error for agnostic linear classification, which is known to be computationally hard \citep{diakonikolas2023near}. Further, we show that any $(1+\alpha)$-approximation algorithm for conditional classification implies an $(1 + \alpha)$-approximation algorithm for standard classification, down to polynomially small losses. (The converse is not known to hold.)
        
        The hardness of distribution-specific conditional classification is based on the sub-exponential hardness of ``continuous Learning With Errors'' (cLWE), which is a variant of the ``Learning With Errors'' (LWE) assumption. Informally speaking, in the problem of LWE, we are given labelled examples from two hypothesis cases. In one case, the labels are biased by some secret vector, while, in another case, the labels are generated uniformly at random. We wish to distinguish between these cases. We formally define the problem of LWE \citep{regev2009lattices}, following \citet{diakonikolas2023near}:
        \begin{definition}[Learning With Errors]
            For $m,d\in\naturals$, $q\in\reals<+>$, let $\distr<sample>,\distr<secret>,\distr<noise>$ be distributions on $\reals[d], \reals[d], \reals$ respectively. In the LWE$(m, \distr<sample>,\distr<secret>,\distr<noise>, \mathrm{mod}_q)$ problem, with $m$ independent samples $\lbr*{(\bvar{x}(1),\lvar{y}(1)), \ldots, (\bvar{x}(m), \lvar{y}(m))}$, we want to distinguish between the following two cases:
            \begin{compactitem}
                \item \textbf{Alternative hypothesis}: each $(\bvar{x}(i), \lvar{y}(i))$ is generated as $\lvar{y}(i) = \mathrm{mod}_q(\innerprod*{\bvar{x}(i)}{\bvar{s}} + z)$, where $\bvar{x}(i)\sim\distr<sample>, \bvar{s}\sim\distr<secret>, z\sim\distr<noise>$.
                \item \textbf{Null hypothesis}: each $\lvar{y}(i)$ is sampled uniformly at random on the support of its marginal distribution in the alternative hypothesis, independent of $\bvar{x}(i)\sim\distr<sample>$.
            \end{compactitem}
        \end{definition}
        
        An algorithm is said to be able to \emph{solve the LWE problem with $\Delta$ advantage} if the probability that the algorithm outputs ``alternative hypothesis'' is $\Delta$ larger than the probability that it outputs ``null hypothesis'' when the given data is sampled from the alternative hypothesis distribution.
        
        Let $\sphere[d -1]:=\lbr{\bvar{x}\in \reals[d] \cond \norm{\bvar{x}}_2= 1}$, $\reals<q>:=[0, q)$, and $\mathrm{mod}_q:\reals[d]\rightarrow \reals<q>[d]$ to be the function that applies $\mathrm{mod}_q$ operation on each coordinate of $\bvar{x}$. Essentially, the hardness of cLWE is based on the sub-exponential hardness of LWE (see Appendix \ref{sec:analysis-of-hardness-results}). We formally state the assumption of sub-exponential hardness of cLWE as follows. 
        
        \begin{assumption}[\citep{gupte2022continuous,diakonikolas2023near} Sub-exponential cLWE Assumption]\label{asp:sub-exponential-assumption-of-clwe}
            For any $d\in\naturals$, any constants $\kappa\in\naturals, \alpha\in(0, 1), \beta\in\reals<+>$ and any $\log^\beta d\leq k\leq Cd$ where $C>0$ is a sufficiently small universal constant, the problem LWE$(d^{\bigO{k^\alpha}}, \gaussian[d], \sphere[d-1], \gaussian[][0][\sigma^2], \mathrm{mod}_T)$ over $\reals[d]$ with $\sigma\geq k^{-\kappa}$ and $T = 1/C'\sqrt{k\log d}$, where $C'>0$ is a sufficiently large universal constant, cannot be solved in time $d^{\bigO{k^\alpha}}$ with $d^{-\bigO{k^\alpha}}$ advantage.
        \end{assumption} 
    
        For simplicity, we define $y \equiv \indicator[c(\bvar{x}) \neq \lvar{y} ']$ for $(\bvar{x},y')\sim\distr '$ and construct the distribution $(\bvar{x}, y)\sim\distr$. Notice that, in agnostic settings, since $\distr'$ is worst case, $\distr$ is also worst case. Therefore, this replacement does not affect the difficulty of the problems we consider.

        Normally, for the problem of agnostic classification, one would consider its loss function to be the expected disagreement between the classifier and the labelling. However, it is more convenient for us to consider a labelling $y = 1$ as an ``occurrence of an error'' and, hence, define the loss function in terms of agreement to compare with the conditional classification loss.  Specifically, for any binary classifier as a subset $S\subseteq\reals[d]$ and any distribution $\distr$ on $\reals[d]\times\booldomain$, we define the classification loss:
        \begin{equation}
            \err<\distr>[S] = \prob<(\bvar{x}, \lvar{y})\sim\distr>{y = \indicator[\bvar{x}\in S]}.\label{eq:definition-of-classification-loss}
        \end{equation}
        Note that this definition of classification loss is essentially the same as the traditional one defined in terms of disagreement since we can convert from one to another by simply negating the labelling.

        Analogously, for any binary classifiers as subsets $S, T\subseteq\reals[d]$ and any distribution $\distr$ on $\reals[d]\times\booldomain$, we denote the conditional classification loss by 
        \begin{equation}
            \err<\distr| T>[S] = \prob<(\bvar{x}, \lvar{y})\sim\distr>{y = \indicator[\bvar{x}\in S]\cond \bvar{x}\in T}.\label{eq:notation-of-conditional-classification-loss}
        \end{equation}
        For simplicity, we write $\err<\distr| T>$ instead of $\err<\distr| T>[S]$ when $S\equiv T$.

        We state our distribution-specific hardness result for conditional classification as Theorem \ref{thm:hardness-of-conditional-classification}.
        \begin{theorem}[Hardness Of Conditional Classification]\label{thm:hardness-of-conditional-classification}
            Let $\distr$ be any distribution on $\reals[d]\times\booldomain$ with standard normal $\bvar{x}$-marginals, $\hypothesisclass$ be the class of halfspaces on $\reals[d]$, and define $\hypothesisclass<\distr>[a,b] = \lbr{\hypothesis<t>[\bvar{w}]\in\hypothesisclass\cond \prob<\bvar{x}\sim\distr<\bvar{x}>>{\bvar{x}\in\hypothesis<t>[\bvar{w}]}\in[\lvar{a},\lvar{b}]}$ for any $0\leq a\leq b\leq 1$. Under Assumption \ref{asp:sub-exponential-assumption-of-lwe}, for any constant $\alpha\in(0,2)$, $\gamma > 1/2$ and any $c/\sqrt{d\log d}\geq \epsilon\leq1/\log^\gamma d$ where $c$ is a sufficiently large constant, there is no algorithm that can find a halfspace $\hypothesis<t'>[\bvar{w}]\in\hypothesisclass<\distr>[a,b]$ such that $\err<\distr|\hypothesis<t'>[\bvar{w}]> \leq \min_{\hypothesis<t>[\bvar{u}]\in\hypothesisclass<\distr>[a,b]}\err<\distr|\hypothesis<t>[\bvar{u}]> + \epsilon$ and runs in time $d^{\bigO{1/\sbr*{\epsilon\sqrt{\log d}}^\alpha}}$.
        \end{theorem}

        Theorem \ref{thm:hardness-of-conditional-classification} is actually a simple consequence of Proposition \ref{prop:reducing-linear-classification-to-conditional-classification} and Lemma \ref{lma:distribution-specific-hardness-for-agnostic-learning}, where the former one shows that conditional classification is at least as hard as agnostic classification and the latter one states the hardness of agnostically learning halfspaces. 
        
        Our main contribution is Proposition \ref{prop:reducing-linear-classification-to-conditional-classification}, but before getting into it, we first show a simple but critical observation that reveals the relationship between $\err<\distr>[S]$ and $\err<\distr| S>$. That is, the loss of agnostic classification can be explictly expressed by the loss of conditional classification.
        \begin{lemma}[Classification Error Decomposition]\label{lma:classification-error-decomposition}
            Let $\distr$ be any distribution on $\reals[d]\times\booldomain$ and $\subsets$ be any subset of $\reals[d]$, there are $\err<\distr>[\subsets] =2\err<\distr|\subsets>\prob<\distr>{\bvar{x}\in\subsets} + \prob<\distr>{y = 0} - \prob<\distr>{\bvar{x}\in\subsets}$ as well as $\err<\distr>[\subsets] =2\err<\distr|\subsets*>[\subsets]\prob<\distr>{\bvar{x}\in\subsets*} + \prob<\distr>{y = 1} - \prob<\distr>{\bvar{x}\in\subsets*}$.
        \end{lemma}

        Due to page limits, we defer its proof to Appendix \ref{sec:analysis-of-hardness-results}. Lemma \ref{lma:classification-error-decomposition} is a powerful result since it allows us to establish a reduction from classification to conditional classification. 
        
        Briefly speaking, if we know $\prob{\bvar{x}\in\subsets[*]}$ for some optimal solution $\subsets[*]$ to the agnostic classification problem, we can approximate $\err<\distr>[\subsets[*]]$ by approximating its conditional classification loss, i.e.\ $\err<\distr|\subsets[*]>$. Even though we do not know $\prob{\bvar{x}\in\subsets[*]}$, we can guess a small range containing $\prob{\bvar{x}\in\subsets[*]}$, and enforce such a constraint just as in Definition \ref{def:agnostic-conditional-classification}. Then, we sweep over all such small intervals and one of the instances being solved must include $\prob{\bvar{x}\in\subsets}$. Once we take these intervals small enough, it won't incur a significant error. We use this strategy to prove Proposition \ref{prop:reducing-linear-classification-to-conditional-classification}, but the formal proof is deferred to Appendix \ref{sec:analysis-of-hardness-results} due to the page limit.

        \begin{proposition}[Reducing Classification To Conditional Classification]\label{prop:reducing-linear-classification-to-conditional-classification}
            Let $\distr$ be any distribution on $\reals[d]\times\booldomain$, $\hypothesisclass$ be any subset of the power set of $\reals[d]$ closed under complement, and define $\hypothesisclass<\distr>[a,b] = \lbr{\subsets\in\hypothesisclass\cond \prob<\distr>{\bvar{x}\in\subsets}\in[\lvar{a},\lvar{b}]}$ for any $0\leq a\leq b\leq 1$. For any $0\leq\lvar{a}\leq\lvar{b}\leq1$ and $\epsilon,\delta > 0$, given sample access to $\distr$, if there exists an algorithm $\algo<1>[\epsilon, \delta, a,b]$ runs in time $\poly{d, 1/\epsilon,1/\delta}$, and outputs a subset $\subsets<1>\in\hypothesisclass<\distr>[a,b]$ such that $\err<\distr|\subsets<1>> \leq \min_{\subsets\in\hypothesisclass<\distr>[a,b]}\err<\distr|\subsets> + \epsilon$ with probability as least $1 - \delta$, there exists another algorithm $\algo<2>[\epsilon, \delta]$, runs in time $\poly{d, 1/\epsilon,1/\delta}$, and outputs a subset $\subsets<2>\in\hypothesisclass$ such that $\err<\distr>[\subsets<2>] \leq \min_{\subsets\in\hypothesisclass}\err<\distr>[\subsets] + 6\epsilon$ with probability at least $1 - \delta$.
        \end{proposition}

        Furthermore, the following distribution-specific hardness result states that agnostically learning halfspaces up to small additive error is computationally hard.
        \begin{lemma}[Corollary 3.2 of \citet{diakonikolas2023near}]\label{lma:distribution-specific-hardness-for-agnostic-learning}
            Let $\distr$ be any distribution on $\reals[d]\times\booldomain$ with standard normal $\bvar{x}$-marginals, and $\hypothesisclass$ be the class of halfspaces on $\reals[d]$. Under Assumption \ref{asp:sub-exponential-assumption-of-lwe}, for any constant $\alpha\in(0,2)$, $\gamma > 1/2$ and any $c/\sqrt{d\log d}\geq \epsilon\leq1/\log^\gamma d$ where $c$ is a sufficiently large constant, there is no algorithm that can find a halfspace $\hypothesis<t'>[\bvar{v}]\in\hypothesisclass$ such that $\err<\distr>[\hypothesis<t'>[\bvar{v}]] \leq \min_{\hypothesis<t>[\bvar{u}]\in\hypothesisclass}\err<\distr>[\hypothesis<t>[\bvar{u}]] + \epsilon$ and runs in time $d^{\bigO{1/\sbr*{\epsilon\sqrt{\log d}}^\alpha}}$.
        \end{lemma}

        Since Proposition \ref{prop:reducing-linear-classification-to-conditional-classification}  holds for halfspaces on $\reals[d]$,  conditional learning has at least the same hardness by combining Proposition \ref{prop:reducing-linear-classification-to-conditional-classification} and Lemma \ref{lma:distribution-specific-hardness-for-agnostic-learning}.

        Analogously, a reduction in multiplicative form can also be obtained using a similar analysis to that in the proof of Proposition \ref{prop:reducing-linear-classification-to-conditional-classification}. In particular, we show that if there exists a multiplicative approximation algorithm for conditional classification with factor $1+\alpha$, there must exist another multiplicative approximation algorithm for classification in agnostic setting with the same factor $1 + \alpha$.

        \begin{claim}[Reduction In Multiplicative Form]\label{clm:reduction-in-multiplicative-form}
            Let $\distr$ be any distribution on $\reals[d]\times\booldomain$, $\hypothesisclass$ be any subset of the power set of $\reals[d]$ closed under complement, and define $\hypothesisclass<\distr>[a,b] = \lbr{\subsets\in\hypothesisclass\cond \prob<\distr>{\bvar{x}\in\subsets}\in[\lvar{a},\lvar{b}]}$ for any $0\leq a\leq b\leq 1$. For any $0\leq\lvar{a}\leq\lvar{b}\leq1$, $\alpha,\epsilon,\delta > 0$, given sample access to $\distr$, if there exists an algorithm $\algo<1>[\alpha, \delta, a,b]$, runs in time $\poly{d, 1/\alpha,1/\delta}$, and outputs a subset $\subsets<1>\in\hypothesisclass<\distr>[a,b]$ such that $\err<\distr|\subsets<1>> \leq (1+\alpha)\min_{\subsets\in\hypothesisclass<\distr>[a,b]}\err<\distr|\subsets>$ with probability as least $1 - \delta$, there exists another algorithm $\algo<2>[\alpha,\epsilon, \delta]$, runs in time $\poly{d, 1/\alpha,1/\epsilon,1/\delta}$, and outputs a subset $\subsets<2>\in\hypothesisclass$ such that $\err<\distr>[\subsets<2>] \leq (1 + \alpha) \sbr*{\min_{\subsets\in\hypothesisclass}\err<\distr>[\subsets] + 4\epsilon}$ with probability at least $1 - \delta$.
        \end{claim}

        Again, we defer the proof to Appendix \ref{sec:analysis-of-hardness-results} because of page limits. Although there is an extra $4\epsilon$ additive error in the final guarantee of Claim \ref{clm:reduction-in-multiplicative-form}, we can afford to take $\epsilon$ polynomially small w.r.t.\ $d,\alpha,\delta$, thus obtaining the multiplicative error guarantee down to polynomially small error. Informally we observe that Proposition \ref{prop:reducing-linear-classification-to-conditional-classification} and Claim \ref{clm:reduction-in-multiplicative-form} indicate that any form of approximation algorithm for conditional classification yields an approximation algorithm of the same factor for agnostic classification. In the case of multiplicative approximation in particular, the reverse is not known and we observe that it might be strictly harder to approximate the conditional classification objective.

    \section{Limitations And Future Work}
        Our algorithmic result is limited in three aspects. First and foremost, the restriction of selectors to homogeneous halfspaces is a major drawback especially for the task of conditional classification. Indeed, the advantage of conditional classification with halfspaces compared with regular linear classification really shines when we have the ability to select a minority of the data distribution. Therefore, even with guarantees worse than $\bigO*{\sqrt{\epsilon}}$, moving from homogeneous halfspaces to general halfspaces would constitute a significant advace. Another limitation of our result is the strong assumption on the marginal distribution. Real-world data almost never has standard normal marginals, and testing for a standard normal distribution is costly. Hence, it's worth trying to extend our result to more general classes of distributions, such as log-concave distributions. Last but not the least, one can also try to improve our error guarantee under the current setting as the error guarantee $\bigO{\sqrt{\epsilon}}$ appears sub-optimal.
    
    \bibliographystyle{abbrvnat} 
    \bibliography{refs} 

\begin{thebibliography}{44}
\providecommand{\natexlab}[1]{#1}
\providecommand{\url}[1]{\texttt{#1}}
\expandafter\ifx\csname urlstyle\endcsname\relax
  \providecommand{\doi}[1]{doi: #1}\else
  \providecommand{\doi}{doi: \begingroup \urlstyle{rm}\Url}\fi

\bibitem[Bakshi and Kothari(2021)]{bakshi2021list}
A.~Bakshi and P.~K. Kothari.
\newblock List-decodable subspace recovery: Dimension independent error in polynomial time.
\newblock In \emph{Proceedings of the 2021 ACM-SIAM Symposium on Discrete Algorithms (SODA)}, pages 1279--1297. SIAM, 2021.

\bibitem[Blumer et~al.(1989)Blumer, Ehrenfeucht, Haussler, and Warmuth]{blumer1989learnability}
A.~Blumer, A.~Ehrenfeucht, D.~Haussler, and M.~K. Warmuth.
\newblock Learnability and the vapnik-chervonenkis dimension.
\newblock \emph{Journal of the ACM (JACM)}, 36\penalty0 (4):\penalty0 929--965, 1989.

\bibitem[Calderon et~al.(2020)Calderon, Juba, Li, Li, and Ruan]{calderon2020conditional}
D.~Calderon, B.~Juba, S.~Li, Z.~Li, and L.~Ruan.
\newblock Conditional linear regression.
\newblock In \emph{International Conference on Artificial Intelligence and Statistics}, pages 2164--2173. PMLR, 2020.

\bibitem[Charikar et~al.(2017)Charikar, Steinhardt, and Valiant]{charikar2017learning}
M.~Charikar, J.~Steinhardt, and G.~Valiant.
\newblock Learning from untrusted data.
\newblock In \emph{Proceedings of the 49th Annual ACM SIGACT Symposium on Theory of Computing}, pages 47--60, 2017.

\bibitem[Chow(1970)]{1054406}
C.~Chow.
\newblock On optimum recognition error and reject tradeoff.
\newblock \emph{IEEE Transactions on Information Theory}, 16\penalty0 (1):\penalty0 41--46, 1970.
\newblock \doi{10.1109/TIT.1970.1054406}.

\bibitem[Chow(1957)]{5222035}
C.~K. Chow.
\newblock An optimum character recognition system using decision functions.
\newblock \emph{IRE Transactions on Electronic Computers}, EC-6\penalty0 (4):\penalty0 247--254, 1957.
\newblock \doi{10.1109/TEC.1957.5222035}.

\bibitem[Cortes et~al.(2016)Cortes, DeSalvo, and Mohri]{cortes2016learning}
C.~Cortes, G.~DeSalvo, and M.~Mohri.
\newblock Learning with rejection.
\newblock In \emph{Algorithmic Learning Theory: 27th International Conference, ALT 2016, Bari, Italy, October 19-21, 2016, Proceedings 27}, pages 67--82. Springer, 2016.

\bibitem[Devroye and Lugosi(2001)]{devroye2001combinatorial}
L.~Devroye and G.~Lugosi.
\newblock \emph{Combinatorial methods in density estimation}.
\newblock Springer Science \& Business Media, 2001.

\bibitem[Diakonikolas et~al.(2020{\natexlab{a}})Diakonikolas, Kane, Kontonis, Tzamos, and Zarifis]{diakonikolas2020polynomial}
I.~Diakonikolas, D.~M. Kane, V.~Kontonis, C.~Tzamos, and N.~Zarifis.
\newblock A polynomial time algorithm for learning halfspaces with tsybakov noise.
\newblock \emph{arXiv preprint arXiv:2010.01705}, 2020{\natexlab{a}}.

\bibitem[Diakonikolas et~al.(2020{\natexlab{b}})Diakonikolas, Kontonis, Tzamos, and Zarifis]{diakonikolas2020learning}
I.~Diakonikolas, V.~Kontonis, C.~Tzamos, and N.~Zarifis.
\newblock Learning halfspaces with massart noise under structured distributions.
\newblock In \emph{Conference on Learning Theory}, pages 1486--1513. PMLR, 2020{\natexlab{b}}.

\bibitem[Diakonikolas et~al.(2020{\natexlab{c}})Diakonikolas, Kontonis, Tzamos, and Zarifis]{diakonikolas2020non}
I.~Diakonikolas, V.~Kontonis, C.~Tzamos, and N.~Zarifis.
\newblock Non-convex sgd learns halfspaces with adversarial label noise.
\newblock \emph{Advances in Neural Information Processing Systems}, 33:\penalty0 18540--18549, 2020{\natexlab{c}}.

\bibitem[Diakonikolas et~al.(2022)Diakonikolas, Kontonis, Tzamos, and Zarifis]{diakonikolas2022learning}
I.~Diakonikolas, V.~Kontonis, C.~Tzamos, and N.~Zarifis.
\newblock Learning general halfspaces with adversarial label noise via online gradient descent.
\newblock In \emph{International Conference on Machine Learning}, pages 5118--5141. PMLR, 2022.

\bibitem[Diakonikolas et~al.(2023)Diakonikolas, Kane, and Ren]{diakonikolas2023near}
I.~Diakonikolas, D.~Kane, and L.~Ren.
\newblock Near-optimal cryptographic hardness of agnostically learning halfspaces and relu regression under gaussian marginals.
\newblock In \emph{International Conference on Machine Learning}, pages 7922--7938. PMLR, 2023.

\bibitem[Diakonikolas et~al.(2024)Diakonikolas, Kane, Kontonis, Liu, and Zarifis]{diakonikolas2024efficient}
I.~Diakonikolas, D.~Kane, V.~Kontonis, S.~Liu, and N.~Zarifis.
\newblock Efficient testable learning of halfspaces with adversarial label noise.
\newblock \emph{Advances in Neural Information Processing Systems}, 36, 2024.

\bibitem[El-Yaniv and Wiener(2012)]{el2012active}
R.~El-Yaniv and Y.~Wiener.
\newblock Active learning via perfect selective classification.
\newblock \emph{Journal of Machine Learning Research}, 13\penalty0 (2), 2012.

\bibitem[El-Yaniv et~al.(2010)]{el2010foundations}
R.~El-Yaniv et~al.
\newblock On the foundations of noise-free selective classification.
\newblock \emph{Journal of Machine Learning Research}, 11\penalty0 (5), 2010.

\bibitem[Frei et~al.(2021)Frei, Cao, and Gu]{frei2021agnostic}
S.~Frei, Y.~Cao, and Q.~Gu.
\newblock Agnostic learning of halfspaces with gradient descent via soft margins.
\newblock In \emph{International Conference on Machine Learning}, pages 3417--3426. PMLR, 2021.

\bibitem[Gangrade et~al.(2021)Gangrade, Kag, and Saligrama]{pmlr-v130-gangrade21a}
A.~Gangrade, A.~Kag, and V.~Saligrama.
\newblock Selective classification via one-sided prediction.
\newblock In A.~Banerjee and K.~Fukumizu, editors, \emph{Proceedings of The 24th International Conference on Artificial Intelligence and Statistics}, volume 130 of \emph{Proceedings of Machine Learning Research}, pages 2179--2187. PMLR, 13--15 Apr 2021.
\newblock URL \url{https://proceedings.mlr.press/v130/gangrade21a.html}.

\bibitem[Geifman and El-Yaniv(2017)]{geifman2017selective}
Y.~Geifman and R.~El-Yaniv.
\newblock Selective classification for deep neural networks.
\newblock \emph{Advances in neural information processing systems}, 30, 2017.

\bibitem[Geifman and El-Yaniv(2019)]{pmlr-v97-geifman19a}
Y.~Geifman and R.~El-Yaniv.
\newblock {S}elective{N}et: A deep neural network with an integrated reject option.
\newblock In K.~Chaudhuri and R.~Salakhutdinov, editors, \emph{Proceedings of the 36th International Conference on Machine Learning}, volume~97 of \emph{Proceedings of Machine Learning Research}, pages 2151--2159. PMLR, 09--15 Jun 2019.
\newblock URL \url{https://proceedings.mlr.press/v97/geifman19a.html}.

\bibitem[Gelbhart and El-Yaniv(2019)]{JMLR:v20:17-147}
R.~Gelbhart and R.~El-Yaniv.
\newblock The relationship between agnostic selective classification, active learning and the disagreement coefficient.
\newblock \emph{Journal of Machine Learning Research}, 20\penalty0 (33):\penalty0 1--38, 2019.
\newblock URL \url{http://jmlr.org/papers/v20/17-147.html}.

\bibitem[Gupte et~al.(2022)Gupte, Vafa, and Vaikuntanathan]{gupte2022continuous}
A.~Gupte, N.~Vafa, and V.~Vaikuntanathan.
\newblock Continuous lwe is as hard as lwe \& applications to learning gaussian mixtures.
\newblock In \emph{2022 IEEE 63rd Annual Symposium on Foundations of Computer Science (FOCS)}, pages 1162--1173. IEEE, 2022.

\bibitem[Hainline et~al.(2019)Hainline, Juba, Le, and Woodruff]{pmlr-v89-hainline19a}
J.~Hainline, B.~Juba, H.~S. Le, and D.~Woodruff.
\newblock Conditional sparse $l_p$-norm regression with optimal probability.
\newblock In K.~Chaudhuri and M.~Sugiyama, editors, \emph{Proceedings of the Twenty-Second International Conference on Artificial Intelligence and Statistics}, volume~89 of \emph{Proceedings of Machine Learning Research}, pages 1042--1050. PMLR, 16--18 Apr 2019.
\newblock URL \url{https://proceedings.mlr.press/v89/hainline19a.html}.

\bibitem[Hanczar and Dougherty(2008)]{hanczar2008classification}
B.~Hanczar and E.~R. Dougherty.
\newblock Classification with reject option in gene expression data.
\newblock \emph{Bioinformatics}, 24\penalty0 (17):\penalty0 1889--1895, 2008.

\bibitem[Hanneke(2016)]{hanneke2016optimal}
S.~Hanneke.
\newblock The optimal sample complexity of pac learning.
\newblock \emph{Journal of Machine Learning Research}, 17\penalty0 (38):\penalty0 1--15, 2016.

\bibitem[Haussler(1988)]{haussler1988quantifying}
D.~Haussler.
\newblock Quantifying inductive bias: Ai learning algorithms and valiant's learning framework.
\newblock \emph{Artificial intelligence}, 36\penalty0 (2):\penalty0 177--221, 1988.

\bibitem[Hsu et~al.(2024)Hsu, Huang, and Juba]{hsu2024distribution}
D.~Hsu, J.~Huang, and B.~Juba.
\newblock Distribution-specific auditing for subgroup fairness.
\newblock In \emph{5th Symposium on Foundations of Responsible Computing (FORC 2024)}. Schloss Dagstuhl--Leibniz-Zentrum f{\"u}r Informatik, 2024.

\bibitem[Juba(2016)]{Juba_2016}
B.~Juba.
\newblock Learning abductive reasoning using random examples.
\newblock \emph{Proceedings of the AAAI Conference on Artificial Intelligence}, 30\penalty0 (1), Feb. 2016.
\newblock \doi{10.1609/aaai.v30i1.10099}.
\newblock URL \url{https://ojs.aaai.org/index.php/AAAI/article/view/10099}.

\bibitem[Juba(2017)]{juba2016conditional}
B.~Juba.
\newblock Conditional sparse linear regression.
\newblock In \emph{8th Innovations in Theoretical Computer Science Conference (ITCS 2017)}. Schloss-Dagstuhl-Leibniz Zentrum f{\"u}r Informatik, 2017.

\bibitem[Juba et~al.(2018)Juba, Li, and Miller]{10.5555/3504035.3505075}
B.~Juba, Z.~Li, and E.~Miller.
\newblock Learning abduction under partial observability.
\newblock In \emph{Proceedings of the Thirty-Second AAAI Conference on Artificial Intelligence and Thirtieth Innovative Applications of Artificial Intelligence Conference and Eighth AAAI Symposium on Educational Advances in Artificial Intelligence}, AAAI'18/IAAI'18/EAAI'18. AAAI Press, 2018.
\newblock ISBN 978-1-57735-800-8.

\bibitem[Kearns et~al.(2018)Kearns, Neel, Roth, and Wu]{kearns2018preventing}
M.~Kearns, S.~Neel, A.~Roth, and Z.~S. Wu.
\newblock Preventing fairness gerrymandering: Auditing and learning for subgroup fairness.
\newblock In \emph{International conference on machine learning}, pages 2564--2572. PMLR, 2018.

\bibitem[Kearns et~al.(1994)Kearns, Schapire, and Sellie]{kearns1994toward}
M.~J. Kearns, R.~E. Schapire, and L.~M. Sellie.
\newblock Toward efficient agnostic learning.
\newblock \emph{Machine Learning}, 17:\penalty0 115--141, 1994.

\bibitem[Khan et~al.(2001)Khan, Wei, Ringner, Saal, Ladanyi, Westermann, Berthold, Schwab, Antonescu, Peterson, et~al.]{khan2001classification}
J.~Khan, J.~S. Wei, M.~Ringner, L.~H. Saal, M.~Ladanyi, F.~Westermann, F.~Berthold, M.~Schwab, C.~R. Antonescu, C.~Peterson, et~al.
\newblock Classification and diagnostic prediction of cancers using gene expression profiling and artificial neural networks.
\newblock \emph{Nature medicine}, 7\penalty0 (6):\penalty0 673--679, 2001.

\bibitem[Kothari et~al.(2018)Kothari, Steinhardt, and Steurer]{kothari2018robust}
P.~K. Kothari, J.~Steinhardt, and D.~Steurer.
\newblock Robust moment estimation and improved clustering via sum of squares.
\newblock In \emph{Proceedings of the 50th Annual ACM SIGACT Symposium on Theory of Computing}, pages 1035--1046, 2018.

\bibitem[Mao et~al.(2024{\natexlab{a}})Mao, Mohri, Mohri, and Zhong]{mao2024two}
A.~Mao, C.~Mohri, M.~Mohri, and Y.~Zhong.
\newblock Two-stage learning to defer with multiple experts.
\newblock \emph{Advances in neural information processing systems}, 36, 2024{\natexlab{a}}.

\bibitem[Mao et~al.(2024{\natexlab{b}})Mao, Mohri, and Zhong]{mao2024theoretically}
A.~Mao, M.~Mohri, and Y.~Zhong.
\newblock Theoretically grounded loss functions and algorithms for score-based multi-class abstention.
\newblock In \emph{International Conference on Artificial Intelligence and Statistics}, pages 4753--4761. PMLR, 2024{\natexlab{b}}.

\bibitem[Mao et~al.(2024{\natexlab{c}})Mao, Mohri, and Zhong]{pmlr-v237-mao24a}
A.~Mao, M.~Mohri, and Y.~Zhong.
\newblock Predictor-rejector multi-class abstention: Theoretical analysis and algorithms.
\newblock In C.~Vernade and D.~Hsu, editors, \emph{Proceedings of The 35th International Conference on Algorithmic Learning Theory}, volume 237 of \emph{Proceedings of Machine Learning Research}, pages 822--867. PMLR, 25--28 Feb 2024{\natexlab{c}}.
\newblock URL \url{https://proceedings.mlr.press/v237/mao24a.html}.

\bibitem[Mossel and Sudan(2016)]{mossel-sudan2016}
E.~Mossel and M.~Sudan.
\newblock Personal communication, 2016.

\bibitem[Pugnana and Ruggieri(2023)]{pmlr-v206-pugnana23a}
A.~Pugnana and S.~Ruggieri.
\newblock Auc-based selective classification.
\newblock In F.~Ruiz, J.~Dy, and J.-W. van~de Meent, editors, \emph{Proceedings of The 26th International Conference on Artificial Intelligence and Statistics}, volume 206 of \emph{Proceedings of Machine Learning Research}, pages 2494--2514. PMLR, 25--27 Apr 2023.
\newblock URL \url{https://proceedings.mlr.press/v206/pugnana23a.html}.

\bibitem[Regev(2009)]{regev2009lattices}
O.~Regev.
\newblock On lattices, learning with errors, random linear codes, and cryptography.
\newblock \emph{Journal of the ACM (JACM)}, 56\penalty0 (6):\penalty0 1--40, 2009.

\bibitem[Shen(2021)]{pmlr-v139-shen21a}
J.~Shen.
\newblock On the power of localized perceptron for label-optimal learning of halfspaces with adversarial noise.
\newblock In M.~Meila and T.~Zhang, editors, \emph{Proceedings of the 38th International Conference on Machine Learning}, volume 139 of \emph{Proceedings of Machine Learning Research}, pages 9503--9514. PMLR, 18--24 Jul 2021.
\newblock URL \url{https://proceedings.mlr.press/v139/shen21a.html}.

\bibitem[Wiener and El-Yaniv(2011)]{wiener2011agnostic}
Y.~Wiener and R.~El-Yaniv.
\newblock Agnostic selective classification.
\newblock \emph{Advances in neural information processing systems}, 24, 2011.

\bibitem[Wiener and El-Yaniv(2015)]{wiener2015agnostic}
Y.~Wiener and R.~El-Yaniv.
\newblock Agnostic pointwise-competitive selective classification.
\newblock \emph{Journal of Artificial Intelligence Research}, 52:\penalty0 171--201, 2015.

\bibitem[Zhang et~al.(2017)Zhang, Mathew, and Juba]{zhang2017improved}
M.~Zhang, T.~Mathew, and B.~Juba.
\newblock An improved algorithm for learning to perform exception-tolerant abduction.
\newblock In \emph{Proceedings of the AAAI Conference on Artificial Intelligence}, volume~31, 2017.

\end{thebibliography}

    \newpage
    \appendix
    \section{Review of Robust List Learning of Sparse Linear Classifiers}\label{sec:robust-learning}

     \begin{algorithm}[ht]
            \caption{Robust list learning of sparse linear classifiers}\label{alg:robust-list-learn}
            \DontPrintSemicolon
            \SetKwProg{myproc}{procedure}{}{}
            \myproc{\scshape{SparseList}$(\distr, m)$}{
                Initialize $L=\varnothing$\;
                Sample $(\bvar{x}^{(1)},y^{(1)}),\ldots,(\bvar{x}^{(m)},y^{(m)})\sim \distr$\;
                \ForEach{$(i_1,\ldots,i_s)\in [d]^s$ and $(j_1,\ldots,j_{s})\in [m]^{s}$}{
                    Put $$\bvar{w}=\left[\begin{array}{ccc}y^{(j_1)}x_{i_1}^{(j_1)}& \cdots &y^{(j_1)}x_{i_s}^{(j_1)}\\&\vdots&\\y^{(j_s)}x_{i_1}^{(j_{s})}&\cdots&y^{(j_s)}x_{i_s}^{(j_{s})}\end{array}\right]^{-1}\left[\begin{array}{c}y^{(j_1)}-\nu\\\vdots\\y^{(j_s)}-\nu\end{array}\right]$$\;
                     Concatenate $\bvar{w}$ to $L$.\;
                }
                \Return $L$
            }           
        \end{algorithm}

        For completeness, we now describe an algorithm to solve the robust list learning problem for sparse linear classifiers. It is based on the approach used in the algorithm for conditional sparse linear regression \citep{juba2016conditional}, using an observation by \citet{mossel-sudan2016}. We will prove the following:

        \begin{theorem}\label{thm:robust-list-learn}
        Algorithm \ref{alg:robust-list-learn} solves robust list-learning of linear classifiers with $s=O(1)$ nonzero coefficients from $m=O(\frac{1}{\alpha\epsilon}(s\log d+\log\frac{1}{\delta}))$ examples in polynomial time with list size $O((md)^s)$.
        \end{theorem}
        \begin{proof}
            We observe that the running time and list size of Algorithm \ref{alg:robust-list-learn} is clearly as promised. To see that it solves the problem, we first recall that results by \citet{blumer1989learnability} and \citet{hanneke2016optimal} showed that given $O(\frac{1}{\epsilon}(D+\log\frac{1}{\delta}))$ examples labeled by a class of VC-dimension $D$, any consistent hypotheses achieves error $\epsilon$ with probability $1-\delta$. In particular, halfspaces in $\mathbb{R}^d$ have VC-dimension $d$; \citet{haussler1988quantifying} observed that $s$-sparse linear classifiers in $\mathbb{R}^d$ have VC-dimension $s\log d$. Hence, if the data includes a set $S$ of at least $\Omega(\frac{1}{\epsilon}(s\log d+\log\frac{1}{\delta}))$ inliers and we find a $s$-sparse classifier that agrees with the labels on $S$, it achieves error $1-\epsilon$ on $S$ with probability $1-\delta/2$. Observe that in a sample of size $O(\frac{1}{\alpha\epsilon}(s\log d+\log\frac{1}{\delta}))$, with an $\alpha$ fraction of inliers, we indeed obtain $\Omega(\frac{1}{\epsilon}(s\log d+\log\frac{1}{\delta}))$ inliers with probability $1-\delta/2$.

            Now, suppose we write our linear threshold function with a standard threshold of $1$, and suppose are examples are drawn from $\mathbb{R}^d\times \{-1,1\}$. Then a classifier with weight vector $\bvar{w}$ labels $\bvar{x}$ with $1$ if $\langle\bvar{w},\bvar{x}\rangle\geq 1$, and labels $\bvar{x}$ with $-1$ if $\langle\bvar{w},\bvar{x}\rangle < 1$. We observe that by Cramer's rule, we can find a value $\nu^*>0$ (of size at least $2^{-(bs+s\log s)}$ if the numbers are $b$-bit rational values) such that if $\langle\bvar{w},\bvar{x}\rangle < 1$, $\langle\bvar{w},\bvar{x}\rangle \leq 1-\nu^*$. So, it is sufficient for $\langle\bvar{w},y\bvar{x}\rangle \geq y-\nu$  for a given $(\bvar{x},y)$, for some margin $\nu\geq 2^{-(bs+s\log s)}$. Thus, to find a consistent $\bvar{w}$, it suffices to solve the linear program $\langle\bvar{w},\lvar{y}(j)\bvar{x}(j)\rangle \geq \lvar{y}(j)-\nu$ for each $j$th example in $S$. Observe that if we parameterize $\bvar{w}$ by only the nonzero coefficients, we obtain a linear program in $s$ dimensions, for which we can obtain a feasible solution at a vertex, given by $s$ tight constraints. Now, Algorithm~\ref{alg:robust-list-learn} enumerates \emph{all} $s$-tuples of indices and examples, which in particular therefore must include any $s$-tuple of examples in the inlier set $S$ and the $s$ nonzero coordinates of $\bvar{w}$. Hence, with probability at least $1-\delta$, $L$ indeed contains some $\bvar{w}$ that attains error $\epsilon$ on $S$, as needed.
        \end{proof}

\section{Convergence Analysis of Projected SGD}
\label{sec:convergence-analysis-of-projected-sgd}
        We show our formal analysis of the Projected SGD (Algorithm \ref{algo:projected-stochastic-gradient-descent-for-minimizing-convex-surrogate-loss}) in this section.
        
        We first show that the gradient of statistic ReLU, $\derivative<\bvar{w}>\loss<\distr>{\bvar{w}}$, is almost Lipschitz continuous, which will be a critical piece in our convergence analysis of Projected SGD.
                
        \begin{lemma}[Relative Smoothness Of Statistic ReLU]\label{lma:relative-smoothness-of-statistic-relu}
            Let $\distr$ be a distribution on $\reals[d]\times\interval$ with standard normal $\bvar{x}$-marginal, $\loss<\distr>{\bvar{w}} = \E<(\bvar{x},\lvar{y})\sim\distr>{y\cdot\max(0, \innerprod{\bvar{x}}{\bvar{w}})}$. Then, for any $\bvar{v}, \bvar{w}\in\reals[d]$ such that at least one of $\norm{\bvar{v}}_2, \norm{\bvar{w}}_2$ is non-zero, we have
            \begin{equation}
                \norm{\derivative<\bvar{w}>\loss<\distr>{\bvar{w}} - \derivative<\bvar{v}>\loss<\distr>{\bvar{v}}}_2\leq\frac{2}{\max\sbr*{\norm{\bvar{v}}_2, \norm{\bvar{w}}_2}}\norm{\bvar{w} - \bvar{v}}_2
            \end{equation}
        \end{lemma}
        \begin{proof}
            Without loss of generosity, assume $\norm{\bvar{v}}_2\geq\norm{\bvar{w}}_2$, we prove the approximate Lipschitz continuity of $\derivative<\bvar{w}>\loss<\distr>{\bvar{w}}$ by showing that, for every $\bvar{w},\bvar{v} \in \reals[d]$, $\norm{\derivative<\bvar{w}>\loss<\distr>{\bvar{w}} - \derivative<\bvar{v}>\loss<\distr>{\bvar{v}}}_2$ is upper bound by $L\cdot\norm{\bvar{w} - \bvar{v}}_2/\norm{\bvar{v}}_2$ for some constant $L \geq 1$. 

            We firstly show that $\norm{\derivative<\bvar{w}>\loss<\distr>{\bvar{w}} - \derivative<\bvar{v}>\loss<\distr>{\bvar{v}}}_2 = \bigO{\theta(\bvar{v},\bvar{w})}$, then we will prove $\theta(\bvar{v},\bvar{w})$ can be upper bounded by $\norm{\bvar{w} - \bvar{v}}_2/\norm{\bvar{v}}_2$ asymptotically for $\theta(\bvar{v},\bvar{w})\in[0,\pi/2]$ and $\theta(\bvar{v},\bvar{w})\in[\pi/2, \pi]$ separately.
            
            Recall that $\derivative<\bvar{w}>\loss<\distr>{\bvar{w}} = \E<\sbr{\bvar{x}, y}\sim\distr>{y\cdot\bvar{x}\cdot\indicator[\bvar{x}\in\hypothesis[\bvar{w}]]}$. For conciseness of the proof, we define
            \begin{equation*}
                \bvar{u} = \argmax_{\norm{\bvar{z}}_2 = 1}\innerprod{\derivative<\bvar{w}>\loss<\distr>{\bvar{w}} - \derivative<\bvar{v}>\loss<\distr>{\bvar{v}}}{\bvar{z}}.
            \end{equation*}
            Then, we have
            \begin{align}
                \norm{\derivative<\bvar{w}>\loss<\distr>{\bvar{w}} - \derivative<\bvar{v}>\loss<\distr>{\bvar{v}}}_2 =& \innerprod{\derivative<\bvar{w}>\loss<\distr>{\bvar{w}} - \derivative<\bvar{v}>\loss<\distr>{\bvar{v}}}{\bvar{u}}\notag\\
                =& \E{y\cdot\innerprod{\bvar{x}}{\bvar{u}}\cdot\sbr{\indicator[\bvar{x}\in\hypothesis[\bvar{w}]\isect\hypothesis*[\bvar{v}]] - \indicator[\bvar{x}\in\hypothesis*[\bvar{w}]\isect\hypothesis[\bvar{v}]]}}\notag\\
                \leq& \E{\abs{\innerprod{\bvar{x}}{\bvar{u}}}\cdot\sbr{\indicator[\bvar{x}\in\hypothesis[\bvar{w}]\isect\hypothesis*[\bvar{v}]] + \indicator[\bvar{x}\in\hypothesis*[\bvar{w}]\isect\hypothesis[\bvar{v}]]}}.\label{eq:relative-smoothness-of-statistic-relu-decoupling-from-label}
            \end{align}
            Now, let's notice that the expectation above only has constraints on a $2$-dimensional subspace spanned by $\lbr{\bvar{v},\bvar{w}}$. Thus, will show that $\abs{\innerprod{\bvar{x}}{\bvar{u}}}$ is essentially upper bounded by the $l_2$ norm of the projection of $\bvar{x}$ onto a $3$-dimensional subspace, which will allow us to use polar coordinates to calculate the above expectation.

            We construct a set of orthonormal basis $V = \lbr{\bvar{e}<1>, \bvar{e}<2>, \bvar{e}<3>}$ as follow. At first, let $\theta_1 = \theta(\bvar{v},\bvar{w})$ so that $\theta_1\in[0,\pi]$, and we define $\bvar*{w} = \bvar{e}<2>$ as well as $\bvar*{v} = -\bvar{e}<1>\sin\theta_1 + \bvar{e}<2>\cos\theta_1$. Then, denote $\bvar{u}<W>$ to be the projection of $\bvar{u}$ on to the subspace spanned by $W = \lbr{\bvar{e}<1>, \bvar{e}<2>}$ and $\theta_2 = \theta(\bvar{u}<W>, \bvar{e}<1>)$ so that $\bvar*{u}<W> = \bvar{e}<1>\cos\theta_2 + \bvar{e}<2>\sin\theta_2$. At last, we define $\theta_3 = \theta(\bvar{u}, \bvar{u}<W>)$ so that $\theta_3\in[0,\pi/2]$ and $\bvar{e}<3>$ to be such that
            \begin{align*}
                \bvar{u} =& \bvar*{u}<W>\cos\theta_3 + \bvar{e}<3>\sin\theta_3 \\
                =&\bvar{e}<1>\cos\theta_2\cos\theta_3 + \bvar{e}<2>\sin\theta_2\cos\theta_3+ \bvar{e}<3>\sin\theta_3.
            \end{align*}
            Denote $\bvar{x}<i> = \innerprod{\bvar{x}}{\bvar{e}<i>}$ and $\bvar{x}<V>$ to be the projection of $\bvar{x}$ onto the subspace spanned by $V$, by Cauchy inequality, there is 
            \begin{align*}
                \innerprod{\bvar{x}}{\bvar{u}}=& \bvar{x}<1>\cos\theta_2\cos\theta_3 + \bvar{x}<2>\sin\theta_2\cos\theta_3+ \bvar{x}<3>\sin\theta_3\\
                =&\innerprod{\bvar{x}<V>}{\bvar{u}}\\
                \leq& \norm{\bvar{x}<V>}_2
            \end{align*}
            \begin{figure}[ht]
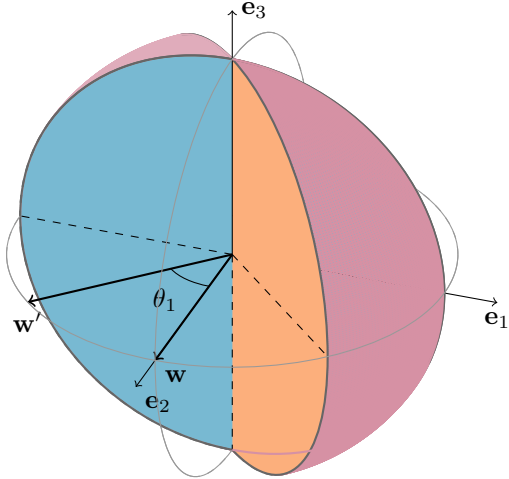

                \centering
                \adjustbox{max width=0.6\textwidth}{\drawintegralspaceALT*}
                \caption{Spherical coordinate interpretation.}
                \label{fig:spherical-coordinate-interpretation}
            \end{figure}
            Then, we transform the standard 3-dimensional coordinate system into a spherical coordinate system, also see figure \ref{fig:spherical-coordinate-interpretation}. For any $\bvar{x}<\lvar{V}> = (\bvar{x}<1>,\bvar{x}<2>,\bvar{x}<3>)$, let $\phi = \theta(\bvar{x}<\lvar{V}>,\bvar{e}<3>)$, $\theta = \theta(\bvar{x}<\lvar{V}\bvar{e}<3>[\bot]>, \bvar{e}<1>)$, and $r = \norm{\bvar{x}<\lvar{V}>}_2$, then we have $\bvar{x}<3> = r\cos\phi$, $\bvar{x}<1> = r\sin\phi\cos\theta$, and $\bvar{x}<2> = r\sin\phi\sin\theta$. Now, applying the standard Jacobian matrix that maps the spherical coordinates to 3-dimensional Cartesian coordinates yields $d\bvar{x}<1>d\bvar{x}<2>d\bvar{x}<3> = r^2\sin\phi drd\phi d\theta$. Therefore, following with inequality \eqref{eq:relative-smoothness-of-statistic-relu-decoupling-from-label}, we have
            \begin{align}
                \norm{\derivative<\bvar{w}>\loss<\distr>{\bvar{w}} - \derivative<\bvar{v}>\loss<\distr>{\bvar{v}}}_2\leq& 2\E{\norm{\bvar{x}<V>}_2\cdot\indicator[\bvar{x}\in\hypothesis[\bvar{w}]\isect\hypothesis*[\bvar{v}]]}\notag\\
                \ceq&2\E{\norm{\bvar{x}<V>}_2\cdot\indicator[\bvar{x}<2>\cos\theta_1\leq\bvar{x}<1>\sin\theta_1,\bvar{x}<2>\geq 0, \bvar{x}<3>\in\reals]}\notag\\
                \ceq[(i)]&\frac{1}{\sqrt{2\pi^3}}\int_0^{\theta_1}\int_0^{\pi}\int_0^{+\infty} r^3\sin\phi e^{-r^2/2} drd\phi d\theta\notag\\
                =&\theta_1\sqrt{\frac{2}{\pi^3}}\int_0^{+\infty}r^3e^{-r^2/2}dr\notag\\
                \ceq[(ii)]& \theta_1(2/\pi)^{3/2}\int_0^{+\infty}re^{-r^2/2}dr\notag\\
                =&\theta_1(2/\pi)^{3/2}\label{eq:relative-smoothness-of-statistic-relu-gradients-difference-is-bounded-by-angle}
            \end{align}
            where the first inequality holds because $\lbr*{\bvar{x}<\lvar{V}>\cond \bvar{x}<\lvar{V}>\in\hypothesis[\bvar{w}]\isect\hypothesis*[\bvar{v}]}$ and $\lbr*{\bvar{x}<\lvar{V}>\cond \bvar{x}<\lvar{V}>\in\hypothesis*[\bvar{w}]\isect\hypothesis[\bvar{v}]}$ are symmetric under Gaussian measure, the integral domain in equation (i) is valid for $\theta,\theta_1\in[0,\pi]$ because $\bvar{x}<2>\cos\theta_1\leq\bvar{x}<1>\sin\theta_1$ implies $\cot\theta_1\leq \cot\theta$, which, in turn, indicates $0\leq\theta\leq\theta_1$ as $\cot\theta$ is a monotone decreasing function on $\theta\in(0,\pi)$, and $\bvar{x}<2>\geq 0$ implies $\phi\in[0, \pi]$ as we know $r, \sin\theta\geq 0$ by construction, inequality (ii) is by using the law of integration by parts.

            \textbf{For the case of $\theta_1\in[0, \pi/2]$}, it is easy to see that
            \begin{align*}
                \norm{\bvar{w} - \bvar{v}}_2\geq& \norm{\norm{\bvar{v}}_2\cos\theta_1\cdot\bvar*{w} - \bvar{v}}_2\notag\\
                =& \norm{\bvar{v}}_2\sin\theta_1\\
                \cgeq[(i)]&\norm{\bvar{v}}_2\theta_1\cos\frac{\pi}{2\sqrt{3}}\\
                \cgeq[(ii)]&\sbr{\frac{\pi}{2}}^{3/2}\cos\frac{\pi}{2\sqrt{3}}\norm{\bvar{v}}_2\norm{\derivative<\bvar{w}>\loss<\distr>{\bvar{w}} - \derivative<\bvar{v}>\loss<\distr>{\bvar{v}}}_2\\
                \geq& \norm{\bvar{v}}_2\norm{\derivative<\bvar{w}>\loss<\distr>{\bvar{w}} - \derivative<\bvar{v}>\loss<\distr>{\bvar{v}}}_2
            \end{align*}
            where the first inequality holds because the RHS represent the shortest distance from vector $\bvar{v}$ to vector $\bvar{w}$, (i) is by the elementary inequality $x\cos\sbr*{x/\sqrt{3}}\leq \sin x$ as well as the assumption $\theta_1\in[0, \pi/2]$, (ii) is by inequality \eqref{eq:relative-smoothness-of-statistic-relu-gradients-difference-is-bounded-by-angle}, the last inequality holds due to $3/5 < \cos(\pi/2\sqrt{3})$ and $5/3 < (\pi/2)^{3/2}$.

            \textbf{For the case of $\theta_1\in[\pi/2, \pi]$}, by inequality \eqref{eq:relative-smoothness-of-statistic-relu-gradients-difference-is-bounded-by-angle} and $\norm{\bvar{w} - \bvar{v}}_2\geq \norm{\bvar{v}}_2$, we simply have
            \begin{equation*}
                \norm{\derivative<\bvar{w}>\loss<\distr>{\bvar{w}} - \derivative<\bvar{v}>\loss<\distr>{\bvar{v}}}_2\leq\sqrt{\pi/2}\leq \frac{2}{\norm{\bvar{v}}_2}\norm{\bvar{w} - \bvar{v}}_2
            \end{equation*}
            which completes the proof by taking $L = 2$.
        \end{proof}

        Now we are ready to show the convergence of the gradient norm in Algorithm \ref{algo:projected-stochastic-gradient-descent-for-minimizing-convex-surrogate-loss}.
        \begin{proposition}[Proposition \ref{prop:upper-bound-on-the-norm-of-statistic-relu-gradient}]\label{prop:prop:upper-bound-on-the-norm-of-statistic-relu-gradient-appendix}
            Let $\distr$ be a distribution on $\reals[d]\times\booldomain$ with standard normal $\bvar{x}$-marginal, $\loss<\distr>{\bvar{w}} = \E<(\bvar{x},\lvar{y})\sim\distr>{y\cdot\max(0, \innerprod{\bvar{x}}{\bvar{w}})}$, and $\func{g}<\bvar{w}>[\bvar{x}, y] = y\cdot\bvar{x}<\bvar{w}[\bot]>\cdot\indicator[\bvar{x}\in\hypothesis[\bvar{w}]]$. With $\beta = \sqrt{1/Td}$, after $T$ iterations, the output $\sbr*{\bvarseq{w}(T)}$ in algorithm \ref{algo:projected-stochastic-gradient-descent-for-minimizing-convex-surrogate-loss} will satisfies
            \begin{equation*}
                \E<\distr*(1),\ldots, \distr*(T)\sim\distr>{\frac{1}{T}\sum_{i=1}^T\norm*{\E<(\bvar{x},\lvar{y})\sim\distr>{\func{g}<\bvar{w}(i)>[\bvar{x},\lvar{y}]}}<2>[2]}\leq\sqrt{\frac{d}{T}}.
            \end{equation*}
            In addition, if $T \geq \sbr*{4d + \ln\sbr*{1/\delta}}/\epsilon^4$, it holds $\min_{i = 1,\ldots,T}\norm{\E<(\bvar{x},\lvar{y})\sim\distr>{\func{g}<\bvar{w}(i)>[\bvar{x},\lvar{y}]}}<2>\leq \epsilon$, with probability at least $1 - \delta$.
        \end{proposition}
        \begin{proof}
            Consider the $i$th iteration of algorithm \ref{algo:projected-stochastic-gradient-descent-for-minimizing-convex-surrogate-loss}. Based on the update $\bvar{u}(i) = \bvar{w}(i - 1) - \beta \E<\distr*(i)>{\func{g}<\bvar{w}(i-1)>[\bvar{x},\lvar{y}]}$, we have
            \begin{align}
                &\loss<\distr>{\bvar{u}(i)} - \loss<\distr>{\bvar{w}(i - 1)}
                \\
                =& \innerprod{\derivative<\bvar{w}>\loss<\distr>{\bvar{w}(i-1)}}{\bvar{u}(i) - \bvar{w}(i-1)}\notag
                \\
                &+\int_0^1\innerprod{\derivative<\bvar{w}>\loss<\distr>{\bvar{w}(i-1) + t\sbr*{\bvar{u}(i) - \bvar{w}(i - 1)}} - \derivative<\bvar{w}>\loss<\distr>{\bvar{w}(i-1)}}{\bvar{u}(i) - \bvar{w}(i-1)}dt\notag
                \\
                \leq& -\beta\innerprod{\derivative<\bvar{w}>\loss<\distr>{\bvar{w}(i-1)}}{\E<\distr*(i)>{\func{g}<\bvar{w}(i-1)>[\bvar{x},\lvar{y}]}}\notag
                \\
                &+\int_0^1\norm{\derivative<\bvar{w}>\loss<\distr>{\bvar{w}(i-1) + t\sbr*{\bvar{u}(i) - \bvar{w}(i - 1)}} - \derivative<\bvar{w}>\loss<\distr>{\bvar{w}(i-1)}}_2\norm{\bvar{u}(i) - \bvar{w}(i-1)}_2dt\notag
                \\
                \leq& -\beta \innerprod{\E<\distr>{\func{g}<\bvar{w}(i-1)>[\bvar{x},\lvar{y}]}}{\E<\distr*(i)>{\func{g}<\bvar{w}(i-1)>[\bvar{x},\lvar{y}]}} + \beta^2\norm*{\E<\distr*(i)>{\func{g}<\bvar{w}(i-1)>[\bvar{x},\lvar{y}]}}<2>[2]\label{eq:statistic-relu-function-change-wrt-gradient}
            \end{align}
            where the first term of the last inequality holds because $\bvar{x} = \bvar{w}(i-1)[\otimes 2]\bvar{x} + \bvar{x}<\bvar{w}(i-1)[\bot]>$ and $\innerprod*{\bvar{z}<\bvar{w}(i-1)[\bot]>}{\bvar{w}(i-1)[\otimes 2]\bvar{x}} = 0$ for any $\bvar{z}\in \reals[d]$, the second term holds due to lemma \ref{lma:relative-smoothness-of-statistic-relu} and that the projection step (line \ref{line:psgd-projection-step}) of algorithm \ref{algo:projected-stochastic-gradient-descent-for-minimizing-convex-surrogate-loss} ensures that $\norm{\bvar{w}(i-1)}_2 = 1$. Then, observe that $\E<\distr*(i)>{\func{g}<\bvar{w}(i-1)>[\bvar{x},\lvar{y}]}$ lies on the orthogonal subspace of $\bvar{w}(i-1)$ (see figure \ref{fig:length-increase-appendix}), which implies
            \begin{equation*}
                \norm*{\bvar{u}(i)}<2>[2] = \norm*{\bvar{w}(i - 1)}<2>[2] + \beta \norm*{\E<\distr*(i)>{\func{g}<\bvar{w}(i-1)>[\bvar{x},\lvar{y}]}}<2>[2] \geq 1
            \end{equation*}
            \begin{figure}[ht]
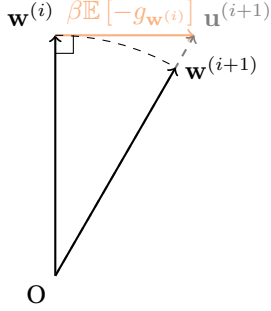

                \centering
                \adjustbox{max width=0.3\textwidth}{\drawlengthincrease(0.8)}
                \caption{Weight update step (line \ref{line:psgd-gradient-update}) and projection step (line \ref{line:psgd-projection-step}) in algorithm \ref{algo:projected-stochastic-gradient-descent-for-minimizing-convex-surrogate-loss}.}
                \label{fig:length-increase-appendix}
            \end{figure}

            Now, since $\loss<\distr>{\bvar{w}} = \E<\sbr{\bvar{x}, y}\sim\distr>{y\cdot\max\sbr*{0, \innerprod{\bvar{x}}{\bvar{w}}}}$ by definition, there is $\loss<\distr>{\bvar{w}} = \norm{\bvar{w}}_2\loss<\distr>{\bvar*{w}}$, which, along with the fact that $\norm{\bvar{u}(i)}<2>\geq 1$, indicates $\loss<\distr>{\bvar{w}(i)}\leq \loss<\distr>{\bvar{u}(i)}$. Therefore, applying $\loss<\distr>{\bvar{w}(i)}\leq \loss<\distr>{\bvar{u}(i)}$ to the LHS of inequality \eqref{eq:statistic-relu-function-change-wrt-gradient} gives
            \begin{equation}
                \loss<\distr>{\bvar{w}(i)} - \loss<\distr>{\bvar{w}(i - 1)}\leq -\beta \innerprod{\E<\distr>{\func{g}<\bvar{w}(i-1)>[\bvar{x},\lvar{y}]}}{\E<\distr*(i)>{\func{g}<\bvar{w}(i-1)>[\bvar{x},\lvar{y}]}} + \beta^2\norm*{\E<\distr*(i)>{\func{g}<\bvar{w}(i-1)>[\bvar{x},\lvar{y}]}}<2>[2].
            \end{equation}

            Then, conditioning on the previous samples $\distr*(1),\ldots, \distr*(i-1)$, we have, by the independence between $\distr$ and $\distr*(i)$ and Jensen's inequality, that 
            \begin{align*}
                &\E<\distr*(i)\sim\distr>{\loss<\distr>{\bvar{w}(i)} - \loss<\distr>{\bvar{w}(i - 1)}\cond \distr*(1), \ldots,\distr*(i-1)}
                \\
                =& -\beta\norm*{\E<\distr>{\func{g}<\bvar{w}(i-1)>[\bvar{x},y]}}+ \beta^2\E<\distr*(i)\sim\distr>{\norm*{\E<\distr*(i)>{\func{g}<\bvar{w}(i-1)>[\bvar{x},y]}}<2>[2]}
                \\
                \leq& -\beta\norm*{\E<\distr>{\func{g}<\bvar{w}(i-1)>[\bvar{x},y]}}+ \beta^2\E<\distr*(i)\sim\distr>{\E<\distr*(i)>{\norm*{\func{g}<\bvar{w}(i-1)>[\bvar{x},y]}<2>[2]}}
                \\
                \leq& -\beta\norm*{\E<\distr>{\func{g}<\bvar{w}(i-1)>[\bvar{x},y]}}<2>[2] + \frac{\beta^2 d}{2}
            \end{align*}
            where the last inequality holds because $\E<\distr*(i)\sim\distr>{\E<\distr*(i)>{\norm{\func{g}<\bvar{w}(i-1)>[\bvar{x},y]}<2>[2]}} = \E<\distr>{\norm{\func{g}<\bvar{w}(i-1)>[\bvar{x},y]}<2>[2]}$ and property (3) of lemma \ref{lma:upper-bound-on-relu-and-its-gradient-norm} gives $\E<\distr>{\norm{\func{g}<\bvar{w}(i-1)>[\bvar{x},y]}<2>[2]}\leq d/2$. Averaging the above inequality over all $T$ iterations and using the \emph{law of total expectation} gives
            \begin{align*}
                \frac{1}{T}\sum_{i=1}^T\norm*{\E<\distr>{\func{g}<\bvar{w}(i)>[\bvar{x},y]}}<2>[2]\leq&\frac{\loss<\distr>{\bvar{w}(0)} - \E<\distr*(T+1)\sim\distr>{\loss<\distr>{\bvar{w}(T+1)}}}{\beta T} + \frac{\beta d}{2}\\
                \leq& \frac{1}{\sqrt{2\pi}\beta T} + \frac{\beta d}{2}
            \end{align*}
            where the last inequality is derived through using property (1) of lemma \ref{lma:upper-bound-on-relu-and-its-gradient-norm} with $\norm{\bvar{w}(i)}_2 = 1$ for all $i = 0, \ldots, T+1$, and the fact that $\loss<\distr>{\bvar{w}}\geq 0$. Taking $\beta = \sqrt{1/T d}$ gives the first claim.

            To obtain the high-probability version, we define
            \begin{equation*}
                \func{G}<T>[\bvar{w}(1), \ldots, \bvar{w}(T)] = \frac{1}{T}\sum_{i=1}^T\norm*{\E<(\bvar{x},y)\sim\distr>{\func{g}<\bvar{w}(i)>[\bvar{x},y]}}<2>[2]
            \end{equation*}
            which implies
            \begin{align*}
                &\abs{\func{G}<T>[\bvar{w}(1),\ldots, \bvar{w}(i), \ldots, \bvar{w}(T)] - \func{G}<T>[\bvar{w}(1),\ldots, \bvar{w}(i)['], \ldots, \bvar{w}(T)]}
                \\
                \leq& \frac{1}{T}\abs{\norm{\E{\func{g}<\bvar{w}(i)>[\bvar{x},\lvar{y}]}}<2>[2] - \norm{\E{\func{g}<\bvar{w}(i) '>[\bvar{x},\lvar{y}]}}<2>[2]}
                \\
                \leq& \frac{1}{\sqrt{2\pi}T}
            \end{align*}
            where the last step holds due to property (2) of lemma \ref{lma:upper-bound-on-relu-and-its-gradient-norm}. Now using lemma \ref{lma:concentration-bound-by-devroye}, we get
            \begin{equation*}
                \prob{\func{G}<T>[\bvarseq{w}(T)] - \E<\distr*(1),\ldots,\distr*(T)\sim\distr>{\func{G}<T>[\bvarseq{w}(T)]}\geq t}\leq \exp\sbr{-4\pi t^2T}
            \end{equation*}
            Choosing $T\geq \sbr*{4d + \ln\sbr*{1/\delta}}/\epsilon^4$ gives $\E{\func{G}<T>[\bvarseq{w}(T)]}\leq \epsilon^2/2$ by our first claim, and, hence, 
            \begin{equation*}
                \prob{\frac{1}{T}\sum_{i=1}^T\norm{\E{\func{g}<\bvar{w}(i)>[\bvar{x},y]}}<2>[2]\leq \epsilon^2} = \prob{\func{G}<T>[\bvarseq{w}(T)]\leq\epsilon^2}\geq 1 - \delta
            \end{equation*}
             Finally, since $\min_{i = 1,\ldots,T}\norm{\E{\func{g}<\bvar{w}(i)>[\bvar{x},y]}}<2>[2]$ is at most the average, we obtain the second claim.
        \end{proof}

        Below are a few tools we needed in the proof of proposition \ref{prop:upper-bound-on-the-norm-of-statistic-relu-gradient}.
        \begin{lemma}[Theorem 2.2 of \citet{devroye2001combinatorial}]\label{lma:concentration-bound-by-devroye}
            Suppose that $\lvarseq{X}<d>\in\mathcal{X}$ are independent random variables, and let $f: \mathcal{X}^d\rightarrow \reals$. Let $\lvarseq{c}<n>$ satisfies
            \begin{equation*}
                \sup_{\lvarseq{x}<d>,\lvar{x}<i> '}\abs{f(\lvarseq{x}<i>,\ldots,\lvar{x}<d>) - f(\lvar{x}<1>,\ldots,\lvar{x}<i> ',\ldots,\lvar{x}<d>)}\leq \lvar{c}<i>
            \end{equation*}
            for $i\in[d]$. Then
            \begin{equation*}
                \prob{f(X) - \E{f(X)}\geq t}\leq \exp\sbr{-\frac{2t^2}{\sum_{i\in[d]}\lvar{c}<2>[2]}}.
            \end{equation*}
        \end{lemma}

        \begin{lemma}\label{lma:upper-bound-on-relu-and-its-gradient-norm}
            Let $\distr$ be a distribution on $\reals[d]\times\interval$ with standard normal $\bvar{x}$-marginal, $\loss<\distr>{\bvar{w}} = \E<(\bvar{x},\lvar{y})\sim\distr>{y\cdot\max(0, \innerprod{\bvar{x}}{\bvar{w}})}$, and $\func{g}<\bvar{w}>[\bvar{x}, y] = y\cdot\bvar{x}<\bvar{w}[\bot]>\cdot\indicator[\bvar{x}\in\hypothesis[\bvar{w}]]$. Then, for any $\bvar{w}\in\reals[d]$, we have the following properties:
            \begin{enumerate}
                \item $\loss<\distr>{\bvar{w}}\leq \norm{\bvar{w}}_2/\sqrt{2\pi}$,
                \item $\norm{\E<(\bvar{x}, y)\sim\distr>{\func{g}<\bvar{w}>[\bvar{x}, y]}}<2>\leq 1/\sqrt{2\pi}$,
                \item $\E<(\bvar{x},y)\sim\distr>{\norm{\func{g}<\bvar{w}>[\bvar{x}, y]}<2>[2]}\leq d/2$.
            \end{enumerate}
        \end{lemma}
        \begin{proof}
            To show the first claim, recall that 
            \begin{align*}
                \loss<\distr>{\bvar{w}} =& \E<\sbr{\bvar{x}, y}\sim\distr>{y\cdot\max\sbr*{0, \innerprod{\bvar{x}}{\bvar{w}}}}\\
                \leq& \E<\bvar{x}\sim\distr<\bvar{x}>>{\innerprod{\bvar{x}}{\bvar{w}}\cdot\indicator[\innerprod{\bvar{x}}{\bvar{w}}\geq 0]}\\
                \ceq[(i)]&\frac{\norm{\bvar{w}}_2}{\sqrt{2\pi}}\int_0^{+\infty}ze^{-z^2/2}dz\\
                =&\frac{\norm{\bvar{w}}_2}{\sqrt{2\pi}}
            \end{align*}
            where inequality (i) holds because $\bvar{x}\sim\gaussian[d]$ and, hence, $\innerprod{\bvar{x}}{\bvar*{w}}\sim\gaussian$. 
            
            To prove property (2), let $\bvar{u} = \argmax_{\norm{\bvar{z}}_2 =1 }\innerprod{\E<(\bvar{x}, y)\sim\distr>{y\cdot\bvar{x}<\bvar{w}[\bot]>\cdot\indicator[\bvar{x}\in\hypothesis[\bvar{w}]]}}{\bvar{z}}$, we have
            \begin{align*}
                &\norm*{\E<(\bvar{x}, y)\sim\distr>{y\cdot\bvar{x}<\bvar{w}[\bot]>\cdot\indicator[\bvar{x}\in\hypothesis[\bvar{w}]]}}<2>
                \\
                =& \E{y\cdot\innerprod{\bvar{x}<\bvar{w}[\bot]>}{\bvar{u}}\cdot\indicator[\bvar{x}\in\hypothesis[\bvar{w}]]}
                \\
                \leq& \E{\abs{\innerprod{\bvar{x}<\bvar{w}[\bot]>}{\bvar{u}}}\cdot\sbr{\indicator[\bvar{x}\in\hypothesis[\bvar{w}], \innerprod{\bvar{x}<\bvar{w}[\bot]>}{\bvar{u}}\geq 0] - \indicator[\bvar{x}\in\hypothesis[\bvar{w}], \innerprod{\bvar{x}<\bvar{w}[\bot]>}{\bvar{u}}< 0]}}
                \\
                \leq& \E{\innerprod{\bvar{x}<\bvar{w}[\bot]>}{\bvar{u}}\cdot\indicator[\innerprod{\bvar{x}<\bvar{w}[\bot]>}{\bvar{u}}\geq 0]}
                \\
                =&\frac{1}{\sqrt{2\pi}}.
            \end{align*}

            To obtain the last property, notice that $\norm{\bvar{x}<\bvar{w}[\bot]>}<2>\leq\norm{\bvar{x}}<2>$ because $\bvar{x}<\bvar{w}[\bot]>$ is a projection of $\bvar{x}$, then we have
            \begin{align*}
                \E<(\bvar{x}, y)\sim\distr>{\norm{y\cdot\bvar{x}<\bvar{w}[\bot]>\cdot\indicator[\innerprod{\bvar{x}}{\bvar{w}}>0]}<2>[2]}\leq&\E{\norm{\bvar{x}}<2>[2]\cdot\indicator[\bvar{x}\in\hypothesis[\bvar{w}]]}
                \\
                \cleq[(i)]&\frac{1}{2}\E{\norm{\bvar{x}}<2>[2]}
                \\
                =&d/2
            \end{align*}
            where inequality (i) holds because $y\geq 0$ and the symmetry of standard normal distribution.
        \end{proof}

    \section{Optimality Analysis of Approximate Stationary Point}
    \label{sec:optimality-analysis-of-approximate-stationary-point}
        We present our analysis for the main theorem of our algorithmic results in this section.
        \begin{theorem}[Theorem \ref{thm:main-theorem}]\label{thm:main-theorem-appendix}
            Let $\distr$ be a distribution on $\reals[d]\times\booldomain$ with standard normal $\bvar{x}$-marginal, and $\conceptclass$ be a class of binary classifiers on $\reals[d]\times\booldomain$. If there exists a unit vector $\bvar{v}\in\reals[d]$ such that, for some sufficiently small $\epsilon\in[0, 1/e]$, $\min_{c\in\conceptclass}\prob<\sbr{\bvar{x}, y}\sim\distr>{\bvar{x}\in\hypothesis[\bvar{v}]\isect c(\bvar{x})\neq y}\leq\epsilon$, then, with at most $\bigO*{d/\epsilon^6}$ examples, algorithm \ref{algo:conditional-classification-with-homogeneous-halfspaces-for-finite-classes} will return a $\bvar{w}(c)$ such that $\prob<\sbr{\bvar{x}, y}\sim\distr>{\bvar{x}\in\hypothesis[\bvar{w}(c)]\isect c(\bvar{x})\neq y} = \bigO*{\sqrt{\epsilon}}$ with probability at least $1 - \delta$ and running time at most $\bigO{d\abs{\conceptclass}/\epsilon^6}$. 
        \end{theorem}
        \begin{proof}
            For conciseness of the proof, let the error indicator function $\func{f}<\bvar{w}>(c):\reals[d]\times\booldomain\rightarrow\booldomain$ be such that $\func{f}<\bvar{w}>(c)[\bvar{x},y] = \indicator[\bvar{x}\in\hypothesis[\bvar{w}]\isect c(\bvar{x})\neq y]$.
            
            Consider the $c\in\conceptclass$ that satisfies $\min_{\bvar{w}}\prob*<\distr>{\func{f}<\bvar{w}>(c)[\bvar{x},y] = 1}\leq \epsilon$. For $T = \sbr*{4d + \ln\sbr*{8/\delta_1}}/\epsilon^4$, $N \geq 1600\ln^2\sbr*{16T/\delta_1}/\epsilon^2$, lemma \ref{lma:psgd-returns-at-least-one-optimal-solution-appendix} and a union bound over the two calls of algorithm \ref{algo:projected-stochastic-gradient-descent-for-minimizing-convex-surrogate-loss} guarantees that there exists a $\bvar{w} '\in\parameterset{W}(c)$ such that $\prob<\sbr{\bvar{x}, y}\sim\distr>{\func{f}<\bvar{w} '>(c)[\bvar{x},y]}\leq \sbr{C + 1}\sbr*{\epsilon\sqrt{\ln\epsilon^{-1}}}^{1/2}$ with probability at least $1 - \delta_1/2$. 

            While estimating each $\bvar{w}\in\parameterset{W}(c)$ at line \ref{line:find-the-loss-minimizer} with $\ln(4T/\delta_1)/2\epsilon$ samples in $\distr*$, we know that
            \begin{equation*}
                \prob{\abs{\E<\distr*>{\func{f}<\bvar{w}>(c)[\bvar{x},y]} - \E<\distr>{\func{f}<\bvar{w}>(c)[\bvar{x},y]} }> \sqrt{\epsilon}}\leq \delta_1/2T
            \end{equation*}
            by lemma \ref{lma:chernoff-bound}. Taking a union bound over all $\bvar{w}\in\parameterset{W}(c)$ gives
            \begin{equation*}
                \prob{\E<\distr>{\func{f}<\bvar{w}(c)>(c)[\bvar{x},y]}> \E<\distr>{\func{f}<\bvar{w} '>(c)[\bvar{x},y]} + 2\sqrt{\epsilon}}\leq \delta_1/2
            \end{equation*}
            In addition, taking another union bound taking over lines \ref{line:clc-call-psgd},\ref{line:find-the-loss-minimizer} of algorithm \ref{algo:conditional-classification-with-homogeneous-halfspaces-for-finite-classes} and using the optimality of $\bvar{w}(c)$, we can conclude that $\prob*<\sbr{\bvar{x}, y}\sim\distr>{\bvar{x}\in\hypothesis[\bvar{w}(c)]\isect c(\bvar{x})\neq y}=  \bigO*{\sqrt{\epsilon}}$ with probability at least $1 - \delta_1$ in this iteration. 
            
            Finally, taking an union bound again over all $c\in\conceptclass$ and choosing $\delta_1 = \delta/\abs{\conceptclass}$, we know that the total number of examples needed is $\bigO{TN} =  \bigO{d\ln^2\sbr*{16T\abs{\conceptclass}/\delta}/\epsilon^6} = \bigO*{d/\epsilon^6}$ and the running time is simply $\bigO{\abs{\conceptclass}TN} = \bigO*{d\abs{\conceptclass}/\epsilon^6}$, since we can reuse the example for each $c\in\conceptclass$.
        \end{proof}
        \begin{proposition}[Proposition \ref{prop:gradient-projection-lower-bound}]\label{prop:gradient-projection-lower-bound-appendix}
            Let $\distr$ be a distribution on $\reals[d]\times\booldomain$ with standard normal $\bvar{x}$-marginal, and $\func{g}<\bvar{w}>[\bvar{x}, y] = y\cdot\bvar{x}<\bvar{w}[\bot]>\cdot\indicator[\bvar{x}\in\hypothesis[\bvar{w}]]$. Suppose $\bvar{v},\bvar{w}\in\reals[d]$ are unit vectors such that $\prob<\sbr{\bvar{x}, y}\sim\distr>{\bvar{x}\in\hypothesis[\bvar{v}]\isect y = 1}\leq\epsilon$ and $\theta(\bvar{v}, \bvar{w})\in[0,\pi/2)$, then, if $\prob<\sbr{\bvar{x}, y}\sim\distr>{\bvar{x}\in\hypothesis[\bvar{w}]\isect y = 1}\geq \frac{5}{2}\sbr*{\epsilon\sqrt{\ln\epsilon^{-1}}}^{1/2}$, it holds that
            \begin{equation*}
                \innerprod{\E<(\bvar{x}, y)\sim\distr>{-\func{g}<\bvar{w}>[\bvar{x},\lvar{y}]}}{\bvar*{v}<\bvar{w}[\bot]>} \geq \frac{2}{5}\epsilon\sqrt{\ln\epsilon^{-1}}
            \end{equation*}
            for some sufficiently small $\epsilon\in[0, 1/e]$.
        \end{proposition}
        \begin{proof}
            For conciseness, let $\theta = \theta(\bvar{v},\bvar{w})$ and define two orthonormal basis $\bvar{e}<1>, \bvar{e}<2>$ such that $\bvar{w} = \bvar{e}<2>$ and $\bvar{v} = -\bvar{e}<1>\sin\theta + \bvar{e}<2>\cos\theta$, which implies $\bvar{e}<1> = -\bvar*{v}<\bvar{w}[\bot]>$. Denote $\bvar{x}<i> = \innerprod{\bvar{x}}{\bvar{e}<i>}$ so that $\innerprod{\bvar{x}}{\bvar{w}} = \bvar{x}<2>$ and $\innerprod{\bvar{x}}{\bvar{v}} = -\bvar{x}<1>\sin\theta + \bvar{x}<2>\cos\theta$. Because $\innerprod{\bvar{x}}{\bvar{e}<1>} = \innerprod*{\bvar{x}<2>\bvar{e}<2> + \bvar{x}<\bvar{e}<2>[\bot]>}{\bvar{e}<1>} = -\innerprod{\bvar{x}<\bvar{w}[\bot]>}{\bvar*{v}<\bvar{w}[\bot]>}$, we have 
            \begin{align}
                \innerprod{\E{-\func{g}<\bvar{w}>[\bvar{x},\lvar{y}]}}{\bvar*{v}<\bvar{w}[\bot]>} =&-\E{y\cdot \innerprod{\bvar{x}<\bvar{w}[\bot]>}{\bvar*{v}<\bvar{w}[\bot]>}\cdot\indicator[\bvar{x}\in\hypothesis[\bvar{w}]]}\notag\\
                =& \E{y\cdot \innerprod{\bvar{x}}{\bvar{e}<1>}\cdot\indicator[\bvar{x}\in\hypothesis[\bvar{w}]]}\notag\\
                =& \E{y\cdot \bvar{x}<1>\cdot\sbr{\indicator[\bvar{x}\in\hypothesis[\bvar{w}]\isect\hypothesis[\bvar{v}]] + \indicator[\bvar{x}\in\hypothesis[\bvar{w}]\isect\hypothesis*[\bvar{v}]]}}\notag\\
                \geq& \underbrace{\E{\abs{\bvar{x}<1>}\cdot\indicator[\bvar{x}<1>\tan\theta > \bvar{x}<2>\geq 0, y = 1]}}_{I_1}\notag\\
                &- \underbrace{\E{\abs{\bvar{x}<1>}\cdot\indicator[\bvar{x}<2>\geq 0, \bvar{x}<2>\geq\bvar{x}<1>\tan\theta, y =1]}}_{I_2}\label{eq:gradient-projection-lower-bound-decomposition-of-gradient-projection}.
            \end{align}            
            where the last inequality holds because $\cos\theta > 0$ by our assumption that $\theta(\bvar{v},\bvar{w})\in[0,\pi/2)$, and $\hypothesis[\bvar{w}] = \lbr*{\bvar{x}\cond \innerprod{\bvar{x}}{\bvar{w}} \geq 0}, \hypothesis[\bvar{v}] = \lbr*{\bvar{x}\cond \innerprod{\bvar{x}}{\bvar{v}}\geq 0}$ imply that $\bvar{x}<2>\geq 0,  \bvar{x}<2> \geq \bvar{x}<1>\tan\theta$ by construction. This decomposition above can also be seen from figure \ref{fig:illustration-of-I1-and-I2}. Then, we will apply lemma \ref{lma:subset-expection-bounds} to bound the above two terms.

            \begin{figure}[ht]
                \centering
                \drawerrorregion*(0.75)
                \caption{\textcolor{myblue}{Blue} area represent $\hypothesis[\bvar{v}]\isect\hypothesis[\bvar{w}]$, while \textcolor{myorange}{orange} area represents $\hypothesis[\bvar{w}]\isect\hypothesis*[\bvar{v}]$.}
                \label{fig:illustration-of-I1-and-I2}
            \end{figure}
            
            Observe that, since $\bvar{x}$ is sampled from a standard normal distribution and $\bvar{e}<1>, \bvar{e}<2>$ are two orthonormal basis, $\bvar{x}<1>, \bvar{x}<2>$ are two independent one-dimension standard normal random variables. Then, observe that we can bound $I_1$ and $I_2$ by applying lemma \ref{lma:subset-expection-bounds} with carefully chosen $\alpha$ and $\beta$.

            \textbf{To apply lemma \ref{lma:subset-expection-bounds} on $I_2$}, by treating $\bvar{x}<2>\geq 0$ to be the event $T$ and the rest to be $S$ in lemma \ref{lma:subset-expection-bounds}, we show that there exists an $\alpha > 0$ such that $\prob{\bvar{x}<2>\geq 0\isect\bvar{x}<2>\geq\bvar{x}<1>\tan\theta\isect y = 1}\leq \prob{\bvar{x}<2>\geq 0\isect\abs{\bvar{x}<1>}\geq \alpha}$. 
            
            First of all, notice that $\prob{\bvar{x}<2>\geq 0\isect\bvar{x}<2>\geq\bvar{x}<1>\tan\theta\isect y = 1} = \prob{\bvar{x}\in\hypothesis[\bvar{v}]\isect y = 1}\leq\epsilon$ by our assumption. Suppose $\alpha = \sqrt{2\ln\epsilon^{-1} - 2\ln\sbr*{\kappa\sqrt{\ln\epsilon^{-1}}}}$ for some $\kappa > 1$, then, due to the independence between $\bvar{x}<1>,\bvar{x}<2>$ as well as lemma \ref{lma:gaussian-tail-lower-bound}, there is 
            \begin{align*}
                \prob{\bvar{x}<2>\geq 0\isect\abs{\bvar{x}<1>}\geq \alpha}\geq&\frac{\exp\sbr{-\ln\epsilon^{-1} + \ln \sbr{\kappa\sqrt{\ln\epsilon^{-1}}}}}{\sqrt{2\pi}\sbr{\sqrt{2\ln\epsilon^{-1} - 2\ln\sbr{\kappa\sqrt{\ln\epsilon^{-1}}}} + 1}}\\
                =&\frac{\epsilon\kappa}{\sqrt{2\pi}\sbr{\sqrt{2 - 2\ln\sbr{\kappa\sqrt{\ln\epsilon^{-1}}}/\ln\epsilon^{-1}} + 1/\sqrt{\ln\epsilon^{-1}}}}\\
                \geq& \frac{\epsilon\kappa}{\sqrt{2\pi}\sbr{\sqrt{2} + 1}}
            \end{align*}
            where the last inequality holds because $\kappa > 1$ and $\epsilon\in[0, 1/e]$ so that $\ln\sbr{\kappa\sqrt{\ln\epsilon^{-1}}}/\ln\epsilon^{-1}\geq 0$ as well as $\ln\epsilon^{-1}\geq 1$. Taking $\kappa=  \sqrt{2\pi}\sbr{\sqrt{2} + 1}$ results to $\prob{\bvar{x}<2>\geq 0\isect\abs{\bvar{x}<1>}\geq \alpha}\geq\epsilon$. Then, lemma \ref{lma:subset-expection-bounds} gives
            \begin{align}
                I_2\leq& \E{\abs{\bvar{x}<1>}\cdot\indicator[\bvar{x}<2>\geq 0, \abs{\bvar{x}<1>} \geq \alpha]}\notag\\
                =&\frac{1}{\sqrt{2\pi}}\int_{\geq\alpha}\bvar{x}<1>e^{-\bvar{x}<1>[2]/2}d\bvar{x}<1>\notag\\
                =&\frac{\exp\sbr{\ln\epsilon + \ln\sbr{\sqrt{2\pi}\sbr{\sqrt{2} + 1}\sqrt{\ln\epsilon^{-1}}}}}{\sqrt{2\pi}}\notag\\
                \leq& 3\epsilon\sqrt{\ln\epsilon^{-1}}.\label{eq:gradient-projection-lower-bound-upper-bound-on-I2}
            \end{align}
            
            \textbf{To apply lemma \ref{lma:subset-expection-bounds} on $I_1$}, notice that the event $\bvar{x}<1>\tan\theta > \bvar{x}<2>\geq 0$ in $I_1$ is a subset of event $\bvar{x}<1>\geq 0\isect\bvar{x}<2>\geq0$ because $\theta(\bvar{v},\bvar{w})\in[0, \pi/2)$. Therefore, we can view the event $\bvar{x}<1>\geq 0\isect\bvar{x}<2>\geq0$ as $T$ in lemma \ref{lma:subset-expection-bounds} and show that there exists a $\beta > 0$ such that $ \prob{0\leq\bvar{x}<1>\leq \beta\isect\bvar{x}<2>\geq 0}\leq  \prob{\bvar{x}<1>\tan\theta > \bvar{x}<2>\geq 0\isect y = 1}$ to apply lemma \ref{lma:subset-expection-bounds}. 
            
            At first, observe that, by our assumption that $\prob{\bvar{x}\in\hypothesis[\bvar{v}]\isect y = 1}\leq \epsilon$ as well as $\prob{\bvar{x}\in\hypothesis[\bvar{w}]\isect y = 1} \geq \frac{5}{2}\sbr{\epsilon\sqrt{\ln\epsilon^{-1}}}^{1/2}$, there is
            \begin{align*}
                \sbr{e^{-1/2} + e^{1/2}}\sbr{\epsilon\sqrt{\ln\epsilon^{-1}}}^{1/2} - \epsilon <&\frac{5}{2}\sbr{\epsilon\sqrt{\ln\epsilon^{-1}}}^{1/2} - \epsilon\\
                \leq&\prob{\bvar{x}\in\hypothesis[\bvar{w}]\isect y = 1} - \prob{\bvar{x}\in\hypothesis[\bvar{v}]\isect\bvar{x}\in\hypothesis[\bvar{w}]\isect y = 1}\\
                =&\prob{\bvar{x}\in\hypothesis*[\bvar{v}]\isect\bvar{x}\in\hypothesis[\bvar{w}]\isect y = 1}\\
                \ceq& \prob{\bvar{x}<1>\tan\theta > \bvar{x}<2>\geq 0\isect y = 1}
            \end{align*}
            where the first inequality holds because $e^{-1/2} + e^{1/2}\leq 5/2$. Then, taking $\beta = 2\sqrt{2e\pi}\sbr{\epsilon\sqrt{\ln\epsilon^{-1}}}^{1/2}$ yields
            \begin{align*}
                \prob{0\leq\bvar{x}<1>\leq \beta\isect\bvar{x}<2>\geq 0}=& \frac{1}{2}\prob{0\leq\bvar{x}<1>\leq \beta}\\
                \cleq[(i)]&\sqrt{e}\sbr{\epsilon\sqrt{\ln\epsilon^{-1}}}^{1/2}\\
                =& \sbr{e^{-1/2} + e^{1/2}}\sbr{\epsilon\sqrt{\ln\epsilon^{-1}}}^{1/2} - e^{-1/2}\sbr{\epsilon\sqrt{\ln\epsilon^{-1}}}^{1/2}\\
                \leq& \sbr{e^{-1/2} + e^{1/2}}\sbr{\epsilon\sqrt{\ln\epsilon^{-1}}}^{1/2} - \epsilon
            \end{align*}
            where the first equation holds because $\bvar{x}<1>,\bvar{x}<2>$ are independent, inequality (i) holds due to the fact that standard normal density is never greater than $1/\sqrt{2\pi}$, and the last inequality holds because $\epsilon\in[0, e^{-1}]$ so that $e^{-1/2}\geq \sqrt{\epsilon}/\ln^{1/4}\epsilon^{-1}$. Applying lemma \ref{lma:subset-expection-bounds} gives
            \begin{align}
                I_1\geq& \E{\bvar{x}<1>\cdot\indicator[0\leq\bvar{x}<1>\leq 2\sqrt{2e\pi}\sbr{\epsilon\sqrt{\ln\epsilon^{-1}}}^{1/2},\bvar{x}<2>\geq 0]}\notag\\
                =& \frac{1}{2\sqrt{2\pi}}\int_0^{2\sqrt{2e\pi}\sbr{\epsilon\sqrt{\ln\epsilon^{-1}}}^{1/2}}\bvar{x}<1>e^{-\bvar{x}<1>[2]/2}d\bvar{x}<1>\notag\\
                =&\frac{1 - \exp\sbr{-4e\pi \epsilon\sqrt{\ln\epsilon^{-1}}}}{2\sqrt{2\pi}}\notag\\
                \geq&\sqrt{\frac{\pi}{2}}e\epsilon\sqrt{\ln\epsilon^{-1}}\label{eq:gradient-projection-lower-bound-lower-bound-on-I1}
            \end{align}
            where the last inequality holds because of the fundamental inequality $x/2\leq 1 - e^{-x}$ for $x\in [0, 1.59]$.

            At last, since $e\sqrt{\pi/2} - 3 > 0.4$, taking inequalities \eqref{eq:gradient-projection-lower-bound-lower-bound-on-I1} and \eqref{eq:gradient-projection-lower-bound-upper-bound-on-I2} back to inequality \eqref{eq:gradient-projection-lower-bound-decomposition-of-gradient-projection} gives the desired result.
        \end{proof}

        The following lemma plays a key role in proving proposition \ref{prop:gradient-projection-lower-bound}.
        \begin{lemma}\label{lma:subset-expection-bounds}
            Let $\distr$ be an arbitrary distribution on $\reals[d]$, and $S,T$ be any events such that $\prob<\distr>{S\isect T} =  p$ for some $ p\in (0,1)$. Then, for any unit vector $\bvar{u}\in\reals[d]$, and parameters $\alpha,\beta$ that satisfies $\prob{T\isect \abs{\innerprod{\bvar{x}}{\bvar{u}}}\leq \beta}\leq p\leq \prob{T\isect\abs{\innerprod{\bvar{x}}{\bvar{u}}}\geq \alpha}$, it holds that
            \begin{equation*}
                \E<\distr>{\abs{\innerprod{\bvar{x}}{\bvar{u}}}\cdot\indicator[T, \abs{\innerprod{\bvar{x}}{\bvar{u}}}\leq \beta]} \leq\E<\distr>{\abs{\innerprod{\bvar{x}}{\bvar{u}}}\cdot\indicator[ S,T]}\leq \E<\distr>{\abs{\innerprod{\bvar{x}}{\bvar{u}}}\cdot\indicator[T, \abs{\innerprod{\bvar{x}}{\bvar{u}}}\geq \alpha]}.
            \end{equation*}
        \end{lemma}
        \begin{proof}
            For conciseness of the proof, we denote $\errorregion<\geq t> = \lbr{\bvar{x}\cond T\isect \abs{\innerprod{\bvar{x}}{\bvar{u}}}\geq t}$, $\errorregion<\leq t> = \lbr{\bvar{x}\cond T\isect \abs{\innerprod{\bvar{x}}{\bvar{u}}}\leq t}$, and $\errorregion<S> = \lbr{\bvar{x}\cond S\isect T}$. 
            
            To show the first property, let $\alpha>0$ be such that $ p\leq \prob{T\isect\abs{\innerprod{\bvar{x}}{\bvar{u}}}\geq \alpha} = \prob{\bvar{x}\in\errorregion<\geq\alpha>}$. Then, if $\bvar{x}\in \errorregion<S>\backslash\errorregion<\geq\alpha>$, there must be $\abs{\innerprod{\bvar{x}}{\bvar{u}}}\leq \alpha$. Therefore, we have
            \begin{align*}
                \E<\bvar{x}\sim\distr>{\abs{\innerprod{\bvar{x}}{\bvar{u}}}\cdot\indicator[ S,T]} =& \E{\abs{\innerprod{\bvar{x}}{\bvar{u}}}\cdot\indicator[\bvar{x}\in\errorregion<S>]}
                \\
                =& \E{\abs{\innerprod{\bvar{x}}{\bvar{u}}}\cdot\indicator[\bvar{x}\in\errorregion<S>\isect\errorregion<\geq\alpha>]} + \E{\abs{\innerprod{\bvar{x}}{\bvar{u}}}\cdot\indicator[\bvar{x}\in\errorregion<S>\backslash\errorregion<\geq\alpha>]}
                \\
                \leq& \E{\abs{\innerprod{\bvar{x}}{\bvar{u}}}\cdot\indicator[\bvar{x}\in\errorregion<S>\isect\errorregion<\geq\alpha>]} + \E{\alpha\cdot\indicator[\bvar{x}\in\errorregion<S>\backslash\errorregion<\geq\alpha>]}
                \\
                \cleq[(i)]&\E{\abs{\innerprod{\bvar{x}}{\bvar{u}}}\cdot\indicator[\bvar{x}\in\errorregion<S>\isect\errorregion<\geq\alpha>]} + \E{\abs{\innerprod{\bvar{x}}{\bvar{u}}}\cdot\indicator[\bvar{x}\in\errorregion<\geq\alpha>\backslash\errorregion<S>]}
                \\
                =& \E{\abs{\innerprod{\bvar{x}}{\bvar{u}}}\cdot\indicator[T, \abs{\innerprod{\bvar{x}}{\bvar{u}}}\geq \alpha]}
            \end{align*}
            where inequality (i) holds because $\prob{\bvar{x}\in\errorregion<S>} \leq \prob{\bvar{x}\in\errorregion<\geq\alpha>}$ by construction, which implies $\prob{\bvar{x}\in\errorregion<S>\backslash\errorregion<\geq\alpha>}\leq \prob{\bvar{x}\in\errorregion<\geq\alpha>\backslash\errorregion<S>}$, and every $\bvar{x}\in\errorregion<\geq\alpha>$ satisfies $\abs{\innerprod{\bvar{x}}{\bvar{u}}}\geq \alpha$.

            To prove the second claim, we similarly define $\beta > 0$ be such that $ p\geq \prob{T\isect \abs{\innerprod{\bvar{x}}{\bvar{u}}}\leq \beta} = \prob{\bvar{x}\in\errorregion<\leq\beta>}$. Similar to the case of $\abs{\innerprod{\bvar{x}}{\bvar{u}}}\leq \alpha$, we should notice that, if $\bvar{x}\in \errorregion<S>\backslash\errorregion<\leq \beta>$, there is $\abs{\innerprod{\bvar{x}}{\bvar{u}}}\geq \beta$. Hence, with a similar argument as above, we have
            \begin{align*}
                \E<\bvar{x}\sim\distr>{\abs{\innerprod{\bvar{x}}{\bvar{u}}}\cdot\indicator[S,T]}=&\E{\abs{\innerprod{\bvar{x}}{\bvar{u}}}\cdot\indicator[\bvar{x}\in\errorregion<S>]}
                \\
                =& \E{\abs{\innerprod{\bvar{x}}{\bvar{u}}}\cdot\indicator[\bvar{x}\in\errorregion<S>\isect\errorregion<\leq \beta>]} + \E{\abs{\innerprod{\bvar{x}}{\bvar{u}}}\cdot\indicator[\bvar{x}\in\errorregion<S>\backslash\errorregion<\leq \beta>]}
                \\
                \geq& \E{\abs{\innerprod{\bvar{x}}{\bvar{u}}}\cdot\indicator[\bvar{x}\in\errorregion<S>\isect\errorregion<\leq \beta>]} + \E{\beta\cdot\indicator[\bvar{x}\in\errorregion<S>\backslash\errorregion<\leq \beta>]}
                \\
                \geq& \E{\abs{\innerprod{\bvar{x}}{\bvar{u}}}\cdot\indicator[\bvar{x}\in\errorregion<S>\isect\errorregion<\leq \beta>]} + \E{\abs{\innerprod{\bvar{x}}{\bvar{u}}}\cdot\indicator[\bvar{x}\in\errorregion<S>\backslash\errorregion<\leq \beta>]}
                \\
                =& \E{\abs{\innerprod{\bvar{x}}{\bvar{u}}}\cdot\indicator[T, \abs{\innerprod{\bvar{x}}{\bvar{u}}}\leq \beta]}
            \end{align*}
            which completes the proof.
        \end{proof}

        The following corollary is an immediate result of Proposition \ref{prop:gradient-projection-lower-bound-appendix}.
        \begin{corollary}\label{cor:small-gradient-norm-indicates-optimality-appendix}
            Let $\distr$ be a distribution on $\reals[d]\times\booldomain$ with standard normal $\bvar{x}$-marginal, and $\func{g}<\bvar{w}>[\bvar{x}, y] = y\cdot\bvar{x}<\bvar{w}[\bot]>\cdot\indicator[\bvar{x}\in\hypothesis[\bvar{w}]]$. Suppose $\bvar{v},\bvar{w}\in\reals[d]$ are unit vectors such that $\prob<\sbr{\bvar{x}, y}\sim\distr>{\bvar{x}\in\hypothesis[\bvar{v}]\isect y = 1}\leq\epsilon$ and $\theta(\bvar{v}, \bvar{w})\in[0,\pi/2)$, then, if a unit vector $\bvar{w}$ satisfies that $\norm{\E<(\bvar{x},\lvar{y})\sim\distr>{\func{g}<\bvar{w}>[\bvar{x},\lvar{y}]}}<2>< \frac{2}{5}\epsilon\sqrt{\ln\epsilon^{-1}}$, we have
            \begin{equation*}
                \prob<\sbr{\bvar{x}, y}\sim\distr>{\bvar{x}\in\hypothesis[\bvar{w}]\isect y = 1}< \frac{5}{2}\sbr*{\epsilon\sqrt{\ln\epsilon^{-1}}}^{1/2}
            \end{equation*}
            for some small enough $\epsilon\in[0, 1/e]$.
        \end{corollary}
        \begin{proof}
            By Cauchy's inequality and our assumption, it holds that
            \begin{equation*}
                \innerprod{\E<(\bvar{x},\lvar{y})\sim\distr>{-\func{g}<\bvar{w}>[\bvar{x},\lvar{y}]}}{\bvar*{v}<\bvar{w}[\bot]>}\leq \norm*{\E<(\bvar{x},\lvar{y})\sim\distr>{\func{g}<\bvar{w}>[\bvar{x},\lvar{y}]}}_2 < \frac{2}{5}\epsilon\sqrt{\ln\epsilon^{-1}}
            \end{equation*}
            Then, negating Proposition \ref{prop:gradient-projection-lower-bound} gives the desired result.
        \end{proof}

        Now we are ready to prove that at least one of the halfspaces selector returned by the Projected SGD is close to the optimal one of the classifier $c\in\conceptclass$ in one iteration in Algorithm \ref{algo:conditional-classification-with-homogeneous-halfspaces-for-finite-classes}.
        
        \begin{lemma}[Lemma \ref{lma:psgd-returns-at-least-one-optimal-solution}]\label{lma:psgd-returns-at-least-one-optimal-solution-appendix}
            Let $\distr$ be a distribution on $\reals[d]\times\booldomain$ with standard normal $\bvar{x}$-marginal, and $\func{g}<\bvar{w}>[\bvar{x}, y] = y\cdot\bvar{x}<\bvar{w}[\bot]>\cdot\indicator[\bvar{x}\in\hypothesis[\bvar{w}]]$. Suppose $\bvar{v}\in\reals[d]$ is a unit vectors such that $\prob<\sbr{\bvar{x}, y}\sim\distr>{\bvar{x}\in\hypothesis[\bvar{v}]\isect y = 1}\leq\epsilon$, if $T \geq \sbr*{4d + \ln\sbr*{2/\delta}}/\epsilon^4$, $N \geq 1600\ln^2\sbr*{4T/\delta}/\epsilon^2$, and $\theta(\bvar{v},\bvar{w}(0))\in[0, \pi/2)$, it holds that at least one of $\bvar{w}\in\parameterset{W} = \lbr*{\bvarseq{w}(T)}$ returned by algorithm \ref{algo:projected-stochastic-gradient-descent-for-minimizing-convex-surrogate-loss} will satisfies
            \begin{equation*}
                \prob<\sbr{\bvar{x}, y}\sim\distr>{\bvar{x}\in\hypothesis[\bvar{w}]\isect y = 1}\leq \frac{5}{2}\sbr*{\epsilon\sqrt{\ln\epsilon^{-1}}}^{1/2}
            \end{equation*}
            with probability at least $1 - \delta$ for some sufficiently small $\epsilon\in[0,1/e]$.
        \end{lemma}
        \begin{proof}
            By proposition \ref{prop:prop:upper-bound-on-the-norm-of-statistic-relu-gradient-appendix} with $T \geq \sbr*{4d + \ln\sbr*{2/\delta}}/\epsilon^4$, there exists a $\bvar{w}\in\parameterset{W}$ such that $\norm{\E<\distr>{\func{g}<\bvar{w}>[\bvar{x},\lvar{y}]}}<2>\leq \epsilon$ with probability at least $1 - \delta/2$. Suppose $\bvar{w}\in\parameterset{W}$ is indexed in the same order that the iterations happened in algorithm \ref{algo:projected-stochastic-gradient-descent-for-minimizing-convex-surrogate-loss}, and let $\bvar{w}(t)$ be the first parameter in that order such that $\norm{\E<\distr>{\func{g}<\bvar{w}(t)>[\bvar{x},\lvar{y}]}}<2>\leq \epsilon$. 
            
            Consider now the subset $S = \lbr*{\bvarseq{w}(t-1)}\subset\parameterset{W}$, there are two possible cases, either there already exists a $\bvar{w}\in S$ such that $\prob{\hypothesis[\bvar{x},\bvar{w}]\geq 0\isect y = 1}\leq \frac{5}{2}\sbr*{\epsilon\sqrt{\ln\epsilon^{-1}}}^{1/2}$, or none of them have low error rate. The former case already satisfies the desired requirement, hence, we will focus on prove the latter case also implies the existence of a good parameter.

            We first show that, by induction, every $\bvar{w}(i)\in\lbr*{\bvarseq{w}[][0](t)}$ satisfies $\theta(\bvar{v}, \bvar{w}(i))\in [0, \pi/2)$ with high probability. 
            
            For $\bvar{w}(0) = \bvar{e}<1>$, since we assumed $\theta(\bvar{e}<1>, \bvar{v})\in[0, \pi/2)$, it is trivially true.

            Inductively, assume $\theta(\bvar{v}, \bvar{w}(i))\in [0, \pi/2)$. Then, due to our previous assumption in this case that $\prob{\hypothesis[\bvar{x},\bvar{w}(i)]\geq 0\isect y = 1} > \frac{5}{2}\sbr*{\epsilon\sqrt{\ln\epsilon^{-1}}}^{1/2}$ for every $\bvar{w}(i)\in S$ and some sufficiently small $\epsilon$, we can refer proposition \ref{prop:gradient-projection-lower-bound-appendix} to obtain $\innerprod{\E{-\func{g}<\bvar{w}(i)>[\bvar{x},\lvar{y}]}}{\bvar*{v}<\bvar{w}(i)[\bot]>} \geq \frac{2}{5}\epsilon\sqrt{\ln\epsilon^{-1}}$. Notice that, in algorithm \ref{algo:projected-stochastic-gradient-descent-for-minimizing-convex-surrogate-loss}, the update step in algorithm \ref{algo:projected-stochastic-gradient-descent-for-minimizing-convex-surrogate-loss} tells us that
            \begin{equation*}
                \bvar{u}(i+1) = \bvar{w}(i) + \beta\E<(\bvar{x}, \lvar{y})\sim\distr*(i+1)>{-\func{g}<\bvar{w}(i)>[\bvar{x},\lvar{y}]}
            \end{equation*}
            Then, by lemma \ref{lma:concentration-bound-on-the-projected-gradient} with $N\geq 15/2\epsilon\sqrt{\ln\epsilon^{-1}}$ as well as $N \geq 41\epsilon^{-2}\ln\epsilon^{-1}\ln^2(2N)$, we have that
            \begin{equation*}
                \prob<\distr>{\abs{\innerprod{\E<\distr*(i+1)>{\func{g}<\bvar{w}(i)>[\bvar{x},\lvar{y}]} - \E<\distr>{\func{g}<\bvar{w}(i)>[\bvar{x},\lvar{y}]}}{\bvar*{v}<\bvar{w}(i)[\bot]>}}\geq \frac{2}{5}\epsilon\sqrt{\ln\epsilon^{-1}}}\leq 2\exp\sbr{-\epsilon\sqrt{N\ln\epsilon^{-1}}/40}
            \end{equation*}
            which implies $\innerprod{\E<\distr*(i+1)>{-\func{g}<\bvar{w}(i)>[\bvar{x},\lvar{y}]}}{\bvar*{v}<\bvar{w}(i)[\bot]>}\geq 0$ with probability at least $1 - 2\exp\sbr{-\epsilon\sqrt{N\ln\epsilon^{-1}}/40}$ for some sufficiently small $\epsilon$. Therefore, we have  
            \begin{align*}
                \innerprod{\E<\distr*(i+1)>{-\func{g}<\bvar{w}(i)>[\bvar{x},\lvar{y}]}}{\bvar{v}} =& \innerprod{\E<\distr*(i+1)>{-\func{g}<\bvar{w}(i)>[\bvar{x},\lvar{y}]}}{\bvar{v}<\bvar{w}(i)[\bot]>}\\
                =& \norm{\bvar{v}<\bvar{w}[\bot](i)>}<2>\innerprod{\E<\distr*(i+1)>{-\func{g}<\bvar{w}(i)>[\bvar{x},\lvar{y}]}}{\bvar*{v}<\bvar{w}(i)[\bot]>}\\
                \geq& 0
            \end{align*}
            Now, by lemma \ref{lma:diakonikolas-correlation-improvement}, we can conclude that $\innerprod{\bvar{w}(i+1)}{\bvar{v}}\geq \innerprod{\bvar{w}(i)}{\bvar{v}}$, which implies $\theta(\bvar{w}(i+1),\bvar{v})\in[0,\pi/2)$ with probability at least $1 - 2\exp\sbr{-\epsilon\sqrt{N\ln\epsilon^{-1}}/40}$. Taking a union bound over all $T\geq t$ iterations gives that $\theta(\bvar{w}(t),\bvar{v})\in[0,\pi/2)$ with probability $1 - 2T\exp\sbr{-\epsilon\sqrt{N\ln\epsilon^{-1}}/40}$.

            At last, combining $\theta(\bvar{w}(t),\bvar{v})\in[0,\pi/2)$ and the assumption that $\norm{\E<\distr>{\func{g}<\bvar{w}(t)>[\bvar{x},\lvar{y}]}}<2>\leq \epsilon$, corollary \ref{cor:small-gradient-norm-indicates-optimality-appendix} gives $\prob{\hypothesis[\bvar{x},\bvar{w}]\geq 0\isect y = 1}\leq \frac{5}{2}\sbr*{\epsilon\sqrt{\ln\epsilon^{-1}}}^{1/2}$. Taking $N \geq 1600\ln^2\sbr*{4T/\delta}/\epsilon^2$ satisfies both $N\geq 15/2\epsilon\sqrt{\ln\epsilon^{-1}}$ and $N \geq 41\epsilon^{-2}\ln\epsilon^{-1}\ln^2(2N)$ because $T\geq \sbr*{4d + \ln\sbr*{2/\delta}}/\epsilon^4$, which completes the proof.
        \end{proof}

        We need the following lemma to aid the above argument.
        \begin{lemma}[Correlation Improvement \citet{diakonikolas2020polynomial}]\label{lma:diakonikolas-correlation-improvement}
            For unit vectors $\bvar{v}, \bvar{w}\in\reals[d]$, let $\bvar{u}\in \reals[d]$ be such that $\innerprod{\bvar{u}}{\bvar{v}}\geq c$, $\innerprod{\bvar{u}}{\bvar{w}} = 0$, and $\norm{\bvar{u}}_2\leq 1$, with $c > 0$. Then, for $\bvar{w}' = \bvar{w} + \lambda \bvar{u}$, we have that $\innerprod{\bvar*{w}'}{\bvar{v}} \geq \innerprod{\bvar{w}}{\bvar{v}} + \lambda c/8$.
        \end{lemma}

    \section{Analysis of Algorithm \ref{algo:conditional-classification-for-sparse-linear-classes}}\label{sec:analysis-of-algorithm-3}
        We prove the generalization of our conditional learning algorithms from finite classes to sparse linear classes in this section.
        \begin{theorem}[Theorem \ref{thm:generalization-to-sparse-linear-classes}]
            Let $\distr$ be a distribution on $\reals[d]\times\booldomain$ with standard normal $\bvar{x}$-marginal, and $\conceptclass$ be a class of sparse linear classifiers on $\reals[d]\times\booldomain$ with sparsity $s = \bigO{1}$. If there exist a unit vector $\bvar{v}\in\reals[d]$ and a classifier $c\in \conceptclass$ such that, for some sufficiently small $\epsilon\in[0, 1/e]$, $\prob<\sbr{\bvar{x}, y}\sim\distr>{\bvar{x}\in\hypothesis[\bvar{v}]\isect c(\bvar{x})\neq y}\leq\epsilon$, then, with at most $\poly{d, 1/\epsilon, 1/\delta}$ examples, Algorithm \ref{algo:conditional-classification-for-sparse-linear-classes} will return a $\bvar{w}(c)$, with probability at least $1 - \delta$, such that $\prob<\sbr{\bvar{x}, y}\sim\distr>{\bvar{x}\in\hypothesis[\bvar{w}(c)]\isect c(\bvar{x})\neq y} = \bigO*{\sqrt{\epsilon}}$ and run in time $\poly{d, 1/\epsilon, 1/\delta}$. 
        \end{theorem}
        \begin{proof}
            We first show that the returned list of Algorithm \ref{alg:robust-list-learn} will contain a classifier $c'$ such that $\prob<\sbr{\bvar{x}, y}\sim\distr>{\bvar{x}\in\hypothesis[\bvar{v}]\isect c'(\bvar{x})\neq y}\leq 2\epsilon$.

            We decompose distribution $\distr$ into a convex combination of an inlier distribution $\distr[*]$ and a outlier distribution $\Tilde{\distr}$ in the following way. Let $\distr[*]$ be a distribution on $\reals[d]\times\booldomain$ with standard normal $\bvar{x}$-marginal such that its labels are generated by $c(\bvar{x})$, while $\Tilde{\distr}$ be any distribution on $\reals[d]\times\booldomain$ with standard normal $\bvar{x}$-marginals. Observe that, since $\prob{\bvar{x}\in\hypothesis[\bvar{v}]\isect c(\bvar{x})\neq y}\leq\epsilon$ and $\prob{\hypothesis[\bvar{v}]} = 1/2$, there are at least $1/2 - \epsilon$ fraction (weighted by Gaussian density) of the labels of $\distr$ is consistent with $c(\bvar{x})$. Therefore, there must exist some $\alpha \geq 1 - \epsilon$ such that the labels of $\distr<\bvar{x}>$ can be generated by selecting labels from $\distr[*]$ with probability mass $\alpha$ and from $\Tilde{\distr}$ with probability mass $1 - \alpha$, namely $\distr = \alpha\distr[*] + (1 - \alpha)\Tilde{\distr}$. 
            
            Hence, we can refer Theorem \ref{thm:robust-list-learn} and Definition \ref{def:robust-list-learning} to conclude that there exists a classifier $c'$ in the returned list of Algorithm \ref{alg:robust-list-learn} satisfies $\prob{\bvar{x}\in\hypothesis[\bvar{v}]\isect c'(\bvar{x})\neq y}\leq 2\epsilon$. Meanwhile, it is easy to see that Algorithm \ref{alg:robust-list-learn} takes only $\poly{d, 1/\epsilon,1/\delta}$ examples and runs in $\poly{d, 1/\epsilon, 1/\delta}$ time since $\alpha$ is a constant.

            At last, by Theorem \ref{thm:main-theorem-appendix}, we obtained the claimed result.
        \end{proof}

    \section{Analysis of Hardness Results}
    \label{sec:analysis-of-hardness-results}
        We denote $\integer<q>:=\lbr{0, 1, \ldots, q-1}$, $\reals<q>:=[0, q)$, and $\mathrm{mod}_q:\reals[d]\rightarrow \reals<q>[d]$ to be the function that applies $\mathrm{mod}_q$ operation on each coordinate of $\bvar{x}$.

        \begin{assumption}[Sub-exponential LWE Assumption]\label{asp:sub-exponential-assumption-of-lwe}
            For $q,\kappa\in\naturals, \alpha\in(0, 1)$ and $C > 0$ being a sufficiently large constant, the problem LWE$(2^{O(n^\alpha)}, \integer<q>[d], \integer<q>[d], \gaussian[][0][\sigma], \mathrm{mod}_q)$ with $q\leq d^{\kappa}$ and $\sigma = C\sqrt{d}$ cannot be solved in $2^{\bigO{d^\alpha}}$ time with $2^{\bigO{-d^\alpha}}$ advantage.
        \end{assumption}
        
        For convenience, we restate the notations been used in section \ref{sec:conditional-classification-with-general-halfspaces-is-hard} at first.

        For simplicity, we define $y \equiv \indicator[c(\bvar{x}) \neq \lvar{y} ']$ for $(\bvar{x},y')\sim\distr '$ and only consider the distribution $(\bvar{x}, y)\sim\distr$ for the rest of this section. Notice that, for agnostic setting, since $\distr'$ is adversarial, $\distr$ is also adversarial in the worst case. Therefore, such replacement does not affect the difficulty of the problems we concerned about.

        Normally, for the problem of agnostic classification, one would consider its loss function to be the expectation of the disagreement between the classifier and the labelling. However, it is more convenient for us to consider a labelling $y = 1$ as an "occurring of error" and, hence, define the loss function in terms of agreement to compare with the conditional classification loss.  Specifically, for any binary classifier as a subset $S\subseteq\reals[d]$ and any distribution $\distr$ on $\reals[d]\times\booldomain$, we define the classification loss
        \begin{equation}
            \err<\distr>[S] = \prob<(\bvar{x}, \lvar{y})\sim\distr>{y = \indicator[\bvar{x}\in S]}.\label{eq:definition-of-classification-loss-appendix}
        \end{equation}
        It worth to mention that this definition of classification loss is essentially the same as the ``traditional'' one that defined in terms of disagreement since we can convert from one to another by simply negating the labelling.

        Analogously, for any binary classifiers as a subsets $S, T\subseteq\reals[d]$ and any distribution $\distr$ on $\reals[d]\times\booldomain$, we denote the conditional classification loss as 
        \begin{equation}
            \err<\distr| T>[S] = \prob<(\bvar{x}, \lvar{y})\sim\distr>{y = \indicator[\bvar{x}\in S]\cond \bvar{x}\in T}.\label{eq:notation-of-conditional-classification-loss-appendix}
        \end{equation}
        For simplicity, we write $\err<\distr| T>$ instead of $\err<\distr| T>[S]$ when $S\equiv T$.

        \begin{lemma}[Lemma \ref{lma:classification-error-decomposition}]\label{lma:classification-error-decomposition-appendix}
            Let $\distr$ be any distribution on $\reals[d]\times\booldomain$ and $\subsets$ be any subset of $\reals[d]$, we have $\err<\distr>[\subsets] =2\err<\distr|\subsets>\prob<\distr>{\bvar{x}\in\subsets} + \prob<\distr>{y = 0} - \prob<\distr>{\bvar{x}\in\subsets}$ as well as $\err<\distr>[\subsets] =2\err<\distr|\subsets*>[\subsets]\prob<\distr>{\bvar{x}\in\subsets*} + \prob<\distr>{y = 1} - \prob<\distr>{\bvar{x}\in\subsets*}$.
        \end{lemma}
        \begin{proof}
            By the law of total probability and definition \eqref{eq:definition-of-classification-loss-appendix}, we have
            \begin{align}
                \err<\distr>[\subsets] =& \prob{y = \indicator[\bvar{x}\in\subsets]}\notag\\
                =& \prob{y = 1 \isect \bvar{x}\in\subsets} + \prob{y = 0 \isect \bvar{x}\notin\subsets}\label{eq:initial-classification-error-decomposition}
            \end{align}
            Again, by the law of total probability, we have that 
            \begin{align}
                \prob{y = 0 \isect \bvar{x}\notin\subsets} =& \prob{y = 0} - \prob{y = 0\isect\bvar{x}\in\subsets}\notag\\
                =&\prob{y = 0} - \prob{\bvar{x}\in\subsets} + \prob{y = 1\isect\bvar{x}\in\subsets}\label{eq:joint-classification-error-decomposition}
            \end{align}
            Taking equation \eqref{eq:joint-classification-error-decomposition} back into \eqref{eq:initial-classification-error-decomposition} gives
            \begin{align*}
                \err<\distr>[\subsets] =& 2\prob{y = 1\cond\bvar{x}\in\subsets}\prob{\bvar{x}\in\subsets} + \prob{y = 0} - \prob{\bvar{x}\in\subsets}\notag\\
                =&2\err<\distr|\subsets>\prob{\bvar{x}\in\subsets} + \prob{y = 0} - \prob{\bvar{x}\in\subsets}
            \end{align*}
            where the last equation holds due to definition \eqref{eq:notation-of-conditional-classification-loss-appendix}. Similar to equation \eqref{eq:joint-classification-error-decomposition}, we have
            \begin{equation*}
                \prob{y = 1 \isect \bvar{x}\in\subsets} = \prob{y = 1} - \prob{\bvar{x}\notin \subsets} + \prob{y = 0 \isect \bvar{x}\notin\subsets}
            \end{equation*}
            which, when plugging back to equation \ref{eq:initial-classification-error-decomposition}, gives
            \begin{align*}
                \err<\distr>[\subsets] =& 2\prob{y = 0 \cond \bvar{x}\notin\subsets}\prob{\bvar{x}\notin\subsets} + \prob{y = 1} - \prob{\bvar{x}\notin \subsets}\\
                =& 2\prob{y = \indicator[\bvar{x}\in\subsets] \cond \bvar{x}\in\subsets*}\prob{\bvar{x}\in\subsets*} + \prob{y = 1} - \prob{\bvar{x}\in\subsets*}\\
                =& 2\err<\distr|\subsets*>[\subsets]\prob{\bvar{x}\in\subsets*} + \prob{y = 1} - \prob{\bvar{x}\in\subsets*}.
            \end{align*}
            The proof is completed.
        \end{proof}

        \begin{proposition}[Proposition \ref{prop:reducing-linear-classification-to-conditional-classification}]\label{prop:reducing-linear-classification-to-conditional-classification-appendix}
            Let $\distr$ be any distribution on $\reals[d]\times\booldomain$, $\hypothesisclass$ be any subset of the power set of $\reals[d]$ closed under complement, and define $\hypothesisclass<\distr>[a,b] = \lbr{\subsets\in\hypothesisclass\cond \prob<\distr>{\bvar{x}\in\subsets}\in[\lvar{a},\lvar{b}]}$ for any $0\leq a\leq b\leq 1$. For any $0\leq\lvar{a}\leq\lvar{b}\leq1$ and $\epsilon,\delta > 0$, given sample access to $\distr$, if there exists an algorithm $\algo<1>[\epsilon, \delta, a,b]$ runs in time $\poly{d, 1/\epsilon,1/\delta}$, and outputs a subset $\subsets<1>\in\hypothesisclass<\distr>[a,b]$ such that $\err<\distr|\subsets<1>> \leq \min_{\subsets\in\hypothesisclass<\distr>[a,b]}\err<\distr|\subsets> + \epsilon$ with probability as least $1 - \delta$, there exists another algorithm $\algo<2>[\epsilon, \delta]$, runs in time $\poly{d, 1/\epsilon,1/\delta}$, and outputs a subset $\subsets<2>\in\hypothesisclass$ such that $\err<\distr>[\subsets<2>] \leq \min_{\subsets\in\hypothesisclass}\err<\distr>[\subsets] + 6\epsilon$ with probability at least $1 - \delta$.
        \end{proposition}
        \begin{proof}
            We prove the proposition by showing that there exists a efficient reduction from the problem of agnostic classification to conditional classification in terms of their loss functions.

            Fix a subset $\subsets[*]\in\hypothesisclass$ such that $\subsets[*] = \argmin_{\subsets\in\hypothesisclass}\err<\distr>[\subsets]$ and define $p = \prob{\bvar{x}\in\subsets[*]}$. Then, let $\lvar{p}<l>, \lvar{p}<u>\geq 0$ be any constants such that $\lvar{p}<u> - \lvar{p}<l> = \epsilon$ as well as $p \in [\lvar{p}<l>, \lvar{p}<u>]$. 
            
            Consider now another subset $\subsets '\in\hypothesisclass<\distr>[\lvar{p}<l>,\lvar{p}<u>]$ such that $\subsets ' = \argmin_{\subsets\in\hypothesisclass<\distr>[\lvar{p}<l>,\lvar{p}<u>]}\err<\distr|\subsets>$. Notice that $\subsets[*]$ is a feasible solution for the conditional classification problem on $\hypothesisclass<\distr>[\lvar{p}<l>,\lvar{p}<u>]$, i.e. $\subsets[*]\in\hypothesisclass<\distr>[\lvar{p}<l>,\lvar{p}<u>]$, because $\prob{\bvar{x}\in\subsets[*]} = p\in[\lvar{p}<l>,\lvar{p}<u>]$ by construction. 
            
            Let $\subsets<1>$ be the subset returned by algorithm $\algo<1>[\epsilon, \delta, \lvar{p}<l>, \lvar{p}<u>]$. Then, with probability at least $1 - \delta$, there is
            \begin{align}
                \err<\distr>[\subsets<1>] =& 2\err<\distr|\subsets<1>>\prob{\bvar{x}\in\subsets<1>} + \prob{y = 0} - \prob{\bvar{x}\in\subsets<1>}\notag\\
                \cleq[(i)]& 2\sbr{\err<\distr|\subsets '> + \epsilon}\prob{\bvar{x}\in\subsets<1>} + \prob{y = 0} - \prob{\bvar{x}\in\subsets<1>}\notag\\
                \cleq[(ii)]& 2\err<\distr|\subsets[*]>\prob{\bvar{x}\in\subsets<1>} + \prob{y = 0} - \prob{\bvar{x}\in\subsets<1>} + 2\epsilon\notag\\
                \cleq[(iii)]& 2\err<\distr|\subsets[*]>\sbr{p + \epsilon} + \prob{y = 0} - (p - \epsilon) + 2\epsilon\notag\\
                \cleq[(iv)]& 2\err<\distr|\subsets[*]>\prob{\bvar{x}\in\subsets[*]} + \prob{y = 0} - \prob{\bvar{x}\in\subsets[*]} + 5\epsilon\notag\\
                =& \err<\distr>[\subsets[*]] + 5\epsilon\label{eq:reduction-through-loss-function}
            \end{align}
            where the first equation is derived by lemma \ref{lma:classification-error-decomposition-appendix}, inequality (i) holds due to the error guarantee of algorithm $\algo<1>[\epsilon, \delta, \lvar{p}<l>, \lvar{p}<u>]$, inequality (ii) holds because of the optimality of $\subsets '$ as well as $\subsets[*]\in\hypothesisclass<\distr>[\lvar{p}<l>,\lvar{p}<u>]$, inequality (iii) holds since algorithm $\algo<1>[\epsilon, \delta, \lvar{p}<l>, \lvar{p}<u>]$ guarantees $\subsets<1>\in \hypothesisclass<\distr>[\lvar{p}<l>,\lvar{p}<u>]$, which implies $\lvar{p}<l>\leq\prob{\bvar{x}\in \subsets<1>}\leq\lvar{p}<u>$, and, by definition, there are $\lvar{p}<l>\geq p - \epsilon$, $\lvar{p}<u>\leq p + \epsilon$, inequality (iv) holds because $p = \prob{\bvar{x}\in\subsets[*]}$ by definition as well as $\err<\distr|\subsets[*]>=\prob{y = 1\cond\bvar{x}\in\subsets[*]}\leq 1$, and the last equation is, again, by referring lemma \ref{lma:classification-error-decomposition-appendix}.

            Although we do not know what value should $p$ take exactly, we only need to guess a small range where $p$ lies in to make inequality \eqref{eq:reduction-through-loss-function} holds with high probability. Specifically, we construct algorithm $\algo<2>[\epsilon, \delta]$ by using algorithm $\algo<1>$ as a subroutine in the following way. 
            
            For $k = 1, 2, \ldots, \ceil{1/\epsilon}$, we run algorithm $\algo<1>[\epsilon, \epsilon\delta/2, (k - 1)\epsilon, k\epsilon]$. Observe that, when we ``guessed'' the correct $k$ such that $p\in[(k - 1)\epsilon, k\epsilon]$, inequality \eqref{eq:reduction-through-loss-function} must holds with probability at least $1 - \epsilon\delta/2$ because of the parameters we passed into $\algo<1>$. Let $\subsets(k)$ be the solution returned by algorithm $\algo<1>$ during the $k$th call, we construct an empirical distribution $\distr*\sample\distr$ and choose $\subsets<2>$ such that $\err<\distr*>[\subsets<2>] \leq \min_{k\in\mbr{\ceil{1/\epsilon}}}\err<\distr*>[\subsets(k)]$. Notice that we only need the sample size of $\distr*$ to be polynomially large to guarantee that $\err<\distr>[\subsets<2>] \leq \min_{k\in\mbr{\ceil{1/\epsilon}}}\err<\distr>[\subsets(k)] + \epsilon$ with probability at least $1 - \delta/2$ by lemma \ref{lma:chernoff-bound} (Chernoff Bound). Further, by a union bound over all $\ceil{1/\epsilon}$ calls of algorithm $\algo<1>$ and the estimation of classification error on $\distr*$, we have, with probability at least $1 - \delta$, that
            \begin{align*}
                \err<\distr>[\subsets<2>] \leq& \min_{k\in\ceil{1/2\epsilon}}\err<\distr>[\subsets(k)] + \epsilon\\
                \cleq[(i)]& \err<\distr>[\subsets[*]] + 6\epsilon\\
                =&\min_{\subsets\in\hypothesisclass}\err<\distr>[\subsets] + 6\epsilon
            \end{align*}
            where inequality (i) alone holds with probability at least $1 - \delta/2$ because the second argument, $\epsilon\delta/2$, we passed in algorithm $\algo<1>$ guarantees that inequality \eqref{eq:reduction-through-loss-function} holds with probability at least $1 - \epsilon\delta/2$ when we guessed $p = \prob{\bvar{x}\in\subsets[*]}$ correctly, and taking a union bound over the $\ceil{1/\epsilon}$ guesses gives probability at least $1 - \delta/2$. It is easy to see that each call, $\algo<1>[\epsilon, \epsilon\delta, (k - 1)\epsilon, k\epsilon]$, runs in time $\poly{d, 1/\epsilon, 1/\epsilon\delta}$, and we only called $\algo<1>$ for at most $\ceil{1/\epsilon}$ times, the resulting running time is still $\poly{d, 1/\epsilon, 1/\delta}$, which completes the proof.
        \end{proof}

        \begin{claim}[Claim \ref{clm:reduction-in-multiplicative-form}]\label{clm:reduction-in-multiplicative-form-appendix}
            Let $\distr$ be any distribution on $\reals[d]\times\booldomain$, $\hypothesisclass$ be any subset of the power set of $\reals[d]$ closed under complement, and define $\hypothesisclass<\distr>[a,b] = \lbr{\subsets\in\hypothesisclass\cond \prob<\distr>{\bvar{x}\in\subsets}\in[\lvar{a},\lvar{b}]}$ for any $0\leq a\leq b\leq 1$. For any $0\leq\lvar{a}\leq\lvar{b}\leq1$, $\alpha,\epsilon,\delta > 0$, given sample access to $\distr$, if there exists an algorithm $\algo<1>[\alpha, \delta, a,b]$, runs in time $\poly{d, 1/\alpha,1/\delta}$, and outputs a subset $\subsets<1>\in\hypothesisclass<\distr>[a,b]$ such that $\err<\distr|\subsets<1>> \leq (1+\alpha)\min_{\subsets\in\hypothesisclass<\distr>[a,b]}\err<\distr|\subsets>$ with probability as least $1 - \delta$, there exists another algorithm $\algo<2>[\alpha,\epsilon, \delta]$, runs in time $\poly{d, 1/\alpha,1/\epsilon,1/\delta}$, and outputs a subset $\subsets<2>\in\hypothesisclass$ such that $\err<\distr>[\subsets<2>] \leq (1 + \alpha) \sbr*{\min_{\subsets\in\hypothesisclass}\err<\distr>[\subsets] + 4\epsilon}$ with probability at least $1 - \delta$.
        \end{claim}
        \begin{proof}
            Fix a subset $\subsets[*]\in\hypothesisclass$ such that $\subsets[*] = \argmin_{\subsets\in\hypothesisclass}\err<\distr>[\subsets]$ and define $p = \prob{\bvar{x}\in\subsets[*]}$. This proof generally follows the same strategy of the analysis of proposition \ref{prop:reducing-linear-classification-to-conditional-classification-appendix}. However, differing from the proof of proposition \ref{prop:reducing-linear-classification-to-conditional-classification-appendix}, to complete the multiplicative reduction, we have to deal with two cases, $2\err<\distr|\subsets>\prob<\distr>{\bvar{x}\in\subsets}\leq \err<\distr>[\subsets]$ and $2\err<\distr|\subsets*>[\subsets]\prob<\distr>{\bvar{x}\in\subsets*}\leq \err<\distr>[\subsets]$, because $\err<\distr>[\subsets]$ can be expressed in two forms according to lemma \ref{lma:classification-error-decomposition-appendix}. 
            
            Briefly speaking, when prove the additive reduction, the additive error will be carried through from conditional classification loss to classification loss no matter if $2\err<\distr|\subsets>\prob<\distr>{\bvar{x}\in\subsets} \leq \err<\distr>[\subsets]$ because $\err<\distr>[\subsets]$ is affinely related to $\err<\distr|\subsets>$ by lemma \ref{lma:classification-error-decomposition-appendix}. However, whether a multiplicative error can be passed from one loss to another depends on whether $2\err<\distr|\subsets>\prob<\distr>{\bvar{x}\in\subsets}\leq \err<\distr>[\subsets]$, which, of course, is not always true. Nonetheless, it is easy to see either $2\err<\distr|\subsets>\prob<\distr>{\bvar{x}\in\subsets}\leq \err<\distr>[\subsets]$ or $2\err<\distr|\subsets*>[\subsets]\prob<\distr>{\bvar{x}\in\subsets*}\leq \err<\distr>[\subsets]$ based on lemma \ref{lma:classification-error-decomposition-appendix}: observe that $\prob{y = 0} - \prob{\bvar{x}\in\subsets[*]}+\prob{y = 1} - \prob{\bvar{x}\notin\subsets[*]}=0$, so either $\prob{y = 0} - \prob{\bvar{x}\in\subsets[*]}$ or $\prob{y = 1} - \prob{\bvar{x}\notin\subsets[*]}$ must be nonnegative. We show that the multiplicative factor can be preserved through the reduction for both of these cases.
            
            \textbf{Case I}, $\prob{y = 0} - \prob{\bvar{x}\in\subsets[*]}\geq 0$. Let $\lvar{p}<l>, \lvar{p}<u>\geq 0$ be any constants such that $\lvar{p}<u> - \lvar{p}<l> = \epsilon$ as well as $p \in [\lvar{p}<l>, \lvar{p}<u>]$. Consider now another subset $\subsets '\in\hypothesisclass<\distr>[\lvar{p}<l>,\lvar{p}<u>]$ such that $\subsets ' = \argmin_{\subsets\in\hypothesisclass<\distr>[\lvar{p}<l>,\lvar{p}<u>]}\err<\distr|\subsets>$. Notice that $\subsets[*]$ is a feasible solution for the conditional classification problem on $\hypothesisclass<\distr>[\lvar{p}<l>,\lvar{p}<u>]$, i.e. $\subsets[*]\in\hypothesisclass<\distr>[\lvar{p}<l>,\lvar{p}<u>]$, because $\prob{\bvar{x}\in\subsets[*]} = p\in[\lvar{p}<l>,\lvar{p}<u>]$ by construction. 

            Let $\subsets<1>$ be the subset returned by algorithm $\algo<1>[\epsilon, \delta, \lvar{p}<l>, \lvar{p}<u>]$. Then, with probability at least $1 - \delta$, there is
            \begin{align}
                \err<\distr>[\subsets<1>] =& 2\err<\distr|\subsets<1>>\prob{\bvar{x}\in\subsets<1>} + \prob{y = 0} - \prob{\bvar{x}\in\subsets<1>}\notag\\
                \cleq[(i)]& 2\sbr{1 + \alpha}\err<\distr|\subsets '>\prob{\bvar{x}\in\subsets<1>} + \prob{y = 0} - \prob{\bvar{x}\in\subsets<1>}\notag\\
                \cleq[(ii)]& 2\sbr{1 + \alpha}\err<\distr|\subsets[*]>\prob{\bvar{x}\in\subsets<1>} + \prob{y = 0} - \prob{\bvar{x}\in\subsets<1>}\notag\\
                \cleq[(iii)]& 2\sbr{1 + \alpha}\err<\distr|\subsets[*]>\sbr{p + \epsilon} + \prob{y = 0} - (p - \epsilon)\notag\\
                \cleq[(iv)]& 2\sbr{1 + \alpha}\err<\distr|\subsets[*]>\prob{\bvar{x}\in\subsets[*]} + \sbr{1 + \alpha}\sbr{\prob{y = 0} - \prob{\bvar{x}\in\subsets[*]}} + 3\sbr{1 + \alpha}\epsilon\notag\\
                =& \sbr{1 + \alpha}\sbr{\err<\distr>[\subsets[*]] + 3\epsilon}\label{eq:reduction-through-loss-function-in-multiplicative-form-1}
            \end{align}
            where the first equation is derived by lemma \ref{lma:classification-error-decomposition-appendix}, inequality (i) holds due to the error guarantee of algorithm $\algo<1>[\epsilon, \delta, \lvar{p}<l>, \lvar{p}<u>]$, inequality (ii) holds because of the optimality of $\subsets '$ as well as $\subsets[*]\in\hypothesisclass<\distr>[\lvar{p}<l>,\lvar{p}<u>]$, inequality (iii) holds since algorithm $\algo<1>[\epsilon, \delta, \lvar{p}<l>, \lvar{p}<u>]$ guarantees $\subsets<1>\in \hypothesisclass<\distr>[\lvar{p}<l>,\lvar{p}<u>]$, which implies $\lvar{p}<l>\leq\prob{\bvar{x}\in \subsets<1>}\leq\lvar{p}<u>$, and, by definition, there are $\lvar{p}<l>\geq p - \epsilon$, $\lvar{p}<u>\leq p + \epsilon$, inequality (iv) holds because $p = \prob{\bvar{x}\in\subsets[*]}$ by definition, $\err<\distr|\subsets[*]>=\prob{y = 1\cond\bvar{x}\in\subsets[*]}\leq 1$, and $\prob{y = 0} - \prob{\bvar{x}\in\subsets[*]}\geq 0$ by assumption, the last equation is, again, by referring lemma \ref{lma:classification-error-decomposition-appendix}.

            \textbf{Case II}, $\prob{y = 1} - \prob{\bvar{x}\notin\subsets[*]}\geq 0$. Let $\lvar{p}<l>, \lvar{p}<u>\geq 0$ be any constants such that $\lvar{p}<u> - \lvar{p}<l> = \epsilon$ as well as $1 - p \in [\lvar{p}<l>, \lvar{p}<u>]$. Further, let $\distr<0>$ be the distribution on $\reals[d]\times\booldomain$ constructed by flipping the labels of $\distr$. Notice that, for any $\subsets\in\hypothesisclass$, we have, by definition \ref{eq:notation-of-conditional-classification-loss-appendix}, that
            \begin{align}
                \err<\distr<0>|\subsets> =& \prob<(\bvar{x},\lvar{y})\sim\distr<0>>{y = \indicator[\bvar{x}\in\subsets]\cond\bvar{x}\in\subsets}\notag\\
                =& \prob<(\bvar{x},\lvar{y})\sim\distr<0>>{y = 1\cond\bvar{x}\in\subsets}\notag\\
                \ceq[(i)]& \prob<(\bvar{x},\lvar{y})\sim\distr>{y = 0\cond\bvar{x}\in\subsets}\notag\\
                =& \prob<(\bvar{x},\lvar{y})\sim\distr>{y = \indicator[\bvar{x}\in\subsets*]\cond\bvar{x}\in\subsets}\notag\\
                =&\err<\distr|\subsets>[\subsets*]\label{eq:transform-conditional-loss-by-flipping-labels}
            \end{align}
            where equation (i) is because $\distr<0>$ has reversed labelling from $\distr$ so that every $y = 1$ in $\distr<0>$ is $y = 0$ in $\distr$, and the last equation is by definition \eqref{eq:notation-of-conditional-classification-loss-appendix}.
            
            Consider now another subset $\subsets '\in\hypothesisclass<\distr<0>>[\lvar{p}<l>,\lvar{p}<u>]$ such that $\subsets ' = \argmin_{\subsets\in\hypothesisclass<\distr<0>>[\lvar{p}<l>,\lvar{p}<u>]}\err<\distr<0>|\subsets>$. Notice that $\subsets*[*]$ is a feasible solution for the conditional classification problem on $\hypothesisclass<\distr>[\lvar{p}<l>,\lvar{p}<u>]$, i.e. $\subsets[*]\in\hypothesisclass<\distr>[\lvar{p}<l>,\lvar{p}<u>]$, because $\prob{\bvar{x}\in\subsets*[*]} = \prob{\bvar{x}\notin\subsets[*]} = 1-p\in[\lvar{p}<l>,\lvar{p}<u>]$ by construction. Observe now that, since $\distr<0>$ only differ from $\distr$ on the labelling, any subset $\subsets\in\hypothesisclass<\distr>[\lvar{p}<l>,\lvar{p}<u>]$ must also be in $\hypothesisclass<\distr<0>>[\lvar{p}<l>,\lvar{p}<u>]$, vice versa. Therefore, we also have $\subsets '\in\hypothesisclass<\distr>[\lvar{p}<l>,\lvar{p}<u>]$ as well as $\subsets*[*]\in\hypothesisclass<\distr<0>>[\lvar{p}<l>,\lvar{p}<u>]$

            Let $\subsets<1>$ be the subset returned by algorithm $\algo<1>[\epsilon, \delta, \lvar{p}<l>, \lvar{p}<u>]$ given sample access to $\distr<0>$. Then, with probability at least $1 - \delta$, there is
            \begin{align}
                \err<\distr>[\subsets*<1>] =& 2\err<\distr|\subsets<1>>[\subsets*<1>]\prob{\bvar{x}\in\subsets<1>} + \prob{y = 1} - \prob{\bvar{x}\in\subsets<1>}\notag\\
                \ceq[(i)]& 2\err<\distr<0>|\subsets<1>>\prob{\bvar{x}\in\subsets<1>} + \prob{y = 1} - \prob{\bvar{x}\in\subsets<1>}\notag\\
                \cleq[(ii)]& 2\sbr{1 + \alpha}\err<\distr<0>|\subsets '>\prob{\bvar{x}\in\subsets<1>} + \prob{y = 1} - \prob{\bvar{x}\in\subsets<1>}\notag\\
                \cleq[(iii)]& 2\sbr{1 + \alpha}\err<\distr<0>|\subsets*[*]>\prob{\bvar{x}\in\subsets<1>} + \prob{y = 1} - \prob{\bvar{x}\in\subsets<1>}\notag\\
                \ceq[(iv)]& 2\sbr{1 + \alpha}\err<\distr|\subsets*[*]>[\subsets[*]]\prob{\bvar{x}\in\subsets<1>} + \prob{y = 1} - \prob{\bvar{x}\in\subsets<1>}\notag\\
                \cleq[(v)]& 2\sbr{1 + \alpha}\err<\distr|\subsets*[*]>[\subsets[*]]\sbr{1 - p + \epsilon} + \prob{y = 1} - (1 - p - \epsilon)\notag\\
                \cleq[(vi)]& 2\sbr{1 + \alpha}\err<\distr|\subsets*[*]>[\subsets[*]]\prob{\bvar{x}\notin\subsets[*]} + \sbr{1 + \alpha}\sbr{\prob{y = 1} - \prob{\bvar{x}\notin\subsets[*]}} + 3\sbr{1 + \alpha}\epsilon\notag\\
                =& \sbr{1 + \alpha}\sbr{\err<\distr>[\subsets[*]] + 3\epsilon}\label{eq:reduction-through-loss-function-in-multiplicative-form-2}
            \end{align}
            where the first equation is derived by lemma \ref{lma:classification-error-decomposition-appendix}, inequality (i) holds through using equation \eqref{eq:transform-conditional-loss-by-flipping-labels} reversely on $\err<\distr|\subsets<1>>[\subsets*<1>]$, inequality (ii) holds due to the error guarantee of algorithm $\algo<1>[\epsilon, \delta, \lvar{p}<l>, \lvar{p}<u>]$, inequality (iii) holds because of the optimality of $\subsets '$ as well as $\subsets*[*]\in\hypothesisclass<\distr<0>>[\lvar{p}<l>,\lvar{p}<u>]$ as we discussed previously, inequality (iv) holds by applying equation \ref{eq:transform-conditional-loss-by-flipping-labels} on $\err<\distr<0>|\subsets*[*]>$, inequality (v) holds since algorithm $\algo<1>[\epsilon, \delta, \lvar{p}<l>, \lvar{p}<u>]$ guarantees $\subsets<1>\in \hypothesisclass<\distr<0>>[\lvar{p}<l>,\lvar{p}<u>]$, which implies $\lvar{p}<l>\leq\prob{\bvar{x}\in \subsets<1>}\leq\lvar{p}<u>$, and, by definition that $1 - p\in[\lvar{p}<l>, \lvar{p}<u>]$, there are $\lvar{p}<l>\geq 1 - p - \epsilon$, $\lvar{p}<u>\leq 1 - p + \epsilon$, inequality (vi) holds because $1 - p = \prob{\bvar{x}\in\subsets*[*]} = \prob{\bvar{x}\notin\subsets[*]}$ by definition, $\err<\distr|\subsets*[*]>[\subsets[*]]=\prob{y = 0\cond\bvar{x}\notin\subsets[*]}\leq 1$, and $\prob{y = 1} - \prob{\bvar{x}\notin\subsets[*]}\geq 0$ by assumption, the last equation is, again, by referring lemma \ref{lma:classification-error-decomposition-appendix}.

            Given inequalities \eqref{eq:reduction-through-loss-function-in-multiplicative-form-1} and \eqref{eq:reduction-through-loss-function-in-multiplicative-form-2}, we can conclude that, when $\prob{\bvar{x}\in \subsets[*]}$ is known, we can always use $\algo<1>$ to find a subset $\subsets$ such that, with probability at least $1 - \delta$, $\err<\distr>[\subsets]\leq \sbr{1 + \alpha}\sbr{\err<\distr>[\subsets[*]] + 3\epsilon}$.

            Then, the construction and analysis of $\algo<2>$ is rather identical to those of $\algo<2>$ in the proof of proposition \ref{prop:reducing-linear-classification-to-conditional-classification-appendix}. We will then refer the proof of proposition \ref{prop:reducing-linear-classification-to-conditional-classification-appendix} to completes the analysis.
        \end{proof}
        
    \section{Gaussian Properties And Concentration Tools}
        In this section, we show some common properties of Gaussian distributions for completeness.
        
        \begin{fact}[Gaussian Tail Bound]\label{fac:gaussian-tail-upper-bound}
            Let $\lvar{z}\sim\gaussian[][0][\sigma^2]$, then $\prob{\bvar{x}\geq t}\leq e^{-t^2/2\sigma^2}$.
        \end{fact}

        
        \begin{lemma}[Chernoff Bound]\label{lma:chernoff-bound}
            Let $\lvarseq{X}<m>$ be a sequence of $m$ independent Bernoulli trials, each with probability of success $\E{\lvar{X}<i>} = p$, then with $\gamma\in[0,1]$, we have
            \begin{equation*}
                \prob{\abs{\frac{1}{m}\sum_{i=1}^m\lvar{X}<i> - p} > \gamma}\leq 2e^{-2m\gamma^2}.
            \end{equation*}
        \end{lemma}


        \begin{lemma}[Hoeffding Bound]\label{lma:hoeffding-bound}
            Let $\lvarseq{X}<m>$ be a sequence of $m$ independent random variables such that $\lvar{X}<i>\in[\lvar{a}<i>,\lvar{b}<i>]$ for every $i$. Then, for all $t \geq 0$, we have 
            \begin{equation*}
                \prob{\sum_{i=1}^m(\lvar{X}<i> - \E{\lvar{X}<i>})\geq t}\leq\exp\sbr{-\frac{t^2}{2\sum_{i=1}^m \sbr*{\lvar{b}<i> - \lvar{a}<i>}^2}}.
            \end{equation*}
        \end{lemma}

        \begin{lemma}\label{lma:concentration-bound-on-the-projected-gradient}
            Let $\distr$ be a distribution on $\reals[d]\times\booldomain$ with standard normal $\bvar{x}$-marginal, and $\func{g}<\bvar{w}>[\bvar{x}, y] = y\cdot\bvar{x}<\bvar{w}[\bot]>\cdot\indicator[\bvar{x}\in\hypothesis[\bvar{w}]]$. Given a sequence of $m$ examples $\distr* = \lbr*{(\bvar{x}(1),\lvar{y}(1)),\ldots,(\bvar{x}(m),\lvar{y}(m))}$ sampled i.i.d. from $\distr$ and two orthonormal vectors $\bvar{u}, \bvar{w}\in\reals[d]$, then for $t\in(0, 1]$, it holds that
            \begin{equation*}
                \prob{\abs{\innerprod{\E<\distr*>{\func{g}<\bvar{w}>[\bvar{x}, y]} -\E<\distr>{\func{g}<\bvar{w}>[\bvar{x}, y]}}{\bvar{u}}}\geq t}\leq 2e^{-t\sqrt{m}/16}.
            \end{equation*}
            for any $m\geq 3/t$ that also satisfies $m\geq 256t^{-2}\ln^2(2m)$.
        \end{lemma}
        \begin{proof}
            We prove the concentration bound by showing that $\func{g}<\bvar{w}>[\bvar{x}, y]$ is bound almost surely, then apply lemma \ref{lma:hoeffding-bound} to get the desired result.

            Define $\bvar{x}<1> = \innerprod{\bvar{x}<\bvar{w}[\bot]>}{\bvar{u}}$ as well as $\bvar{x}<2> = \bvar{w}$ so that $\bvar{x}<1>,\bvar{x}<2>$ are independent standard normal random variables. Let $\distr<\abs{\bvar{x}<1>}\geq B>$ (resp. $\distr<\abs{\bvar{x}<1>}\leq B>$) denote the conditional distribution of $\distr<z>$ conditioned on $\abs{\bvar{x}<1>}\geq B$ (resp. $\abs{\bvar{x}<1>}\leq B$) for $B \geq 0$. 
            
            We first bound the difference between $\innerprod{\E<\distr>{\func{g}<\bvar{w}>[\bvar{x}, y]}}{\bvar{u}}$ and $\innerprod*{\E<\distr<\abs{\bvar{x}<1>}\leq B>>{\func{g}<\bvar{w}>[\bvar{x}, y]}}{\bvar{u}}$ using conditional probability tools and Gaussian tail bound as follow: 
            \begin{align}
                &\abs{\E<\distr>{y\cdot\bvar{x}<1>\cdot\indicator[\bvar{x}<2>\geq 0]} - \E<\distr<\abs{\bvar{x}<1>}\leq B>>{y\cdot\bvar{x}<1>\cdot\indicator[\bvar{x}<2>\geq 0]}}\notag\\
                =&\abs{-\E<\distr<\abs{\bvar{x}<1>}\leq B>>{y\cdot\bvar{x}<1>\cdot\indicator[\bvar{x}<2>\geq 0]}\prob<\distr>{\abs{\bvar{x}<1>}> B} + \E<\distr>{y\cdot\bvar{x}<1>\cdot\indicator[\bvar{x}<2>\geq 0, \abs{\bvar{x}<1>} > B]}}\notag\\
                \leq&\E<\distr<\abs{\bvar{x}<1>}\leq B>>{\abs{\bvar{x}<1>}\cdot\indicator[\bvar{x}<2>\geq 0]}\prob<\distr>{\abs{\bvar{x}<1>}> B} + \E<\distr>{\abs{\bvar{x}<1>}\cdot\indicator[\bvar{x}<2>\geq 0, \abs{\bvar{x}<1>} > B]}\notag\\
                =&\E<\distr>{\abs{\bvar{x}<1>}\cdot\indicator[\bvar{x}<2>\geq 0, \abs{\bvar{x}<1>}\leq B]}\cdot \frac{\prob<\distr>{\abs{\bvar{x}<1>}> B}}{\prob<\distr>{\abs{\bvar{x}<1>}\leq B}} + \E<\distr>{\abs{\bvar{x}<1>}\cdot\indicator[\bvar{x}<2>\geq 0, \abs{\bvar{x}<1>} > B]}\notag\\
                \ceq[(i)]& \frac{\prob<\distr>{\abs{\bvar{x}<1>}> B}}{2\sqrt{2\pi}\prob<\distr>{\abs{\bvar{x}<1>}\leq B}}\int_{-B}^B\abs{\bvar{x}<1>}e^{-\bvar{x}<1>[2]/2}d\bvar{x}<1> + \frac{1}{2\sqrt{2\pi}}\sbr{\int_{B}^{+\infty}\bvar{x}<1>e^{-\bvar{x}<1>[2]/2}d\bvar{x}<1> - \int_{-\infty}^{-B}\bvar{x}<1>e^{-\bvar{x}<1>[2]/2}d\bvar{x}<1>}\notag\\
                \ceq[(ii)]& \frac{\prob<\distr>{\abs{\bvar{x}<1>}> B}}{\sqrt{2\pi}\sbr{1 - \prob<\distr>{\abs{\bvar{x}<1>}> B}}}\sbr{1 - e^{-B^2/2}} + \frac{1}{\sqrt{2\pi}}e^{-B^2/2}\notag\\
                \cleq[(iii)]&\frac{2e^{-B^2/2}}{\sqrt{2\pi}\sbr{1 - 2e^{-B^2/2}}}\sbr{1 - e^{-B^2/2}} + \frac{1}{\sqrt{2\pi}}e^{-B^2/2}\notag\\
                \leq& 3e^{-B^2/2}\label{eq:upper-bound-on-expectation-difference-between-conditional-distr}
            \end{align}
            where inequality (i) holds due to the independence between $\bvar{x}<1>,\bvar{x}<2>$ as well as that $\prob<z\sim\gaussian>{z\geq 0} = 1/2$, inequality (ii) holds because of the symmetric property of standard normal distribution, (iii) is obtained by apply lemma \ref{lma:gaussian-tail-lower-bound} to the first term, and the last inequality holds via using the fact that $\frac{a}{b}\leq \frac{a+c}{b+c}$ for $0<a\leq b, 0\leq c$ to $2e^{-B^2/2}/\sbr*{1 - 2e^{-B^2/2}}$ as well as $4/\sqrt{2\pi} < 3$. Then, with a similar analysis, we have, for any $t\in(0, 1]$, that
            \begin{align*}
                &\prob<\distr>{\abs{\innerprod{\E<\distr*>{\func{g}<\bvar{w}>[\bvar{x}, y]} -\E<\distr>{\func{g}<\bvar{w}>[\bvar{x}, y]}}{\bvar{u}}}\geq t}
                \\
                =&\prob<\distr>{\abs{\frac{1}{m}\sum_{i=1}^m\lvar{y}(i)\bvar{x}<1>(i)\indicator[\bvar{x}<2>(i)\geq 0] - \E<\distr>{y\cdot\bvar{x}<1>\cdot\indicator[\bvar{x}<2>\geq 0]}}\geq t}
                \\
                =&\prob<(\bvar{x}(i),\lvar{y}(i))\sample\distr<\abs{\bvar{x}<1>}\leq B>>{\abs{\frac{1}{m}\sum_{i=1}^m\lvar{y}(i)\bvar{x}<1>(i)\indicator[\bvar{x}<2>(i)\geq 0] - \E<\distr>{y\cdot\bvar{x}<1>\cdot\indicator[\bvar{x}<2>\geq 0]}}\geq t}\sbr{1-\prob<(\bvar{x}(i),\lvar{y}(i))\sample\distr>{\bunion_{i=1}^m\abs{\bvar{x}<1>(i)} >B}}
                \\
                &+ \prob<(\bvar{x}(i),\lvar{y}(i))\sample\distr>{\abs{\frac{1}{m}\sum_{i=1}^m\lvar{y}(i)\bvar{x}<1>(i)\indicator[\bvar{x}<2>(i)\geq 0] - \E<\distr>{y\cdot\bvar{x}<1>\cdot\indicator[\bvar{x}<2>\geq 0]}}\geq t\bigisect\sbr{\bunion_{i=1}^m\abs{\bvar{x}<1>(i)} >B}}
                \\
                \leq& \prob<(\bvar{x}(i),\lvar{y}(i))\sample\distr<\abs{\bvar{x}<1>}\leq B>>{\abs{\frac{1}{m}\sum_{i=1}^m\lvar{y}(i)\bvar{x}<1>(i)\indicator[\bvar{x}<2>(i)\geq 0] - \E<\distr>{y\cdot\bvar{x}<1>\cdot\indicator[\bvar{x}<2>\geq 0]}}\geq t}+ \prob<(\bvar{x}(i),\lvar{y}(i))\sample\distr>{\bunion_{i=1}^m\abs{\bvar{x}<1>(i)} >B}
                \\
                \cleq[(i)]& \prob<(\bvar{x}(i),\lvar{y}(i))\sample\distr<\abs{\bvar{x}<1>}\leq B>>{\abs{\frac{1}{m}\sum_{i=1}^m\lvar{y}(i)\bvar{x}<1>(i)\indicator[\bvar{x}<2>(i)\geq 0] - \E<\distr<\abs{\bvar{x}<1>}\leq B>>{y\cdot\bvar{x}<1>\cdot\indicator[\bvar{x}<2>\geq 0]}}\geq t - 3e^{-B^2/2}}+ 2me^{-B^2/2}
                \\
                \cleq[(ii)]& \prob<(\bvar{x}(i),\lvar{y}(i))\sample\distr<\abs{\bvar{x}<1>}\leq B>>{\abs{\frac{1}{m}\sum_{i=1}^m\lvar{y}(i)\bvar{x}<1>(i)\indicator[\bvar{x}<2>(i)\geq 0] - \E<\distr<\abs{\bvar{x}<1>}\leq B>>{y\cdot\bvar{x}<1>\cdot\indicator[\bvar{x}<2>\geq 0]}}\geq\frac{1}{2}t}+ e^{-t\sqrt{m}/16}
            \end{align*}
            \begin{flalign*}
                \cleq[(iii)]& 2\exp\sbr{-\frac{mt^2}{32\sbr{t\sqrt{m}/8 + 2\ln 2m}}} + e^{-t\sqrt{m}/16}&&
                \\
                \cleq& 2e^{-t\sqrt{m}/8} + e^{-t\sqrt{m}/16}&&
                \\
                \leq&2e^{-t\sqrt{m}/16}&&
            \end{flalign*}
            where inequality (i) holds due to inequality \eqref{eq:upper-bound-on-expectation-difference-between-conditional-distr}, inequality (ii) holds because taking $B^2 = t\sqrt{m}/8 + 2\ln 2m$ for any $m \geq 1$ that satisfies $m\geq 256t^{-2}\ln^2(6/t)$ gives $3e^{-B^2/2}\leq t/2$ as well as $2me^{-B^2/2} = e^{-t\sqrt{m}/16}$, inequality (iii) results from applying lemma \ref{lma:hoeffding-bound} to the first term with $\bvar{x}<1>\in[-B,B]$, and the last two inequalities can be obtained by noticing that $t\sqrt{m}/8\geq 2\ln 2m$ as well as $2e^{-t\sqrt{m}/8}\leq e^{-t\sqrt{m}/16}$ whenever $m\geq 256t^{-2}\ln^2(2m)$. At last, taking $m$ large enough gives the claimed result.
        \end{proof}
        
        \begin{lemma}\label{lma:gaussian-tail-lower-bound}
            Let $\lvar{x}\sim\gaussian(0,1)$, then $\prob{\lvar{x} \geq t}\geq\frac{1}{\sqrt{2\pi}(t + 1)}e^{-t^2/2}$ for every $t\geq0$.
        \end{lemma}
        \begin{proof}
            Define $f:\reals\rightarrow\reals$ as
            \begin{align*}
                f(t) =& \sqrt{2\pi}\prob{\lvar{x}\geq t} - \frac{1}{t + 1}e^{-t^2/2}\\
                =& \int_t^{+\infty} e^{-\lvar{x}[2]/2}d\lvar{x} - \frac{1}{t + 1}e^{-t^2/2}.
            \end{align*}
            Observe that $f(0) = \sqrt{\pi/2} - 1 > 0$ and
            \begin{align}
                \nabla_t f(t) =& -e^{-t^2/2}-\sbr{-\frac{1}{(t+1)^2}e^{-t^2/2} -\frac{t}{t+1}e^{-t^2/2}}\\
                =& -\frac{t}{(t+1)^2}e^{-t^2/2}\\
                \leq& 0
            \end{align}
            for $t\geq 0$. Furthermore, we have $\lim_{t\rightarrow +\infty} f(t) = 0$, which implies $f(t)$ is always positive on $t\in[0, +\infty)$ and, hence, the claimed result.
        \end{proof}
\end{document}